\documentclass{article}
\pdfoutput=1

   \usepackage[nonatbib,final]{neurips_2024}

\usepackage[round]{natbib}

\usepackage[utf8]{inputenc} %
\usepackage[T1]{fontenc}    %
\usepackage[pagebackref]{hyperref} 
\usepackage{url}            %
\usepackage{booktabs}       %
\usepackage{amsfonts}       %
\usepackage{nicefrac}       %
\usepackage{microtype}      %
\usepackage[dvipsnames]{xcolor}
\hypersetup{
    colorlinks=true,
    citecolor=MidnightBlue,
    linkcolor=purple,
    urlcolor=MidnightBlue,
}

\usepackage{amsmath}
\usepackage{amssymb}
\usepackage{mathtools}
\usepackage{amsthm}

\usepackage{amsthm}
\usepackage[capitalize,noabbrev]{cleveref}
\theoremstyle{plain}
\newtheorem{theorem}{Theorem}[section]

\newtheorem{lemma}[theorem]{Lemma}

\theoremstyle{definition}
\newtheorem{definition}[theorem]{Definition}
\newtheorem{assumption}[theorem]{Assumption}
\theoremstyle{remark}
\newtheorem{remark}[theorem]{Remark}

\usepackage{enumitem} %

\usepackage{mmll}

\let\classAND\AND
\let\AND\relax
\usepackage{algorithmic}

\let\AND\classAND
\AtBeginEnvironment{algorithmic}{\let\AND\algoAND}

\usepackage{colortbl}
\definecolor{Lightgray}{RGB}{235,235,235}
\usepackage{xcolor,colortbl}
\usepackage{multirow}
\usepackage{booktabs,nicematrix}

\definecolor{Gray}{gray}{0.85}
\definecolor{LightCyan}{rgb}{0.88,1,1}
\newcolumntype{a}{>{\columncolor{Gray}}c}
\newcolumntype{b}{>{\columncolor{white}}c}

\let\hat\widehat
\let\tilde\widetilde
\def\given{{\,|\,}}

\def\Bigggiven{{\,\Bigg|\,}}
\DeclareMathOperator{\Tr}{Tr}

\renewcommand*{\backref}[1]{}
\renewcommand*{\backrefalt}[4]{%
\ifcase #1 %
    No citations.%
\or
    (p. #2.)%
\else
    (pp. #2.)%
\fi}%

\title{Offline Multitask   Representation Learning for Reinforcement Learning}

\author{Haque Ishfaq\footnotemark[1] \\
    Mila, McGill University\\
    \href{mailto:haque.ishfaq@mail.mcgill.ca}{\texttt{haque.ishfaq@mail.mcgill.ca}}\\
    \And
    Thanh Nguyen-Tang \\
    Johns Hopkins University\\
    \href{mailto:nguyent@cs.jhu.edu}{\texttt{nguyent@cs.jhu.edu}}\\
    \And 
    Songtao Feng \\
    University of Florida\\
    \href{mailto:sfeng1@ufl.edu}{\texttt{sfeng1@ufl.edu}}\\
    \And 
    Raman Arora \\
    Johns Hopkins University\\
\href{mailto:arora@cs.jhu.edu}{\texttt{arora@cs.jhu.edu}}\\
    \And 
    Mengdi Wang \\
    Princeton University \\
    \href{mailto:mengdiw@princeton.edu}{\texttt{mengdiw@princeton.edu}}\\
    \And 
    Ming Yin\footnotemark[1]\\
    Princeton University \\
    \href{mailto:my0049@princeton.edu}{\texttt{my0049@princeton.edu}}\\
    \And 
    Doina Precup\thanks{The corresponding authors.} \\
    Mila, McGill University \\
    \href{mailto:dprecup@cs.mcgill.ca}{\texttt{dprecup@cs.mcgill.ca}}\\
    }
\begin{document}

\maketitle

\begin{abstract}
We study offline multitask representation learning in reinforcement learning (RL), where a learner is provided with an offline dataset from different tasks that share a common representation and is tasked to learn the shared representation. We theoretically investigate offline multitask low-rank RL, and propose a new algorithm called MORL for offline multitask representation learning. 
Furthermore, we examine downstream RL in reward-free, offline and online scenarios, where a new task is introduced to the agent that shares the same representation as the upstream offline tasks. Our theoretical results demonstrate the benefits of using the learned representation from the upstream offline task instead of directly learning the representation of the low-rank model.
\end{abstract}

\section{Introduction}
Recent advances in offline reinforcement learning (RL) \citep{levine2020offline} have opened up possibilities for training policies for real-world problems using pre-collected datasets, such as robotics \citep{kalashnikov2018scalable,rafailov2021offline,kalashnikov2021mt}, natural language processing \citep{jaques2019way}, education \citep{de2021discovering}, electricity supply \citep{zhan2022deepthermal}
and healthcare \citep{guez2008adaptive,shortreed2011informing,wang2018supervised,killian2020empirical}. While most offline RL studies focused on single-task problems, there are many practical scenarios where multiple tasks are correlated and it is beneficial to learn multiple tasks jointly by utilizing all of the data available \citep{kalashnikov2018scalable,yu2021conservative,yu2022leverage,xie2022lifelong}. One popular approach in such cases is multitask representation learning, where the agent aims to tackle the problem by extracting a shared low-dimensional representation function among related tasks and then using a simple function (e.g., linear) on top of this common representation to solve each task \citep{caruana1997multitask,baxter2000model}. Despite the empirical success of multitask representation learning, particularly in reinforcement learning for its efficacy in reducing the sample complexity \citep{teh2017distral,sodhani2021multi,arulkumaran2022all}, the theoretical understanding of it is still in its early stages \citep{brunskill2013sample,calandriello2014sparse,arora2020provable,d2020sharing,hu2021near,lu2021power,muller2022meta}. Although some works theoretically studied the online multitask representation learning for RL where the agent is allowed to interact with multiple source tasks to learn the shared representation \citep{cheng2022provable,agarwal2023provable,sam2024limitstransferreinforcementlearning}, there is currently no theoretical understanding on the effectiveness of multitask RL in the \emph{offline} setting. This is crucial as in many practical scenarios \citep{kumar2022pre,yoo2022skills,lin2022switch}, it is not feasible to interact with the different task environments in an online manner.

Moreover, when the tasks share the same representation, offline multitask representation learning can serve as a launchpad for effectively solving many other downstream tasks \citep{kumar2022prerobot}. Consider the problem of learning to control robotic arms where we may already have offline datasets from different pick-and-place tasks in a kitchen such as the Bridge Dataset \citep{ebert2021bridge}. From this one can consider many possible downstream RL tasks where representation learned from these offline datasets can be beneficial. For example, one may consider solving a new pick-and-place task with different previously unseen objects in either an online or offline manner. Alternatively, one may consider a downstream reward-free RL \citep{jin2020reward} task where the agent would first gather additional novel and diverse data without a pre-specified reward function and afterward, when provided with any reward function (e.g. slightly different target placing spot for the picked object), would be asked to provide a good policy without additional interaction. 

In this work, we study the provable benefits of offline multi-task representation learning for RL  in which the learner is only given access to pre-collected data from different source tasks which are modeled by low-rank MDPs \citep{agarwal2020flambe} with a shared (yet unknown) representation.

\textbf{Our contributions.} We develop a new offline multitask reinforcement learning algorithm that enables sample efficient representation learning in low-rank MDPs \citep{agarwal2020flambe} and further provide improved sample complexity to the downstream learning. In summary, our main contributions are:

\begin{itemize}[leftmargin=*,label={\large\textbullet}]
    \item We propose a new offline multitask representation learning algorithm called Multitask Offline Representation Learning (MORL) under low-rank MDPs. MORL represents a standard training procedure in modern machine learning, by pooling the data from all source tasks to learn a shared representation of the dynamics via maximum likelihood estimation oracle.

    \item We prove that, MORL can learn a near-accurate model, and, when combined with the pessimism principle, find a near-optimal policy for each of the source tasks $T$ in the average sense, more sample-efficiently, than learning each task in isolation. To our knowledge, this is the first theoretical result demonstrating the benefit of representation learning in offline multitask RL. 
    
    \item We then show theoretical benefits of using the learned representation from MORL in downstream reward-free RL \cite{jin2020reward,wang2020reward}. In particular, we show that, to guarantee an $\epsilon$-suboptimal policy for uniformly over any reward function, our algorithm requires at most $\tilde{O}(\frac{H^4d^3}{\epsilon^2})$ episodes during the exploration phase where $d$ is the dimension of the feature and $H$ is the planning horizon. This improves the best known sample complexity for the reward-free RL in low-rank MDP \citep{cheng2023improved} by a factor of $\tilde{O}(HdK)$, where $K$ is the cardinality  of the action space. In addition, as a complementary result, we show that using the learned representation from MORL improves the suboptimality gap bound in both offline and online downstream task.

\end{itemize}

\section{Preliminary}

\paragraph{Episodic MDP.}
We consider an episodic discrete-time Markov Decision Process (MDP), denoted by $\cM = (\cS, \cA, H, P, r)$, where $\cS$ is the state space, $\cA$ is the action space with cardinality $K$, $H$ is the finite episode length, $P =\{P_h\}_{h=1}^H$ are the state transition probability distributions with $P_h:\cS\times\cA\rightarrow \Delta(\cS)$, and $r= \{r_h\}_{h=1}^H$ are the deterministic reward functions with $r_h:\cS\times\cA\rightarrow[0,1]$. Following prior work \citep{jiang2017contextual,sun2019model}, we assume that the initial state $s_1$ is fixed for each episode. A policy $\pi$ is a collection of $H$ functions $\{\pi_h : \mathcal{S} \rightarrow \mathcal{A}\}_{h\in [H]}$ where $\pi_h(s)$ is the action that the agent takes at state $s$ and at the  $h$-th step in the episode. Given a starting state $s_h$, $s_{h'} \sim (P,\pi)$ denotes a state sampled by executing policy $\pi$ under the transition model $P$ for $h'-h$ steps and $\EE_{(s_h,a_h)\sim(P,\pi)}[\cdot]$ denotes the expectation over states $s_h \sim (P,\pi)$ and actions $a_h \sim \pi$. Moreover, for each $h \in [H]$, we define the value function under policy $\pi$ when starting from an arbitrary state $s_h = s$ at the $h$-th time step as 
\begin{equation*}
    V^{\pi}_{h,P,r}(s) = \EE_{(s_{h'},a_{h'})\sim (P,\pi)}\bigg[\sum_{h'=h}^H r_{h'}(s_{h'},a_{h'})|s_h =s\bigg].
\end{equation*}
We define the action-value function for a given state-action pair $(s,a)$ under policy $\pi$ at step $h$ as 
\begin{equation*}
    Q^{\pi}_{h,P,r}(s,a) = \EE_{(s_{h'},a_{h'})\sim (P,\pi)}\bigg[\sum_{h'=h}^H r_{h'}(s_{h'},a_{h'})|s_h =s,a_h=a\bigg]. 
\end{equation*}
Defining $(P_h f)(s,a) = \EE_{s'\sim P(\cdot|s,a)}[f(s')]$ for any function $f:\cS\rightarrow\RR$, we write the Bellman equation associated with a policy $\pi$ as 
\begin{equation}\label{eq:bellman-eq}
    Q_{h,P,r}^\pi(s,a) = (r_h + P_h V_{h+1,P,r}^\pi)(s, a),\hspace{0.5em}
    V_{h,P,r}^\pi(s) = Q_{h,P,r}^\pi(s, \pi_h(s)),\hspace{0.5em} 
    V_{H+1,P,r}^\pi(s)  = 0.
\end{equation}
Since the MDP begins with the same initial state $s_1$, for simplicity, we use $V_{P,r}^\pi$ to denote $V_{1,P,r}^\pi(s_1)$. Another useful concept is the notion of occupancy measure of a policy $\pi$ at time step $h$ under transition kernel $P$. Specifically, we use $d^{\pi}_{P_h}(s, a)$ to denote the marginal probability of encountering the state-action pair $(s,a)$ at time step $h$ when executing policy $\pi$ under MDP with transition kernel $P$.
Finally, we denote $\cU(\cS)$ and $\cU(\cS,\cA)$ as the uniform distribution over $\cS$ and $\cS\times\cA$ respectively.

We study low-rank MDPs \citep{jiang2017contextual,agarwal2020flambe} defined as follows. 

\begin{definition}[Low-rank MDPs]\label{Def:low-rank-mdp}
A transition kernel $P_h^*:\cS\times\cA\rightarrow \Delta(\cS)$ admits a low-rank decomposition with dimension $d \in \NN$ if there exists two unknown embedding functions $\phi_h^*:\cS\times\cA\rightarrow \RR^d$ and $\mu_h^*:\cS\rightarrow \RR^d$ such that for all $s,s' \in \cS$ and $a \in \cA$,
$P_h^*(s' \given s,a) = \la \phi_h^*(s,a), \mu_h^*(s')\ra$.
Without loss of generality, we assume $\|\phi_h^*(s,a)\|_2 \leq 1$ for all $(s,a) \in \cS\times\cA$ and for any function $g:\cS\rightarrow [0,1]$, $\|\int \mu_h^*(s)g(s)ds\|_2 \leq \sqrt{d}$.
\end{definition}

We remark that the upper bounds on the norm of $\phi^*$ and $\mu^*$ are just for normalization. As the function class $\Phi$ for $\phi^*$ can be a non-linear, flexible function class, the low-rank MDP generalizes prior works with linear representations \citep{jin2019provably,hu2021near} where it is assumed that the true representation $\phi^*$ is known to the agent a priori. 

\subsection{Offline Multitask RL with Downstream Learning}

In \textbf{offline multitask RL} upstream learning, the agent is provided with an offline dataset collected from $T$ source tasks, where the reward functions $\{r^t\}_{t\in[T]}$ are assumed to be known.
Each task $t \in [T]$ is associated with a low-rank MDP $\cM^t = (\cS, \cA, H, P^t, r^t)$. Here, all $T$ tasks are identical except for (1) their true transition model $P^{(*,t)}$, which admits a low-rank decomposition with dimension $d$: $P_h^{(*,t)}(s'_h\given s_h,a_h) = \la \phi^*_h(s'), \mu_h^{(*,t)}(s,a)\ra$ for all $h \in [H],t \in [T]$, and (2) their reward $r_h^t$. While the tasks may differ in $\mu_h^{(*,t)}$ and $r_h^t$, we emphasize that all tasks share the same feature function $\phi_h^*$. We have access to offline dataset $\cD = \bigcup_{t\in[T], h \in [H]} \cD_h^{(t)}$, where $\cD_h^{(t)} = \{(s_h^{(i,t)}, a_h^{(i,t)}, r_h^{(i,t)}, s_{h+1}^{(i,t)}\}_{i\in[n]}$ with $s_{h+1}^{(i,t)} \sim P_h^{(*,t)}( \cdot \given s_h^{(i,t)}, a_h^{(i,t)})$ and $\cD_h^{(t)}$ was collected using a \emph{fixed behavior policy} $\pi_t^b$. In the upstream learning stage, the goal is to find a near-optimal policy and a near-accurate model for any task $t\in[T]$ and any reward function $\{r_t\}_{t\in[T]}$ through the use of offline dataset $\cD$ and provide a well-learned representation for the downstream task. In order to achieve bounded sample complexity in offline RL, we need additional coverage assumption on the behavior policy $\pi_t^b$. One common coverage assumption is the global coverage assumption \citep{antos2008learning,munos2008finite}, which assumes the occupancy measure under the behavior policy $\pi_t^b$ globally covers the the occupancy measure under any possible policies, i.e., the concentrability ratio satisfies, $\max_{\pi,s,a} d^{\pi}_{P^{(*,t)}_h}(s,a)/d^{\pi_t^b}_{P^{(*,t)}_h}(s,a) < \infty$. Instead, we make a partial coverage assumption and our suboptimality bound scales with the relative condition number \citep{agarwal2020pc,agarwal2021theory} instead of the global concentrability ratio, where the former can be substantially smaller than the latter. Under this assumption, we want to compete against any comparator policy covered by the offline data. In \Cref{Section:theoretical_result_on_upstream}, we  define the partial coverage condition using  relative condition number, which was previously used in the context of single-task offline RL \citep{uehara2021pessimistic,uehara2021representation} and generalize it to offline multitask setting. 

In \textbf{downstream learning} stage, a new target task $T+1$ with a low-rank transition kernel $P^{(*,T+1)}$ and the same $\cS$, $\cA$ and $H$ is assigned to the agent. The transition kernel $P^{(*,T+1)}$ shares the same representation $\phi^*$ with the $T$ upstream tasks, but has a task-distinct $\mu^{(*,T+1)}$. We consider three settings for downstream tasks -- reward-free, offline and online RL, where the agent needs to use the representation function $\hat{\phi}$ learned during the upstream stage to interact with the new task environment. 

In the reward-free setting, firstly proposed in \citet{jin2020reward}, the agent first interacts with the new task environment without accessing the reward function in the exploration phase for up to $K_{\text{RFE}}$ episodes. Afterwards, it is provided with a reward function $r = \{r_h\}_{h=1}^H$ and asked to output an $\epsilon$-optimal policy $\pi$ in the planning phase. We define the sample complexity to be the number of episodes $K_{\text{RFE}}$ required in the exploration phase to output an $\epsilon$- optimal policy $\pi$ in the planning phase for any given reward function $r$.

In the offline and online setting the downstream task $T+1$ is already assigned with an unknown reward function $r^{T+1}$ and the goal is to  find a near-optimal policy for the new task. The agent is expected to expedite its downstream learning through using the representation learned from the offline upstream task. In the online setting, it is allowed to interact with the new task environment and
in the offline setting, it is instead provided with an offline dataset $\cD_{\text{off}} = \bigcup_{h \in [H]} \cD_h$, where $\cD_h = \{(s_h^\tau,a_h^\tau, r_h^\tau, s_{h+1}^\tau)\}_{\tau \in [N_{\text{off}}]}$ and $\cD_{\text{off}}$ were collected using some behavior policy $\rho$.

\section{Upstream Offline Multitask Representation Learning}

In this section, we introduce our algorithm Multitask Offline Reinforcement Learning (MORL) designed for upstream  offline multitask RL in low-rank MDPs and describe its theoretical properties.

\subsection{Algorithm Design}

The details of the algorithm MORL is depicted in \Cref{Algorithm:multi-task-offline}. The agent passes all input offline data to estimate low-rank components $\hat{\phi}_h, \hat{\mu}_h^{(1)}, \ldots, \hat{\mu}_h^{(T)}$ simultaneously via the Maximum Likelihood Estimation (MLE) oracle $MLE\left(\bigcup_{t\in[T]} \cD_h^{(t)}\right)$ on the joint distribution defined as follows:
\begin{equation}\label{Eq:MLE-oracle}
    \left(\hat{\phi}_h, \hat{\mu}_h^{(1)}, \ldots, \hat{\mu}_h^{(T)}\right)= \argmax_{{\phi_h \in \Phi,} \atop {\mu_h^{(1)}, \ldots, \mu_h^{(T)}\in \Psi}} \sum_{i=1}^n \sum_{t=1}^T \log \left( \left\la \phi_h(s_h^{(i,t)}, a_h^{(i,t)}), \mu_h^t(s_{h+1}^{(i,t)}\right\ra\right).
\end{equation}

\begin{algorithm}[t]
\caption{Multitask Offline Representation Learning (MORL)}\label{Algorithm:multi-task-offline}
\begin{algorithmic}[1]
\STATE \textbf{Input:}\\ Dataset $\cD = \{(s_h^{(i,t)}, a_h^{(i,t)}, r_h^{(i,t)}, s_{h+1}^{(i,t)}\}_{i\in[n], h\in [H], t\in[T]}$,  Regularizer $\lambda$, Parameter $\alpha$, Models $\{(\phi, \mu): \phi \in \Psi, \mu \in \Psi\}$
\FOR{$h = 1, \ldots, H$}
\STATE Learn $\left(\hat{\phi}_h, \hat{\mu}_h^{(1)}, \ldots, \hat{\mu}_h^{(T)}\right)$ via $\text{MLE}\left(\bigcup_{t\in[T]} \cD_h^{(t)}\right)$ as ~\eqref{Eq:MLE-oracle}
\ENDFOR
\FOR{$t = 1, \ldots, T$}
\STATE Update estimated transitioned kernels $\hat{P}_h^{(t)}(\cdot \mid \cdot,\cdot)$ as ~\eqref{Eq:estimated-transition-kernel}, empirical covariance matrix $\hat{\Sigma}_h^{(t)}$ as ~\eqref{Eq:estimated-covariance-matrix} and penalty term $\hat{b}_h^{(t)}$ as ~\eqref{Eq:LCB}
\STATE Get policy $\hat{\pi}_t = \argmax_{\pi} V_{\hat{P}^{(t)},r^t-\hat{b}^{(t)}}$
\ENDFOR
\STATE \textbf{Output:} $\hat{\phi}, \hat{P}^{(1)}, \ldots, \hat{P}^{(T)}, \hat{\pi}_1, \ldots, \hat{\pi}_T$
\end{algorithmic}
\end{algorithm}

The MLE oracle in \eqref{Eq:MLE-oracle} is the offline multitask counterpart to the celebrated MLE oracle in the online multitask RL %
\citep{agarwal2020flambe, cheng2022provable,agarwal2023provable}. 
The MLE oracle can be reasonably approximated in practice whenever optimizing over  $\Phi$ and $\Psi$ is feasible through proper parameterization such as by neural network.
For each task $t$, we obtain the estimated transition kernel $\hat{P}^{(t)}$ at each step $h$ using the learned embeddings $\hat{\phi}_h$, $\hat{\mu}_h^{(t)}$:
\begin{equation}\label{Eq:estimated-transition-kernel}
    \hat{P}_h^{(t)}(s' \mid s, a) =  \la \hat{\phi}_h(s,a), \hat{\mu}_h^{(t)}(s') \ra.
\end{equation}
Using the representation estimator $\hat{\phi}_h$, we  set the empirical covariance matrix $\hat{\Sigma}_{h,\hat{\phi}}^{(t)}$ for task $t$ as 
\begin{equation}\label{Eq:estimated-covariance-matrix}
    \hat{\Sigma}_{h,\hat{\phi}}^{(t)} = \sum_{i=1}^n \hat{\phi}_h(s_h^{(i,t)},a_h^{(i,t)})\hat{\phi}_h(s_h^{(i,t)},a_h^{(i,t)})^\top + \lambda I.
\end{equation}
Using both $\hat{\phi}_h$ and $\hat{\Sigma}_{h,\hat{\phi}}^{(t)}$, we construct a lower confidence bound penalty term as follows:
\begin{equation}\label{Eq:LCB}
    \hat{b}_h^{(t)}(s_h, a_h) = \min \left\{\alpha \|\hat{\phi}_h(s_h,a_h)\|_{(\hat{\Sigma}_{h,\hat{\phi}}^{(t)})^{-1}}, 1\right\},
\end{equation}
where $\alpha$ is a pre-determined parameter. 

Finally, for each task $t$, with the learned model $\hat{P}^{(t)}$ and the reward $r^t - \hat{b}^{(t)}$, we do planning to get policy $\hat{\pi}_t$.

\subsection{Theoretical Result on Upstream Task}\label{Section:theoretical_result_on_upstream}

To facilitate the model selection task using the joint MLE oracle in \eqref{Eq:MLE-oracle}, we posit a realizability assumption which is standard in low-rank MDP literature \citep{agarwal2020flambe,cheng2022provable}.

\begin{assumption}[Realizability]
A learning agent has access to a model class $\{\Phi, \Psi\}$ that contains the true model, i.e., for any $h\in[H], t\in [T]$, the embeddings $\phi^*_h \in \Phi$, $\mu^{(*,t)} \in \Psi$. For normalization, we assume that for any $\phi \in \Phi$, $\|\phi(s,a)\|_2\leq 1$ and for any $\mu \in \Psi$ and any function $g:\cS:\rightarrow[0,1]$, $\|\int\mu_h(s)g(s)ds\|_2 \leq \sqrt{d}$.
\label{Ass:realizability}
\end{assumption}

For simplicity, we assume that the cardinality of the function classes $\Phi$ and $\Psi$ are finite. 

Next, we define the multitask relative condition number $C^*$, which is a natural extension of the standard relative condition number \citep{agarwal2020pc,agarwal2021theory,uehara2021representation}. 

\begin{definition}[Multi-task relative condition number]
    For task $t$ and time step $h$, we define $C^*_{t,h}(\pi_t,\pi_t^b)$ as the relative condition number under $\phi^*_h$:
\begin{equation}\label{Equation:relative_condition_number}
    C^*_{t,h}(\pi_t,\pi_t^b) := \sup_{x\in \RR^d}\frac{x^\top \EE_{(s_h,a_h) \sim (P^{(*,t)},\pi_t)} [\phi^*_h(s_h,a_h)\phi^*_h(s_h,a_h)^\top] x}{x^\top \EE_{(s_h,a_h) \sim (P^{(*,t)},\pi_t^b)} [\phi^*_h(s_h,a_h)\phi^*_h(s_h,a_h)^\top] x}.
\end{equation}
We define $C^*_t := \max_{h\in[H]}C^*_{t,h}(\pi_t,\pi_t^b)$ and $C^*:= \max_{t\in[T]}C^*_t$.
\end{definition}

Intuitively, $C^*_{t,h}(\pi_t,\pi_t^b)$ defined in \eqref{Equation:relative_condition_number} measures the deviation between a comparator policy $\pi_t$ and the behavior policy $\pi_t^b$ at time step $h$. When tabular MDP is considered (i.e., $\phi^*_h$ is a one-hot encoding vector), this relative condition number reduces to the density ratio based single-policy concentrability coefficient, $C^*_{t,h,\infty}(\pi_t,\pi_t^b) = \max_{s,a} d^{\pi_t}_{P_h^{(*,t)}}(s,a)/d^{\pi_t^b}_{P_h^{(*,t)}}(s,a)$ \citep{chen2022offline}. The relative condition number $C^*_{t,h}(\pi_t,\pi_t^b)$ is always bounded by the concentrability coefficient \citep{uehara2021pessimistic}. In addition, the relative condition number is computed under the averaging over the state, actions, which could be much smaller than $\max_{s,a} d^{\pi_t}_{P_h^{(*,t)}}(s,a)/d^{\pi_t^b}_{P_h^{(*,t)}}(s,a)$ since the latter will be very large as long as any $s,a$ pair gives large $d^{\pi_t}_{P_h^{(*,t)}}(s,a)/d^{\pi_t^b}_{P_h^{(*,t)}}(s,a)$ ratio.  Therefore, our $C^*_t$ could be much smaller compared to the concentrability coefficient, especially in large-scale MDPs (e.g. continuous state space). Moreover, in our definition of relative condition number, we use the unknown true representation $\phi^*$. Finally, we generalize single task relative condition number to multitask setting by defining $C^*$, by simply taking the maximum over the single-task relative condition numbers.

Now we describe our main theorem.
\begin{theorem}\label{theorem:sample-complexity-upstream}

\begin{enumerate}[label=(\alph*), leftmargin=*, labelsep=5pt]
    \item  Under Assumption \ref{Ass:realizability}, with probability at least $1-\delta$, for any step $h \in [H]$, we have
\begin{equation}\label{Equation:theorem_upstream_TV_distance}
    \frac{1}{T}\sum_{t=1}^T \EE_{{(s_h,a_h)} \atop {\sim (P^{(*,t)},\pi_t^b)}}\left[\left\|\hat{P}_h^{(t)}(\cdot \given s_h,a_h) -P_h^{(*,t)}(\cdot \given s_h, a_h)\right\|_{TV}\right] \leq \sqrt{\frac{2\log(2|\Phi||\Psi|^TnH/\delta)}{nT}},
\end{equation}
where $\hat{\phi}, \hat{P}^{(1)},\ldots, \hat{P}^{(T)}$ be the output of Algorithm~\ref{Algorithm:multi-task-offline}.
\item In addition, in Algorithm~\ref{Algorithm:multi-task-offline}, if we set $\alpha =\sqrt{2n\omega\zeta_n + \lambda d}$, $\lambda = cd\log(|\Phi||\Psi|^TnH/\delta)$ with $\zeta_n := \frac{2\log(2|\Phi||\Psi|^TnH/\delta)}{n}$ and $c$ being a constant, where we assume that 
$\omega := \max_{t}\max_{s,a}(1/\pi_t^b(a\given s)) < \infty$, then under Assumption \ref{Ass:realizability}, with probability at least $1-\delta$, we have
\begin{equation}\label{Equation:theorem_upstream_near_optimal_policies}
    \!\!\!\!\!\!\!\!\frac{1}{T} \!\sum_{t=1}^T\!\! \left[ V^{\pi_t}_{P^{(*,t)},r^t} \!-\! V^{\hat{\pi}_t}_{P^{(*,t)},r^t}\!\right] \!\!\leq\! \omega \alpha d H\sqrt{\frac{C^*}{n}} + 2dH^2\sqrt{\frac{\lambda C^*}{n}} + \omega H^2\sqrt{\frac{dC^*  \zeta_n}{T}} + \alpha \sqrt{\frac{d}{n}} + 2H\sqrt{\frac{\omega  \zeta_n}{T}},
\end{equation}
where $\{\hat{\pi}_t\}_{t \in [T]}$ is the output of \Cref{Algorithm:multi-task-offline}.
\end{enumerate}
\end{theorem}

\Cref{theorem:sample-complexity-upstream} (a) shows a potential benefit of a joint learning of the source-task transition kernels, as compared to independent learning, as measured by the task-average TV distance (defined in the LHS of \Cref{Equation:theorem_upstream_TV_distance}). In particular, to obtain an $\epsilon$-suboptimal transition kernels, it suffices for the joint learning to use $\tilde{O} \left( \frac{\log |\Phi|}{T \epsilon^2} 
 + \frac{\log |\Psi|}{\epsilon^2}\right)$ samples per task, yet for the independent learning to use $\tilde{O} \left( \frac{\log |\Phi|}{ \epsilon^2} 
 + \frac{\log |\Psi|}{\epsilon^2}\right)$ samples per task. This benefit naturally comes from the inductive bias that all the source tasks share a representation in $\Phi$ which can be learned more accurately with the aggregated data pooled from all the source tasks. 
\Cref{Equation:theorem_upstream_near_optimal_policies} in \Cref{theorem:sample-complexity-upstream} further shows that, for all tasks on average, we can uniformly compete with any set of comparator policies $\{\pi_t\}_{t\in[T]}$ satisfying the partial coverage through $C^* < \infty$. In particular, if the optimal policy $\pi_t^*$ is covered by the offline data for all $t \in [T]$, then the output $\{\hat{\pi}_t\}_{t\in[T]}$ of \Cref{Algorithm:multi-task-offline} is able to compete against it on average as well.

\begin{remark}
    We note that in order for the bound to hold in \Cref{theorem:sample-complexity-upstream}, \Cref{Algorithm:multi-task-offline} requires the knowledge of $\omega$ as it is required to set the value of $\alpha$ which is used in the lower confidence bound penalty term defined in \eqref{Eq:LCB}. However, in practice, we expect that $\alpha$ can be treated as a hyperparameter that might be tuned using grid search.
\end{remark}

\textbf{Proof outline. }  Here, we highlight the key steps for the proof of \Cref{theorem:sample-complexity-upstream}. The detailed proof is deferred to \Cref{Appendix:Proof of Multitask Offline Representation Learning}. Using offline multitask MLE lemma (\Cref{Lemma:multitask_mle}) and one-step back lemma (\Cref{Lemma:one-step-back}), we first develop a new upper bound on model estimation error for each task which encapsulates the benefit of joint offline MLE model estimation over single-task offline learning. We then use the following lemma to show near pessimism in the average sense.

\begin{lemma}\label{lemma:bounded_difference_of_summation}
For any  policy $\pi_t$ and reward $r^t$, we have, with probability $1-\delta$
\begin{equation*}
    \frac{1}{T}\sum_{t=1}^T\left[V_{\hat{P}^{(t)},r^t-\hat{b}^{(t)}}^{\pi_t}-V_{P^{(*,t)},r^t}^{\pi_t} \right] \leq H\sqrt{\omega \zeta_n/T},
\end{equation*}
 where $\zeta_n:=\frac{2\log(2|\Phi||\Psi|^TnH/\delta)}{n}$.
\end{lemma}

The proof of \Cref{lemma:bounded_difference_of_summation} relies on simulation lemma and a concentration argument for the penalty term defined in \eqref{Eq:LCB}.
Finally, to obtain a result only depending on the relative condition number using the true representation $\phi^*$ but not the learned feature $\hat{\phi}$, we translate the penalty defined with $\hat{\phi}$ to the potential function $\|\phi^*_{h}(s_{h},a_{h})\|_{(\Sigma_{h,\pi_t^b, \phi^*})^{-1}}$ where $\Sigma_{h,\pi_t^b, \phi^*} = n\EE_{(s_h,a_h)\sim (P^{(*,t)},\pi_t^b)}[\phi^*\phi^{*\top}] + \lambda I$ using a one-step back inequality (\Cref{Lemma:one-step-back}) and a distribution shift lemma (\Cref{Lemma:distribution_shift}).

\section{Downstream RL: Reward-free Exploration, Offline RL and Online RL}

\subsection{Relationship between upstream and downstream MDPs}

In order for us to theoretically study the downstream RL tasks where we would use the learned feature $\hat{\phi}$ from the upstream tasks, first we need to make certain connection between the upstream and downstream MDPs. Naturally, we would have to resort to some assumptions on the transition kernels to make such connections. Next, we describe these assumptions, which we largely adopt from \citet{cheng2022provable}.

\begin{assumption}\label{Assumption:upstream-to-downstream} 

We make the following assumptions 
\begin{enumerate}[leftmargin=*]
    \item For each upstream source task $t$ with transition kernel $P^{(*,t)}$, the behavior policy $\pi_t^b$ is such that $\min_{s\in \cS} \PP_h^{(\pi_t^b,t)}(s) \geq \kappa$, where $\PP_h^{(\pi_t^b,t)}(\cdot):\cS\rightarrow \RR$ is the marginalized state occupancy measure over $\cS$ using policy $\pi_t^b$ at time step $h$.

    \item The state space $\cS$ is compact and has finite measure $1/\nu$, and the induced uniform probability density function is $f(s)=\nu$, $s\in\cS$. 
    
    \item For any two models $P^1(s'|s,a) = \la \phi^1(s,a),\mu^1(s')\ra$ and $P^2(s'|s,a) = \la \phi^2(s,a),\mu^2(s')\ra$ in  the model class $\Phi \times \Psi$, we have, 
    \begin{equation*}
        \|P^1(\cdot|s,a)-P^2(\cdot|s,a)\|_{\text{TV}}  \leq C_R \EE_{(s,a)\sim \cU(\cS,\cA)}\|P^1(\cdot|s,a)-P^2(\cdot|s,a)\|_{\text{TV}},
    \end{equation*}
    for all $(s,a)\in \cS\times\cA$ and $h\in[H]$ where $C_R$ is an absolute constant.

    \item The transition kernel of task $T+1$, $P^{(*,T+1)}$ can be $\xi$-approximated by a linear combination of the $T$ source upstream tasks, i.e. there exist $T$ unknown coefficients $c_1,\ldots,c_T \geq 0$ such that $\sum_{t=1}^T c_t \leq C_L$ and $\xi \geq 0$ such that for all $(s,a) \in \cS\times\cA$ and $h\in[H]$, we have
    \begin{align*}
        \|P^{(*,T+1)}(\cdot|s,a) - \sum_{t=1}^T c_t P^{(*,t)}(\cdot|s,a)\|_{\text{TV}} &\leq \xi.
    \end{align*}
    Here, $\xi$ is called the linear combination misspecification.
\end{enumerate}
    
\end{assumption}

The first point in \Cref{Assumption:upstream-to-downstream} ensures that the behavior policy $\pi^b_t$ for each task $t$ can reach any state in $\cS$ at any time step with a positive probability, an assumption that is previously used in \citep{yin2021near}. Compared to \citet{cheng2022provable}, which assumes the existence of a policy with reachability property in each of the upstream online tasks, ours assumes reachability property for the behavior policies used to collect the upstream offline dataset.

 The third point in \Cref{Assumption:upstream-to-downstream} ensures that for each source task $t$, the point-wise TV error between the learned estimated transition kernel $\hat{P}^{(t)}$ and the true transition kernel $P^{(*,t)}$ is bounded by the population-level TV error. This assumption is necessary to transfer the MLE error from the upstream source tasks to the downstream target task.

Finally, the fourth point in \Cref{Assumption:upstream-to-downstream} connects the upstream source tasks with the downstream target task by assuming that the transition kernel of the target task $P^{(*,T+1)}$ can be approximated by a linear combination of transition kernels of $T$ upstream source tasks. 

The precision of the feature estimation in the upstream has a significant impact on the downstream task's performance because the downstream task utilizes the estimated feature from the upstream. We use the following notion of $\epsilon$-approximate linear MDP to provide a guarantee for the estimated feature.

\begin{definition}[$\epsilon$-approximate linear MDP \citep{jin2019provably, cheng2022provable}]
    For any $\epsilon>0$, we say that MDP $\cM = (\cS,\cA,\cH,P,r)$ is an $\epsilon$-approximate linear MDP with a feature map $\phi_h:\cS\times\cA:\rightarrow \RR^d$, if for any $h\in[H]$, there exist $d$ unknown (signed) measures $\mu_h = (\mu_h^{(1)}, \ldots,\mu_h^{(d)})$ over $\cS$ such that for any $(s,a)\in \cS\times\cA$, we have
    \begin{align}\label{Eq:approximate_feature_map}
        \|P_h(\cdot|s,a) - \la \phi_h(s,a),\mu_h(\cdot)\ra\|_{\text{TV}} \leq \epsilon.
    \end{align}
    Any $\phi$ satisfying ~\eqref{Eq:approximate_feature_map} is called an $\epsilon$-approximate feature map of $\cM$.
\end{definition}

The next lemma shows that the learned feature $\hat{\phi}$ from the upstream offline tasks can approximate the true feature in the new downstream task.
\begin{lemma}\label{Lemma:approximate_feature_new_task}
    Under \Cref{Assumption:upstream-to-downstream}, the output $\hat{\phi}$ of \cref{Algorithm:multi-task-offline} is a $\xi_{\text{down}}$-approximate feature for MDP $\cM^{T+1}$ where $\xi_{\text{down}} = \xi + \frac{C_L C_R\nu}{\kappa} \sqrt{\frac{2T\log(2|\Phi||\Psi|^TnH/\delta)}{n}}$, i.e. there exist a time-dependent unknown (signed) measure $\hat{\mu}^*$ over $\cS$ such that for any $(s,a) \in \cS\times \cA$, we have 
    \begin{equation*}
        \|P_h^{(*,T+1)}(\cdot |s,a) - \la \hat{\phi}_h(s,a), \hat{\mu}_h^*(\cdot)\ra\|_{\text{TV}} \leq \xi_{\text{down}}.
    \end{equation*}
    Furthermore, for any $g:\cS\rightarrow [0,1]$, we have $\|\int \hat{\mu}_h^*(s)g(s)ds\|_2 \leq C_L \sqrt{d}$.
\end{lemma}

\subsection{Downstream Reward-Free RL}

Our goal in this part is to investigate the statistical efficiency of reward-free RL in low-rank MDP while having access to offline datasets from the upstream tasks. 

Our algorithm for the reward-free setting is presented in \Cref{Algorithm:downstream-reward-free-exploration-phase} (exploration phase) and \Cref{Algorithm:downstream-reward-free-planning-phase} (planning phase) which is built on the procedure of optimistic learning as \citet{wang2020reward,zhang2021reward}. While having similar design principle as in \citet{wang2020reward}, our algorithm differs from them due to the misspecification of representation from the upstream task. Thus, the upstream learning error affects the learning accuracy and downstream suboptimality gap and we need to account for that in our analysis. Another difference with \citet{wang2020reward} is that in the exploration phase, like \citet{chen2021near}, we construct more aggressive reward function to avoid overly-conservative exploration, which removes the extra dependency of sample complexity on episode length $H$. Below, we provide our main theorem for downstream reward-free RL task and defer the proof to \Cref{Appendix:Proof for Downstream Reward-Free RL}.

\begin{theorem}\label{Theorem:reward-free-exploration-main}
    Under \Cref{Assumption:upstream-to-downstream}, after collecting $K_{\text{RFE}}$ trajectories during the exploration phase in \Cref{Algorithm:downstream-reward-free-exploration-phase}, with probability at least $1-\delta$, the output of \Cref{Algorithm:downstream-reward-free-planning-phase}, policy $\pi$ satisfies
    \begin{equation}\label{Equation:RFE-theorem-expectation-bound-main}
        \EE_{s_1\sim \mu}[V_1^*(s_1,r) - V_1^\pi(s_1,r)] \leq c'\sqrt{d^3H^4\log(dK_{\text{RFE}}H/\delta)/K_{\text{RFE}}} +  6 H^2\xi_{\text{down}}.
    \end{equation}
    If the linear combination misspecification error $\xi$ in \Cref{Assumption:upstream-to-downstream} satisfies $\tilde{O}(\sqrt{d^3/K_{\text{RFE}}})$ and the number of trajectories in the offline dataset for each upstream task is  at least $\tilde{O}(TK_{\text{RFE}}/d^3)$, then, provided $K_{\text{RFE}}$ is at least $O(H^4d^3\log(dH\delta^{-1}\epsilon^{-1})/\epsilon^2)$,  with probability $1-\delta$, the policy $\pi$ will be an $\epsilon$-optimal policy for any given reward during the planning phase. 
\end{theorem}

    We compare the above result with other algorithms developed for the reward-free RL under low-rank MDPs and summarize the comparison in \Cref{table:bounds}. FLAMBE \citep{agarwal2020flambe} achieves a sample complexity of $\tilde{O}(\frac{H^{22}d^7K^9}{\epsilon^{10}})$ whereas MOFFLE \citep{modi2021model} achieves a sample complexity of $\tilde{O}(\frac{H^7d^{11}K^{14}}{\min\{\epsilon^2\eta,\eta^5\}})$, where $\eta$ is a reachability probability to all states. More recently, \citet{cheng2023improved} proposed RAFFLE which has the best-known sample complexity of $\tilde{O}(\frac{H^5d^4K}{\epsilon^2})$.\footnote{We rescale the result in \citep{cheng2023improved} by a factor of $H^2$ as we do not assume the sum of rewards to be within $[0,1]$.} \citet{cheng2023improved} further shows that the dependence of sample complexity on action space cardinality $K$ is unavoidable when performing reward-free exploration in low-rank MDPs from which they conclude that it is strictly harder to find a near-optimal policy under low-rank MDPs than under linear MDPs. However, as we see from \Cref{Theorem:reward-free-exploration-main}, by using estimated representation from the upstream offline datasets, we can avoid this dependence of sample complexity on  $K$ and overall improve the sample complexity by $\tilde{O}(HdK)$ compared to that of RAFFLE. Moreover, compared to standard linear MDP, where the true representation $\{\phi_h\}_{h=1}^H$ is known \citep{jin2019provably}, the suboptimality gap in \eqref{Equation:RFE-theorem-expectation-bound-main} contains an additional term $H^2\xi_{\text{down}}$ which is due to the upstream misspecification error $\xi_{\text{down}}$. When $\xi_{\text{down}}$ is small enough, our resulting sample complexity of $\tilde{O}(\frac{H^4d^3}{\epsilon^2})$ matches the reward-free exploration sample complexity for linear mixture MDPs \citep{chen2021near}, which is worse off by only $\tilde{O}(d)$ compared to the best known sample complexity of $\tilde{O}(\frac{H^4d^2}{\epsilon^2})$ \citep{hu2022towards} for linear MDP.

\begin{table*}[!htpb]

\begin{center}
\begin{tabular}{l  a  b  a  b  a}
\hline
\rowcolor{LightCyan} Algorithm&Sample Complexity & Task\\ \hline
FLAMBE \citep{agarwal2020flambe} & $\tilde{O}(\frac{H^{22}d^7K^9}{\epsilon^{10}})$ &  Single task  \\ \hline
MOFFLE \citep{modi2021model} & $\tilde{O}(\frac{H^7d^{11}K^{14}}{\min\{\epsilon^2\eta,\eta^5\}})$ &  Single task \\ \hline

RAFFLE \citep{cheng2023improved} & $\tilde{O}(\frac{H^5d^4K}{\epsilon^2})$ & Single task  \\ \hline
This work (\Cref{Algorithm:downstream-reward-free-exploration-phase} and \Cref{Algorithm:downstream-reward-free-planning-phase}) & $\tilde{O}(\frac{H^4d^3}{\epsilon^2})$ & Multi-task  \\\hline 
\end{tabular}
\end{center}
\caption{Sample complexities of different approaches to learning an $\epsilon$-optimal policy for the reward-free RL setting with low-rank MDPs. 
\label{table:bounds}}
\end{table*}

\subsection{Downstream Offline and Online RL}

For completeness, we also consider downstream offline and online RL which was previously studied in \citet{cheng2022provable} to show the effectiveness of our offline representation. In both cases, we assume that the reward function $r^{T+1}$ in the downstream task $T+1$ is linear with respect to the unknown feature $\phi^*: \cS\times \cA \rightarrow \RR^d$. We emphasize that, unlike \citet{cheng2022provable}, we assume the reward function $r^{T+1}$ is unknown.

\textbf{Offline RL. } For downstream offline RL task, similar to \citet{cheng2022provable}, we use standard pessimistic value iteration algorithm (PEVI) \citep{jin2021pessimism} with approximate feature learned from upstream task. 
We make the following data-coverage type of assumption which is standard in the study of offline RL \citep{xie2021policy,DBLP:conf/iclr/WangFK21,yin2021near}. Moreover, this assumption has been shown to be necessary for sample efficient offline RL for tabular and linear MDPs \citep{DBLP:conf/iclr/WangFK21,yin2021towards}.

\begin{assumption}[Feature coverage]\label{Assumption:feature_coverage}
    There exists an absolute constant $\kappa_\rho$ such that for all $h\in[H]$ and $\phi_h \in \Phi_h$, $\lambda_{\min}(\EE_\rho[{\phi}_h(s_h,a_h){\phi}_h(s_h,a_h)^\top | s_1 = s])\geq \kappa_\rho$.
\end{assumption}
Next, we provide our result for the downstream offline RL task and defer the proof to \Cref{Appendix:Proof for Downstream Offline RL}.
\begin{theorem}\label{Theorem:offline_downstream}
    Under \Cref{Assumption:upstream-to-downstream}, setting $\lambda_d = 1$, $\beta = O(Hd\sqrt{\iota}+H\sqrt{dN_{\text{off}}}\xi_{\text{down}})$, where $\iota = \log(HdN_{\text{off}}\max(\xi_{\text{down}},1)/\delta)$, with probability at least $1-\delta$, the suboptimality gap of \Cref{Algorithm:downstream-offline-RL} is at most \begin{equation}\label{Equation:suboptimality_gap_behavior_policy_theorem_offline_RL}
        V^{\pi^*}_{P^{(*,T+1)},r}(s) - V^{\hat{\pi}}_{P^{(*,T+1)},r}(s) \leq 2H^2\xi_{\text{down}} + 2\beta \sum_{h=1}^H \EE_{\pi^*}\bigg[\big\|\hat{\phi}_h(s_h,a_h)\big\|_{\Lambda_h^{-1}} \big| s_1 = s\bigg].
    \end{equation}
    Additionally if \Cref{Assumption:feature_coverage} holds, and the sample size satisfies $N_{\text{off}} \geq 40/\kappa_\rho \cdot \log(4dH/\delta)$, then with probability $1-\delta$, we have,
    \small{
    \begin{equation}\label{Equation:refined-downstream-offline-suboptimality-bound}
        V^{\pi^*}_{P^{(*,T+1)},r}(s) - V^{\hat{\pi}}_{P^{(*,T+1)},r}(s) \leq O\bigg(\kappa_\rho^{-1/2}H^2d\sqrt{\frac{\log(HdN_{\text{off}}\max(\xi_{\text{down}},1)/\delta)}{N_{\text{off}}}}  + \kappa_\rho^{-1/2}H^2d^{1/2}\xi_{\text{down}}\bigg).
    \end{equation}
    }
\end{theorem}

\textbf{Online RL. } For downstream online RL task, where the agent is allowed to interact with the new task MDP $\cM^{T+1}$ for policy optimization, similar to \citet{cheng2022provable}, we use standard LSVI-UCB algorithm \citep{jin2019provably} with approximate feature. We next provide our result for downstream online RL task and defer the proof to \Cref{Appendix:Proof for Downstream Online RL}.

\begin{theorem}\label{Theorem:downstream_online_task}
    Let $\tilde{\pi}$ be the uniform mixture of $\pi^1,\ldots,\pi^{N_{on}}$ in \Cref{Algorithm:downstream-online-RL}. Under \Cref{Assumption:upstream-to-downstream}, setting $\lambda = 1$, $\beta_n= O(Hd\sqrt{\iota_n(\delta)}+ H\sqrt{dn}\xi_{\text{down}}+ C_L \sqrt{Hd})$, where $\iota_n = \log(Hdn\max(\xi_{\text{down}},1)/\delta)$, with probability  $1-\delta$, the suboptimality gap of \Cref{Algorithm:downstream-online-RL} satisfies
    \begin{align*}
        V_{P^{(*,T+1)},r}^* -V_{P^{(*,T+1)},r}^{\tilde{\pi}} &\leq \tilde{O} (H^2d^{3/2}N_{\text{on}}^{-1/2} + H^2 d\xi_{\text{down}}).
    \end{align*}
\end{theorem}

\section{Related Work}

\textbf{Offline Reinforcement Learning. } Offline RL \citep{ernst2005tree,riedmiller2005neural,lange2012batch,levine2020offline} studies the problem of learning a policy from a static dataset without interacting with the environment. The key challenge in offline RL is the insufficient coverage of the dataset, due to the lack of exploration \citep{levine2020offline,liu2020provably}. One prevalent approach to address this challenge is the pessimism principle to penalize the estimated value of the under-covered state-action pairs. There have been extensive studies on incorporating pessimism into the development of different approaches in single-task offline RL, including model-based approach \citep{Rashidinejad2021BridgingOR,uehara2022pessimistic,jin2021pessimism,yu2020mopo,xie2021policy,uehara2021representation,yin2022near}, model-free approaches 
\citep{kumar2020conservative,wu2021uncertainty,bai2022pessimistic,ghasemipour2022so,yan2023efficacy,nguyen-tang2022on,nguyen2023instance,nguyen-tang2023viper}, and policy-based approach \citep{rezaeifar2022offline,xie2021bellman,zanette2021provable,nguyen-tang2024on,nguyen2024sample}. Our algorithm for upstream offline multitask RL is inspired by the uncertainty-based pessimism methods in single-task offline RL.

\noindent{\bf Low-rank MDPs.} \citet{agarwal2020flambe} initiates the study of low-rank MDPs. \citet{uehara2021representation} proposed model-based algorithms for both online and offline RL, while \citet{modi2021model} put forward a model-free algorithm for low-rank MDPs. Moreover, \citet{du2019provably,misra2020kinematic,zhang2022efficient} studied block MDPs, which is a special case of low-rank MDPs.

\textbf{Offline Data Sharing in RL.}
There has been several empirical works that investigated the benefits of using offline datasets from multiple tasks to accelerate downstream learning \citep{eysenbach2020rewriting,kalashnikov2021mt,mitchell2021offline,yu2021conservative,yoo2022skills}. \citet{yu2021conservative} show that selectively sharing data between tasks can be helpful for offline multitask learning. For instance, earlier studies have investigated the development of data sharing strategies through human domain knowledge \citep{kalashnikov2021mt}, inverse RL \citep{reddy2019sqil,eysenbach2020rewriting,li2020generalized}, and estimated Q-values \citep{yu2021conservative}. More recently, \citet{xu2022feasibility} uses offline dataset from diverse tasks to perform offline multitask pretraining of a world model which is then finetuned  on a downstream target task.
\citet{hu2023provable} proposes a provably efficient self-supervised offline data-sharing algorithm for linear MDP. However, they assume access to reward-free data.

\textbf{Comparison to \citet{cheng2022provable}.} Closest to our work is \citet{cheng2022provable} who studied online multitask RL. Their proposed REFUEL algorithm combines design principles from FLAMBE \citep{agarwal2020flambe} and REP-UCB \citep{uehara2021representation} and performs joint MLE based model learning while collecting data in an online manner.  On the contrary, MORL first performs joint MLE based model learning using the offline dataset collected for each source task and then upon constructing penalty terms for each task, performs planning using pessimistic reward functions. While both works rely on an MLE oracle, first proposed in \citet{agarwal2020flambe} for single-task RL, our proposed  offline multitask MLE lemma (\Cref{Lemma:multitask_mle}) conveys fundamentally very different ideas compared to its online counterpart, Lemma 3 in \citet{cheng2022provable}. Lemma 3 in \citet{cheng2022provable} says that when exploration policies for each upstream online source task is uniformly chosen, the summation of the estimation error of transition probability can be bounded with high probability. On the contrary, our \Cref{Lemma:multitask_mle} states that when the offline datasets for each of the upstream offline source tasks are collected using respective behavior policies, the summation of the estimation error is bounded with high probability.
For the downstream task, \citet{cheng2022provable} studies only offline and online RL task, whereas our primary contribution in this part is in the study of downstream reward-free RL which has not been previously studied in the context of multitask representation learning. For completeness, we provide results in downstream offline and online RL setting as a complementary result. Moreover, unlike us, \citet{cheng2022provable} assumes that the reward-function is known in the downstream task, which is a fairly strong assumption. Somewhat in a contrived manner, in \citet{cheng2022provable}, the reward-function is further assumed to be general and not necessarily linear in the feature which complicates their downstream analysis. On the contrary, we assume that the reward function is linear with respect to the feature. Finally, \citet{cheng2022provable} assumes that for each episode in any task MDP, the sum of reward is normalized to be within $[0,1]$. We do not make this assumption for fair comparison to literature on reward-free RL for low-rank MDPs. 

\textbf{Comparison to Concurrent Work.} In a concurrent work, \citet{bose2024offline} studies representational transfer in offline low-rank RL. In the upstream task, similar to ours, \citet{bose2024offline} also uses an offline MLE oracle. Compared to our \Cref{theorem:sample-complexity-upstream}, where we provide bound for average accuracy of the estimated transition kernels of the upstream source tasks, \citet{bose2024offline} provides bound for the sum of the point-wise errors in the transition dynamics averaged over the points in the source datasets. To bound the representational transfer error in the downstream target task, they introduce a notion of neighborhood occupancy density. Moreover, to connect the upstream tasks and the downstream target task, they make a pointwise linear span assumption from \citet{agarwal2023provable}. Finally, for the downstream target task, they only consider offline setting, whereas our primary focus and contribution is in the study of reward-free setting in the downstream task.

\vspace*{-7pt}
\section{Conclusion}
\vspace*{-5pt}
In this paper, we theoretically study multitask RL in the offline setting. We show that offline multitask representation learning is provably more sample efficient than learning each task individually. We further show the benefit of employing the learned representation from the upstream to learn a near-optimal policy of a new downstream task, in reward-free, offline and online setting, that shares the same representation. We believe our work will open up many promising directions for future work, for example, studying the general function class representation learning in offline multitask setting.

\subsubsection*{Acknowledgments}
R. Arora's and T. Nguyen-Tang's research was supported, in part, by the NSF CAREER award IIS-1943251. M. Yin and M. Wang acknowledge the support by NSF IIS-2107304, NSF CPS-2312093, and ONR 1006977.
The authors would like to thank Yingbin Liang and Yu-Xiang Wang
for their helpful discussions.

\bibliographystyle{plainnat}
\bibliography{references}

\appendix

\newpage
\tableofcontents
\newpage

\section{Additional Related Work}\label{Additional_related_work}

\textbf{Reward-Free Exploration (RFE).}
In reward-free RL setting, the agent does not have access to a reward function during exploration phase. However, the agent must propose a near-optimal policy for an arbitrary reward function revealed only after the initial exploration phase. This setting is particularly relevant when there are multiple reward functions of interest \citep{achiam2017constrained} or in the batch RL setting \citep{ernst2005tree}. In recent years, reward-free RL has been extensively studied in both tabular \citep{jin2020reward,kaufmann2021adaptive,menard2021fast} and linear function approximation \citep{wang2020reward,zanette2020provably,zhang2021reward,hu2022towards,wagenmaker2022reward} settings. \citet{agarwal2020flambe,modi2021model, chen2022statistical,cheng2023improved} study reward-free RL in low-rank MDPs which is particularly interesting as here  representation learning is interwined with reward-free exploration.

\section{Omitted Algorithms}

\subsection{Algorithms for Downstream Reward-Free RL}
\begin{algorithm}
\caption{Downstream Reward-Free Exploration: Exploration Phase}\label{Algorithm:downstream-reward-free-exploration-phase}
\begin{algorithmic}[1]
\STATE \textbf{Input:} Feature $\hat{\phi}$, Failure probability $\delta > 0$ and target accuracy $\varepsilon >0$
\STATE $K_{\text{RFE}} \leftarrow c_K \cdot d^3H^4\log(dH\delta^{-1}\varepsilon^{-1})/\varepsilon^2$ for some $c_K > 0$
\STATE $\beta \leftarrow C_L H\sqrt{d} + dH\sqrt{\log(dK_{\text{RFE}}H\max(\xi_{\text{down}},1)/\delta)} + H\xi_{\text{down}}\sqrt{dK_{\text{RFE}}}$
\FOR{$k=1, \ldots,K_{\text{RFE}}$}
\STATE $\hat{Q}_{H+1}^k (\cdot,\cdot) \leftarrow 0$, $\hat{V}_{H+1}^k(\cdot) \leftarrow 0$
\FOR{$h = H,H-1, \ldots, 1$}
\STATE $\Lambda_h^k = \sum_{\tau=1}^{k-1} \hat{\phi}_h(s_h^\tau,a_h^\tau)\hat{\phi}_h(s_h^\tau,a_h^\tau)^\top +  I_d$
\STATE $u_h^k(\cdot,\cdot) \leftarrow \beta \sqrt{\hat{\phi}(\cdot,\cdot)^\top (\Lambda_h^k)^{-1}\hat{\phi}(\cdot,\cdot)}$
\STATE Define the exploration-driven reward function $r_h^k(\cdot,\cdot) \leftarrow u_h^k(\cdot,\cdot)$
\STATE $\hat{w}_h^k = (\Lambda_h^k)^{-1}\sum_{\tau=1}^{k-1} \hat{\phi}_h(s_h^\tau,a_h^\tau)\hat{V}_{h+1}^k(s_{h+1}^\tau)$
\STATE $\hat{Q}^k_h(\cdot,\cdot) = \min\{\hat{\phi}_h(\cdot,\cdot)^\top \hat{w}^k_h + r^k_h(\cdot,\cdot) +u_h^k(\cdot,\cdot),H\}$
\STATE $\hat{V}_h^k(\cdot) = \max_{a\in \cA} \hat{Q}_h^k(\cdot,a)$
\STATE $\pi_h^k(\cdot) = \argmax_{a\in\cA} \hat{Q}_h^k(\cdot, a)$
\ENDFOR
\STATE Receive initial state $s_1^k \sim \mu$
\FOR{$h=1,\ldots,H$}
\STATE Take action $a_h^k = \pi_h^k(s_h^k)$ and observe $s_{h+1}^k \sim P_h^{(*,T+1)}(s_h^k,a_h^k)$
\ENDFOR
\ENDFOR

\STATE \textbf{Output:} $\cD_{\text{RFE}} \leftarrow \{(s_h^k,a_h^k)\}_{(k,h)\in [K]\times[H]}$
\end{algorithmic}
\end{algorithm}

\begin{algorithm}
\caption{Downstream Reward-Free Exploration: Planning Phase}\label{Algorithm:downstream-reward-free-planning-phase}
\begin{algorithmic}[1]
\STATE \textbf{Input:} Feature $\hat{\phi}$, dataset $\cD_{\text{RFE}} = \{(s_h^k,a_h^k)\}_{k \in [K_{\text{RFE}}], h \in [H]}$, reward functions $r = \{r_h\}_{h\in[H]}$
\STATE \textbf{Initialization:} $\hat{Q}_{H+1}(\cdot,\cdot) \leftarrow 0$ and $\hat{V}_{H+1}(\cdot) \leftarrow 0$
\FOR{$h = H,H-1, \ldots, 1$}
\STATE $\Lambda_h = \sum_{\tau=1}^{K_{\text{RFE}}} \hat{\phi}_h(s_h^\tau,a_h^\tau)\hat{\phi}_h(s_h^\tau,a_h^\tau)^\top +  I_d$
\STATE Let $u_h (\cdot, \cdot) \leftarrow \min\left\{\beta \sqrt{\hat{\phi}(\cdot,\cdot)^\top (\Lambda_h)^{-1}\hat{\phi}(\cdot,\cdot)},H\right\}$
\STATE $\hat{w}_h \leftarrow \Lambda_h^{-1}\sum_{\tau=1}^{K_{\text{RFE}}} \hat{\phi}_h(s_h^\tau,a_h^\tau)\hat{V}_{h+1}(s_{h+1}^\tau)$
\STATE $\hat{Q}_h(\cdot,\cdot) \leftarrow \min\{ \hat{\phi}_h(\cdot,\cdot)^\top \hat{w}_h + r_h(\cdot,\cdot) + u_h(\cdot,\cdot) ,H\}$ and $\hat{V}_h(\cdot) \leftarrow \max_{a\in \cA} \hat{Q}_h(\cdot, a)$
\STATE $\hat{\pi}_h(\cdot) \leftarrow \argmax_{a \in \cA} \hat{Q}_h (\cdot,a)$
\ENDFOR
\STATE \textbf{Output:} $\hat{\pi} = \{\hat{\pi}_h\}_{h\in [H]}$
\end{algorithmic}
\end{algorithm}

\newpage

\subsection{Algorithm for Downstream Offline RL}

\begin{algorithm}
\caption{Pessimistic Value Iteration (PEVI) with Approximate Feature \citep{jin2021pessimism}}\label{Algorithm:downstream-offline-RL}
\begin{algorithmic}[1]
\STATE \textbf{Input:} Feature $\hat{\phi}$, dataset $\cD_{\text{down}} = \{(s_h^\tau,a_h^\tau, r_h^\tau, s_{h+1}^\tau)\}_{\tau \in [N_{\text{off}}], h \in [H]}$, parameters $\lambda_d, \beta, \xi_{\text{down}}$
\STATE \textbf{Initialization:} $\hat{V}_{H+1}= 0$
\FOR{$h = H,H-1, \ldots, 1$}
\STATE $\Lambda_h = \sum_{\tau=1}^{N_{\text{off}}} \hat{\phi}_h(s_h^\tau,a_h^\tau)\hat{\phi}_h(s_h^\tau,a_h^\tau)^\top + \lambda_d I_d$
\STATE $\hat{w}_h = \Lambda_h^{-1}(\sum_{\tau=1}^{N_{\text{off}}} \hat{\phi}_h(s_h^\tau,a_h^\tau)\cdot ( r_h^\tau + \hat{V}_{h+1}(s_{h+1}^\tau)))$
\STATE $\Gamma_h(\cdot,\cdot) = H\xi_{\text{down}} + \beta[\hat{\phi}_h^\top \Lambda_h^{-1}\hat{\phi}_h(\cdot,\cdot)]^{1/2}$
\STATE $\hat{Q}_h(\cdot,\cdot) = \min\{\hat{\phi}_h(\cdot,\cdot)^\top \hat{w}_h - \Gamma_h(\cdot,\cdot),H - h + 1\}^+$ 
\STATE $\hat{\pi}_h(\cdot|\cdot) = \argmax_{\pi_h}\la \hat{Q}_h (\cdot,\cdot), \pi_h(\cdot|\cdot)\ra_\cA$
\STATE $\hat{V}_h(\cdot) = \la \hat{Q}_h(\cdot,\cdot), \hat{\pi}_h(\cdot|\cdot)\ra_\cA$
\ENDFOR
\STATE \textbf{Output:} $\{\hat{\pi}_h\}_{h=1}^H$
\end{algorithmic}
\end{algorithm}

\subsection{Algorithm for Downstream Online RL}

\begin{algorithm}
\caption{LSVI-UCB with Approximate Feature \citep{jin2019provably}}\label{Algorithm:downstream-online-RL}
\begin{algorithmic}[1]
\STATE \textbf{Input:} Feature $\hat{\phi}$, parameters $\lambda_d, \beta_n, \xi_{\text{down}}$
\FOR{$n=1, \ldots,N_{\text{on}}$}
\STATE Receive the initial state $s_1^n = s_1$
\FOR{$h = H,H-1, \ldots, 1$}
\STATE $\Lambda_h^n = \sum_{\tau=1}^{n-1} \hat{\phi}_h(s_h^\tau,a_h^\tau)\hat{\phi}_h(s_h^\tau,a_h^\tau)^\top + \lambda_d I_d$
\STATE $\hat{w}_h^n = (\Lambda_h^n)^{-1}\sum_{\tau=1}^{n-1} \hat{\phi}_h(s_h^\tau,a_h^\tau)[r_h(s_h^\tau, a_h^\tau) + \hat{V}_{h+1}^n(s_{h+1}^\tau)]$
\STATE $\hat{Q}^n_h(\cdot,\cdot) = \min\{ \hat{\phi}_h(\cdot,\cdot)^\top \hat{w}^n_h +\beta_n\|\hat{\phi}(\cdot,\cdot)\|_{(\Lambda_h^n)^{-1}},H-h+1\}^+$
\STATE $\hat{V}_h^n(\cdot) = \max_{a\in \cA} \hat{Q}_h^n(\cdot,a)$
\ENDFOR
\STATE $\pi_h^n(\cdot) = \argmax_{a\in\cA} \hat{Q}_h^n(\cdot, a)$
\FOR{$h=1,\ldots,H$}
\STATE Take action $a_h^n = \pi^n(s_h^n)$ and observe $s_{h+1}^n$
\ENDFOR
\ENDFOR

\STATE \textbf{Output:} $\pi^1, \ldots,\pi^{N_\text{on}}$ and $\tilde{\pi}$ where $\tilde{\pi}$ is the uniform mixture of $\pi^1,\ldots,\pi^{N_{on}}$
\end{algorithmic}
\end{algorithm}

\newpage

\section{Notations}\label{section:frequently_used_notations}
We summarize frequently used notations in the following list.

\begin{align*}
\begin{array}{ll}
f_h^{(t)}(s,a) &\|\hat{P}_h^{(t)}(\cdot \given s,a) - P_h^{(\star,t)}(\cdot \given s,a)\|_{TV} \\
\zeta_n & \frac{2\log(2|\Phi||\Psi|^TnH/\delta)}{n}\\
\zeta_h^{(t)} & \EE_{(s_h,a_h) \sim (P^{(*,t)},\pi_t^b)} \left[f_h^{(t)}(s,a)^2\right]\\
\zeta_1^{(t)} &\EE_{(s_1,a_1) \sim (P^{(*,t)},\pi_t^b)}\left[f_1^{(t)}(s_1,a_1)^2\right]\\
\alpha_h^{(t)} &\sqrt{n\omega \zeta_h^{(t)} + \lambda d + n\zeta_{h-1}^{(t)}}\\
\alpha &\sqrt{2n\omega\zeta_n + \lambda d}\\
\Sigma_{h,\pi_t^b, \phi} & n\EE_{(s_h,a_h)\sim (P^{(*,t)},\pi_t^b)}[\phi\phi^\top] + \lambda I\\
\hat{\Sigma}_{h,\phi}^{(t)}   &\sum_{i=1}^n \phi_h(s_h^{(i,t)},a_h^{(i,t)})\phi_h(s_h^{(i,t)},a_h^{(i,t)})^\top + \lambda I\\
\hat{b}_h^{(t)}(s_h, a_h)  &\min \left\{\alpha \|\hat{\phi}_h(s_h,a_h)\|_{\hat{\Sigma}_{h,\hat{\phi}}^{(t)}}, 1\right\}\\
\omega_t & \max_{s,a}(1/\pi_t^b(a\given s))\\
\omega & \max_{t} \omega_t\\
\xi_{\text{down}} & \xi + \frac{C_L C_R\nu}{\kappa} \sqrt{\frac{2T\log(2|\Phi||\Psi|^TnH/\delta)}{n}}\\
\hat{\mu}^*(\cdot) & \sum_{t=1}^T c_t\hat{\mu}^{(t)}(\cdot)\\
\hat{w}_h^* & \int_{s'}\hat{\mu}^*(s')\hat{V}_{h+1}(s')ds'\\
\end{array}
\end{align*}

For any $h \in [H]$, we define 
\begin{align*}
    &P_h^{(*,T+1)}(\cdot|s,a) = \la \phi_h^*(s,a), \mu_h^{(*,T+1)}(\cdot)\ra,\\
    &\overline{P}_h(\cdot|s,a) = \la \hat{\phi}_h(s,a), \hat{\mu}_h^*(\cdot)\ra.
\end{align*}

Given a reward function $r$, for any function $f:\cS\rightarrow \RR$ and $h \in [H]$, we define the transition operators and their corresponding Bellman operators as follows

\begin{align*}
    &(P_h^{(*,T+1)}f)(s,a) = \int_{s'} \la \phi_h^*(s,a), \mu_h^{(*,T+1)}(s')\ra f(s')ds',\\
    & (\BB_h f)(s,a) = r_h(s,a) + (P_h^{(*,T+1)}f)(s,a)\\
    & (\overline{P}_h f)(s,a) = \int_{s'} \la \hat{\phi}_h(s,a), \hat{\mu}_h^*(s')\ra f(s')ds',\\
    & (\overline{B}_h f)(s,a) = r_h(s,a) +(\overline{P}_h f)(s,a)\\ 
\end{align*}

\newpage

\section{Proof of Multitask Offline Representation Learning}\label{Appendix:Proof of Multitask Offline Representation Learning}

We first state the supporting lemmas that are used in the proof of \Cref{theorem:sample-complexity-upstream}. The proof of these lemmas are provided in \Cref{Section:proof of supporting lemmas of multitask offline rl}.

\subsection{Supporting Lemmas}

In the following lemma, we provide an upper bound for the model estimation error for each task that captures the advantage of joint MLE model estimation over single-task learning.

\begin{lemma}\label{lemma:model-estimation-error}
For any task $t$, policy $\pi_t$, and reward $r_t$, we have for all $h \geq 2$,
\begin{align*}
    \EE_{(s_h, a_h) \sim (\hat{P}^{(t)},\pi_t)}\left[f_h^{(t)}(s_h,a_h)\right] &\leq \EE_{(s_{h-1},a_{h-1})\sim (\hat{P}^{(t)}, \pi_t)}\left[\alpha_h^{(t)}\left\|\hat{\phi}_{h-1}(s_{h-1},a_{h-1})\right\|_{(\Sigma_{h-1,\pi_t^b, \hat{\phi}})^{-1}}\right],
\end{align*}
and for $h =1$, we have
\begin{align*}
    \EE_{a_1 \sim \pi_t} \left[f_1^{(t)}(s_1,a_1)\right] &\leq \sqrt{\omega_t\zeta_1^{(t)}},
\end{align*}
where $\omega_t = \max_{s,a}(1/\pi_t^b(a\given s))$.

\end{lemma}

In the next lemma, we prove that $V_{\hat{P}^{(t)},r^t-\hat{b}^{(t)}}^{\pi_t}$ is an almost pessimistic estimator of $V_{P^{(*,t)},r^t}^{\pi_t}$ in the average sense. 
\begin{lemma}[Restatement of \Cref{lemma:bounded_difference_of_summation}]\label{lemma_restatement:bounded_difference_of_summation}
For any  policy $\pi_t$ and reward $r^t$, we have, with probability $1-\delta$
\begin{equation*}
    \frac{1}{T}\sum_{t=1}^T\left[V_{\hat{P}^{(t)},r^t-\hat{b}^{(t)}}^{\pi_t}-V_{P^{(*,t)},r^t}^{\pi_t} \right] \leq H\sqrt{\omega \zeta_n/T}, \quad\mbox{ where } \zeta_n:=\frac{2\log(2|\Phi||\Psi|^TnH/\delta)}{n}
\end{equation*}
\end{lemma}

\subsection{Proof of \Cref{theorem:sample-complexity-upstream}}

We first restate \Cref{theorem:sample-complexity-upstream}.

\begin{theorem}[Restatement of \Cref{theorem:sample-complexity-upstream}]\label{theorem_restate:sample-complexity-upstream}

Under Assumption \ref{Ass:realizability}, with probability at least $1-\delta$, for any step $h \in [H]$, we have
\begin{equation}\label{Equation:restatement-theorem_upstream_TV_distance}
    \frac{1}{T}\sum_{t=1}^T \EE_{{(s_h,a_h)} \atop {\sim (P^{(*,t)},\pi_t^b)}}\left[\left\|\hat{P}_h^{(t)}(\cdot \given s_h,a_h) -P_h^{(*,t)}(\cdot \given s_h, a_h)\right\|_{TV}\right] \leq \sqrt{\frac{2\log(2|\Phi||\Psi|^TnH/\delta)}{nT}},
\end{equation}
where $\hat{\phi}, \hat{P}^{(1)},\ldots, \hat{P}^{(T)}$ be the output of Algorithm~\ref{Algorithm:multi-task-offline}. 

In addition, in Algorithm~\ref{Algorithm:multi-task-offline}, if we set $\alpha =\sqrt{2n\omega\zeta_n + \lambda d}$, $\lambda = cd\log(|\Phi||\Psi|^TnH/\delta)$ with $\zeta_n := \frac{2\log(2|\Phi||\Psi|^TnH/\delta)}{n}$ and $c$ being a constant term, where we assume that 
$\omega := \max_{t}\max_{s,a}(1/\pi_t^b(a\given s)) < \infty$, then under Assumption \ref{Ass:realizability}, with probability at least $1-\delta$, we have
\begin{equation}\label{Equation:restatement-theorem_upstream_near_optimal_policies}
    \frac{1}{T} \sum_{t=1}^T \left[ V^{\pi_t}_{P^{(*,t)},r^t} - V^{\hat{\pi}_t}_{P^{(*,t)},r^t}\right] \leq \omega \alpha d H\sqrt{\frac{C^*}{n}} + 2dH^2\sqrt{\frac{\lambda C^*}{n}} + \omega H^2\sqrt{\frac{dC^*  \zeta_n}{T}}  + \alpha \sqrt{\frac{d}{n}} + 2H\sqrt{\frac{\omega  \zeta_n}{T}},
\end{equation}
where $\{\hat{\pi}_t\}_{t \in [T]}$ is the output of \Cref{Algorithm:multi-task-offline}.

\end{theorem}

\begin{proof}[Proof of \cref{theorem_restate:sample-complexity-upstream}]
    As in \Cref{lemma_restatement:bounded_difference_of_summation}, we condition on the events:
    \begin{equation}\label{Eq:MLE-thm-proof}
         \sum_{t=1}^T \EE_{(s_h,a_h) \sim (P^{(*,t)},\pi_t^b)} \left[f_h^{(t)}(s,a)^2\right] \leq \zeta_n , \quad\mbox{ where } \zeta_n:=\frac{2\log(2|\Phi||\Psi|^TnH/\delta)}{n},
    \end{equation}
    and 
    \begin{equation}\label{Eq:concentration_penalty_thm}
        \forall \phi \in \Phi: \left\|\phi_h(s,a)\right\|_{(\hat{\Sigma}_{h,\phi}^{(t)})^{-1}} = \Theta \left(\left\|\phi_h(s,a)\right\|_{(\Sigma_{h,\pi_t^b, \phi})^{-1}}\right).
    \end{equation}

    From \Cref{Lemma:multitask_mle} and \Cref{Lemma:concentration_penalty_term}, this event happens with probability $1-\delta$. Conditioning on this event, we have
    \begin{align*}
        &\frac{1}{T}\sum_{t=1}^T \EE_{(s_h,a_h) \sim (P^{(*,t)},\pi_t^b)}\left[\left\|\hat{P}_h^{(t)}(\cdot \given s_h,a_h) -P_h^{(*,t)}(\cdot \given s_h, a_h)\right\|_{TV}\right]\\
        &=\frac{1}{T}\sum_{t=1}^T \EE_{(s_h,a_h) \sim (P^{(*,t)},\pi_t^b)}\left[f_h^{(t)}(s,a)\right]\\
        &\overset{(\romannumeral1)}{\leq} \frac{1}{T}\sqrt{T \sum_{t=1}^T \left(\EE_{(s_h,a_h) \sim (P^{(*,t)},\pi_t^b)}\left[f_h^{(t)}(s,a)\right]\right)^2}\\
        &\overset{(\romannumeral2)}{\leq} \frac{1}{T}\sqrt{T \sum_{t=1}^T \EE_{(s_h,a_h) \sim (P^{(*,t)},\pi_t^b)}\left[f_h^{(t)}(s,a)^2\right]}\\
        &\overset{(\romannumeral3)}{\leq} \sqrt{\frac{\zeta_n}{T}}\\
        &= \sqrt{\frac{2\log(2|\Phi||\Psi|^TnH/\delta)}{nT}},
    \end{align*}
    where $(i)$ follows from Cauchy-Schwarz inequality, $(ii)$ follows from Jensen's inequality and $(iii)$ follows from ~\eqref{Eq:MLE-thm-proof}. This completes the first part of the proof.
    
    Conditioning on the event in ~\eqref{Eq:MLE-thm-proof} and ~\eqref{Eq:concentration_penalty_thm}, for any set of policies $\pi_1, \ldots, \pi_T$, we have
    \begin{align}\label{Eq:difference_v_pi_hat_pi}
        & \sum_{t=1}^T \left[V_{P^{(*,t)},r^t}^{\pi_t} - V_{P^{(*,t)},r^t}^{\hat{\pi}_t}\right]\nonumber\\
        &= \sum_{t=1}^T \left[V_{P^{(*,t)},r^t}^{\pi_t} - V_{\hat{P}^{(t)},r^t-\hat{b}^{(t)}}^{\hat{\pi}_t} + V_{\hat{P}^{(t)},r^t-\hat{b}^{(t)}}^{\hat{\pi}_t} -V_{P^{(*,t)},r^t}^{\hat{\pi}_t}\right]\nonumber\\
        &\overset{(\romannumeral1)}{\leq} \sum_{t=1}^T \left[V_{P^{(*,t)},r^t}^{\pi_t} - V_{\hat{P}^{(t)},r^t-\hat{b}^{(t)}}^{\pi_t} + V_{\hat{P}^{(t)},r^t-\hat{b}^{(t)}}^{\hat{\pi}_t} -V_{P^{(*,t)},r^t}^{\hat{\pi}_t}\right]\nonumber\\
        &\overset{(\romannumeral2)}{\leq} \sum_{t=1}^T \sum_{h=1}^H \EE_{(s_h,a_h)\sim (P^{(*,t)},\pi_t)}\left[\hat{b}_h^{(t)}(s_h,a_h) + (P_h^{(*,t)}-\hat{P}_h^{(t)})V_{h+1,\hat{P}^{(t)},r^t-\hat{b}^{(t)}}^{\pi_t}(s_h,a_h)\right]\nonumber\\
        & \qquad + \sum_{t=1}^T \left[V_{\hat{P}^{(t)},r^t-\hat{b}^{(t)}}^{\hat{\pi}_t} -V_{P^{(*,t)},r^t}^{\hat{\pi}_t}\right]\nonumber\\
        &\overset{(\romannumeral3)}{\leq} \sum_{t=1}^T \sum_{h=1}^H \EE_{(s_h,a_h)\sim (P^{(*,t)},\pi_t)}\left[\hat{b}_h^{(t)}(s_h,a_h)\right]+ H\sum_{t=1}^T \sum_{h=1}^H \EE_{(s_h,a_h)\sim (P^{(*,t)},\pi_t)}\left[ f_h^{(t)}(s_h,a_h)\right] \nonumber\\
        & \qquad \qquad +  H\sqrt{\omega T \zeta_n},
    \end{align}
    where $(i)$ follows from the observation that $\hat{\pi}_t$ is the argmax over all Markovian policies as well as all history-dependent policies for $\hat{P}^{(t)}$, $(ii)$ follows from the simulation lemma, \Cref{lemma:simulation_lemma}, $(iii)$ follows from the observation that $V_{h,\hat{P}^{(t)},r^t-\hat{b}^{(t)}}^{\pi_t}\leq H$ and \Cref{lemma_restatement:bounded_difference_of_summation}.

    Now, using \Cref{Lemma:one-step-back} (with setting $P = P^{(*,t)}$ and $\phi = \phi^*$) and noting that $|\hat{b}_h^{(t)}|_\infty\leq 1$, for $h\geq 2$, we have
    \begin{align}\label{Eq:bound_expected_penalty}
        &\EE_{(s_h,a_h)\sim (P^{(*,t)},\pi_t)}\left[\hat{b}_h^{(t)}(s_h,a_h)\right] \nonumber\\
        &\leq \EE_{(s_{h-1},a_{h-1})\sim (P^{(*,t)},\pi_t)}\left[\|\phi^*_{h-1}(s_{h-1},a_{h-1})\|_{(\Sigma_{h-1,\pi_t^b, \phi^*})^{-1}}\times\right.\nonumber\\
        &\quad\scriptstyle\left.\sqrt{n\omega\EE_{(s_{h},a_h) \sim (P^{(\star,t)},\pi_t^b)}\left[\hat{b}_h^{(t)}(s,a)\right]^2 + \lambda d}\right]
    \end{align}
    From ~\eqref{Eq:concentration_penalty_thm}, we have 
    \begin{align}\label{Eq:bound_using_trace}
        & n\EE_{(s_{h},a_h) \sim (P^{(\star,t)},\pi_t^b)}\left[\hat{b}_h^{(t)}(s,a)\right]^2 \nonumber\\
        &\leq n\EE_{(s_{h},a_h) \sim (P^{(\star,t)},\pi_t^b)}\left[\min \left\{\alpha^2 \|\hat{\phi}_h(s_h,a_h)\|_{(\Sigma_{h,\pi_t^b, \hat{\phi}})^{-1}}^2, 1\right\}\right]\nonumber\\
        &\leq n\EE_{(s_{h},a_h) \sim (P^{(\star,t)},\pi_t^b)}\left[\alpha^2 \|\hat{\phi}_h(s_h,a_h)\|_{(\Sigma_{h,\pi_t^b, \hat{\phi}})^{-1}}^2\right]\nonumber\\
        &\leq \alpha^2 \Tr \left[n\EE_{(s_{h},a_h) \sim (P^{(\star,t)},\pi_t^b)}[\hat{\phi}_h\hat{\phi}_h^\top] \{n\EE_{(s_h,a_h)\sim (P^{(*,t)},\pi_t^b)}[\hat{\phi}_h\hat{\phi}_h^\top] + \lambda I\}^{-1}\right]\nonumber\\
        &\leq \alpha^2 \Tr \left[\{n\EE_{(s_{h},a_h) \sim (P^{(\star,t)},\pi_t^b)}[\hat{\phi}_h\hat{\phi}_h^\top] +\lambda I\} \{n\EE_{(s_h,a_h)\sim (P^{(*,t)},\pi_t^b)}[\hat{\phi}_h\hat{\phi}_h^\top] + \lambda I\}^{-1}\right]\nonumber\\
        &= \alpha^2 \Tr[I_d]\nonumber\\
        &= \alpha^2 d
    \end{align}

    Next, we upper bound $\EE_{(s_{h-1},a_{h-1})\sim (P^{(*,t)},\pi_t)}\left[\|\phi^*_{h-1}(s_{h-1},a_{h-1})\|^2_{(\Sigma_{h-1,\pi_t^b, \phi^*})^{-1}}\right]$ as the following
    \begin{align}\label{Eq:dC/n}
        & \EE_{(s_{h-1},a_{h-1})\sim (P^{(*,t)},\pi_t)}\left[\|\phi^*_{h-1}(s_{h-1},a_{h-1})\|^2_{(\Sigma_{h-1,\pi_t^b, \phi^*})^{-1}}\right]\nonumber\\
        &\overset{(\romannumeral1)}{\leq} C^*_{t,h}(\pi_t,\pi_t^b) \EE_{(s_{h-1},a_{h-1})\sim (P^{(*,t)},\pi_t^b)}\left[\|\phi^*_{h-1}(s_{h-1},a_{h-1})\|^2_{(\Sigma_{h-1,\pi_t^b, \phi^*})^{-1}}\right]\nonumber\\
        &\overset{(\romannumeral2)}{\leq} \frac{d C^*_{t,h}(\pi_t,\pi_t^b)}{n}\nonumber\\
        &\leq \frac{dC^*}{n},
    \end{align}
    where $(i)$ follows from \Cref{Lemma:distribution_shift} and $(ii)$ follows from similar steps as in ~\eqref{Eq:bound_using_trace}.

    Combining ~\eqref{Eq:bound_expected_penalty},~\eqref{Eq:bound_using_trace} and ~\eqref{Eq:dC/n}, we get  

    \begin{align}
        &\sum_{t=1}^T \EE_{(s_h,a_h)\sim (P^{(*,t)},\pi_t)}\left[\hat{b}_h^{(t)}(s_h,a_h)\right] \nonumber\\
        &\leq \sum_{t=1}^T \EE_{(s_{h-1},a_{h-1})\sim (P^{(*,t)},\pi_t)}\left[\|\phi^*_{h-1}(s_{h-1},a_{h-1})\|_{(\Sigma_{h-1,\pi_t^b, \phi^*})^{-1}}\times\right.\nonumber\\
        &\quad\scriptstyle\left.\sqrt{n\omega\EE_{(s_{h},a_h) \sim (P^{(\star,t)},\pi_t^b)}\left[\hat{b}_h^{(t)}(s,a)\right]^2 + \lambda d}\right]\nonumber\\
        &\leq \sum_{t=1}^T \sqrt{\EE_{(s_{h-1},a_{h-1})\sim (P^{(*,t)},\pi_t)}\left[\|\phi^*_{h-1}(s_{h-1},a_{h-1})\|^2_{(\Sigma_{h-1,\pi_t^b, \phi^*})^{-1}}\right]}\times \nonumber\\
        & \quad \sqrt{\EE_{(s_{h-1},a_{h-1})\sim (P^{(*,t)},\pi_t)}\left[n\omega\EE_{(s_{h},a_h) \sim (P^{(\star,t)},\pi_t^b)}\left[\hat{b}_h^{(t)}(s,a)\right]^2 + \lambda d\right]}\nonumber\\
        &\leq \sum_{t=1}^T \sqrt{\frac{dC^*}{n}}\sqrt{\EE_{(s_{h-1},a_{h-1})\sim (P^{(*,t)},\pi_t)}\left[n\omega\EE_{(s_{h},a_h) \sim (P^{(\star,t)},\pi_t^b)}\left[\hat{b}_h^{(t)}(s,a)\right]^2 \right] + \lambda d}\nonumber\\
        &\leq \sqrt{\frac{dC^*}{n}} \left(\sum_{t=1}^T\sqrt{n\omega \EE_{(s_{h-1},a_{h-1})\sim (P^{(*,t)},\pi_t)}\left[\EE_{(s_{h},a_h) \sim (P^{(\star,t)},\pi_t^b)}\left[\hat{b}_h^{(t)}(s,a)\right]^2 \right]} + T\sqrt{\lambda d}\right)\nonumber\\
        &\leq \sqrt{\frac{dC^*}{n}} \left(\sum_{t=1}^T\sqrt{n\omega^2 \EE_{(s_{h},a_h) \sim (P^{(\star,t)},\pi_t^b)}\left[\hat{b}_h^{(t)}(s,a)\right]^2} + T\sqrt{\lambda d}\right)\nonumber\\
        &= \omega\sqrt{dC^*}\sum_{t=1}^T \sqrt{\EE_{(s_{h},a_h) \sim (P^{(\star,t)},\pi_t^b)}\left[\hat{b}_h^{(t)}(s,a)\right]^2} + Td\sqrt{\frac{\lambda C^*}{n}}\nonumber\\
        &\leq \omega\sqrt{dC^*} \sqrt{T\sum_{t=1}^T \EE_{(s_{h},a_h) \sim (P^{(\star,t)},\pi_t^b)}\left[\hat{b}_h^{(t)}(s,a)\right]^2} + Td\sqrt{\frac{\lambda C^*}{n}}\nonumber\\
        &\leq \omega\sqrt{dC^* T} \sqrt{\frac{T\alpha^2d}{n}}+ Td\sqrt{\frac{\lambda C^*}{n}}\nonumber\\
        &= \omega \alpha Td \sqrt{\frac{C^*}{n}} + Td\sqrt{\frac{\lambda C^*}{n}}.
    \end{align}

    Following similar steps as in ~\eqref{Eq:bound_using_trace}, we can further show that
    \begin{equation*}
        \EE_{(s_1,a_1) \sim (P^{(*,t)},\pi_t)} \left[\hat{b}_1^{(t)}(s_1,a_1)\right] \leq \sqrt{\EE_{(s_1,a_1) \sim (P^{(*,t)},\pi_t)} \left[\hat{b}_1^{(t)}(s_1,a_1)\right]^2} \leq \alpha\sqrt{\frac{d}{n}}.
    \end{equation*}

    Now noting $|f_h^{(t)}|_\infty \leq 1$, for $h\geq 2$,we get 
    \begin{align}\label{Eq:bound_tv_distance_true_mdp}
        & \sum_{t=1}^T \EE_{(s_h,a_h)\sim (P^{(*,t)},\pi_t)}\left[f_h^{(t)}(s_h,a_h)\right]\nonumber\\
        &\leq \sum_{t=1}^T \EE_{(s_{h-1},a_{h-1})\sim (P^{(*,t)},\pi_t)}\left[\|\phi^*_{h-1}(s_{h-1},a_{h-1})\|_{(\Sigma_{h-1,\pi_t^b, \phi^*})^{-1}}\times\right.\nonumber\\
        &\quad\scriptstyle\left.\sqrt{n\omega\EE_{(s_{h},a_h) \sim (P^{(\star,t)},\pi_t^b)}\left[f_h^{(t)}(s_h,a_h)^2\right] + \lambda d}\right]\nonumber\\
        &\leq \sum_{t=1}^T \sqrt{\EE_{(s_{h-1},a_{h-1})\sim (P^{(*,t)},\pi_t)}\left[\|\phi^*_{h-1}(s_{h-1},a_{h-1})\|^2_{(\Sigma_{h-1,\pi_t^b, \phi^*})^{-1}}\right]}\times\nonumber\\
        &\quad \sqrt{\EE_{(s_{h-1},a_{h-1})\sim (P^{(*,t)},\pi_t)} \left[n\omega\EE_{(s_{h},a_h) \sim (P^{(\star,t)},\pi_t^b)}\left[f_h^{(t)}(s_h,a_h)^2\right] + \lambda d\right]}\nonumber\\
        &\leq \sqrt{\frac{dC^*}{n}}\sum_{t=1}^T \sqrt{n\omega \EE_{(s_{h-1},a_{h-1})\sim (P^{(*,t)},\pi_t)} \left[\EE_{(s_{h},a_h) \sim (P^{(\star,t)},\pi_t^b)}\left[f_h^{(t)}(s_h,a_h)^2\right]\right] + \lambda d}\nonumber\\
        &\leq \sqrt{\frac{dC^*}{n}}\sum_{t=1}^T \left( \sqrt{n\omega \EE_{(s_{h-1},a_{h-1})\sim (P^{(*,t)},\pi_t)} \left[\EE_{(s_{h},a_h) \sim (P^{(\star,t)},\pi_t^b)}\left[f_h^{(t)}(s_h,a_h)^2\right]\right] } + \sqrt{\lambda d} \right) \notag \\
        &\leq \sqrt{\frac{dC^*}{n}}\sum_{t=1}^T \sqrt{n\omega^2\EE_{(s_{h},a_h) \sim (P^{(\star,t)},\pi_t^b)}\left[f_h^{(t)}(s_h,a_h)^2\right]} + T d\sqrt{\frac{\lambda  C^*}{n}}\nonumber\\
        &= \omega \sqrt{dC^*}\sum_{t=1}^T \sqrt{\EE_{(s_{h},a_h) \sim (P^{(\star,t)},\pi_t^b)}\left[f_h^{(t)}(s_h,a_h)^2\right]} + T d\sqrt{\frac{\lambda  C^*}{n}}\nonumber\\
        &\leq \omega \sqrt{dC^*} \sqrt{T \sum_{t=1}^T \EE_{(s_{h},a_h) \sim (P^{(\star,t)},\pi_t^b)}\left[f_h^{(t)}(s_h,a_h)^2\right]} + T d\sqrt{\frac{\lambda  C^*}{n}}\nonumber\\
        &\leq \omega \sqrt{dC^* T \zeta_n} + T d\sqrt{\frac{\lambda  C^*}{n}}
    \end{align}

    Further, note that, 
    \begin{align*}
       \sum_{t=1}^T \EE_{(s_1,a_1) \sim (P^{(*,t)},\pi_t)} \left[f_1^{(t)}(s_1,a_1)\right] &\leq \sum_{t=1}^T \sqrt{\EE_{(s_1,a_1) \sim (P^{(*,t)},\pi_t)} \left[f_1^{(t)}(s_1,a_1)^2\right]}\\
        &\leq \sum_{t=1}^T \sqrt{\omega\EE_{(s_1,a_1) \sim (P^{(*,t)},\pi_t^b)}\left[f_1^{(t)}(s_1,a_1)^2\right]}\\
        &=  \sqrt{\omega}\sum_{t=1}^T \sqrt{\zeta_1^{(t)}}\\
        &\leq \sqrt{\omega T \zeta_n},
    \end{align*}
    where the last inequality follows from the Cauchy-Schwarz inequality and~\eqref{Eq:MLE-thm-proof}.

 Finally, from \eqref{Eq:difference_v_pi_hat_pi} we get
\begin{align}
    & \sum_{t=1}^T \left[V_{P^{(*,t)},r^t}^{\pi_t} - V_{P^{(*,t)},r^t}^{\hat{\pi}_t}\right]\nonumber\\
        &\leq \sum_{t=1}^T \sum_{h=1}^H \EE_{(s_h,a_h)\sim (P^{(*,t)},\pi_t)}\left[\hat{b}_h^{(t)}(s_h,a_h)\right]+ H\sum_{t=1}^T \sum_{h=1}^H \EE_{(s_h,a_h)\sim (P^{(*,t)},\pi_t)}\left[ f_h^{(t)}(s_h,a_h)\right] \nonumber\\
        & \qquad \qquad +  H\sqrt{\omega T \zeta_n}\nonumber\\
        &\leq  \sum_{h=2}^H \sum_{t=1}^T \EE_{(s_h,a_h)\sim (P^{(*,t)},\pi_t)}\left[\hat{b}_h^{(t)}(s_h,a_h) \right] + H \sum_{h=2}^H \sum_{t=1}^T \EE_{(s_h,a_h)\sim (P^{(*,t)},\pi_t)}\left[f_h^{(t)}(s_h,a_h)\right] \nonumber\\
        &\qquad + \sum_{t=1}^T \EE_{(s_1,a_1)\sim (P^{(*,t)},\pi_t)}\left[\hat{b}_1^{(t)}(s_1,a_1) \right] + H\sum_{t=1}^T \EE_{(s_1,a_1)\sim (P^{(*,t)},\pi_t)}\left[ f_1^{(t)}(s_1,a_1)\right] + H\sqrt{\omega T \zeta_n}\nonumber\\
        &\leq H\omega \alpha Td \sqrt{\frac{C^*}{n}} + HTd\sqrt{\frac{\lambda C^*}{n}} + H^2\omega \sqrt{dC^* T \zeta_n} + H^2T d\sqrt{\frac{\lambda  C^*}{n}} + \alpha T\sqrt{\frac{d}{n}} + 2H\sqrt{\omega T \zeta_n} \nonumber\\
        &\leq H\omega \alpha Td \sqrt{\frac{C^*}{n}} + 2H^2Td\sqrt{\frac{\lambda C^*}{n}} + H^2\omega \sqrt{dC^* T \zeta_n}  + \alpha T\sqrt{\frac{d}{n}} + 2H\sqrt{\omega T \zeta_n}.
\end{align}
So, we have
\begin{align}
        \frac{1}{T}\sum_{t=1}^T \left[V_{P^{(*,t)},r^t}^{\pi_t} - V_{P^{(*,t)},r^t}^{\hat{\pi}_t}\right] &\leq H\omega \alpha d \sqrt{\frac{C^*}{n}} + 2H^2d\sqrt{\frac{\lambda C^*}{n}} + H^2\omega \sqrt{\frac{dC^*  \zeta_n}{T}}  + \alpha \sqrt{\frac{d}{n}} + 2H\sqrt{\frac{\omega  \zeta_n}{T}}.
    \end{align}

\end{proof}

\section{Proof of Supporting Lemmas in \Cref{Appendix:Proof of Multitask Offline Representation Learning}}\label{Section:proof of supporting lemmas of multitask offline rl}

In this section, we provide the proofs of the lemmas that we used in the proof of \Cref{theorem_restate:sample-complexity-upstream}.
\subsection{Proof of \cref{lemma:model-estimation-error}}
\begin{proof}[Proof of \cref{lemma:model-estimation-error}]
    For $h =1$,
    \begin{align*}
        \EE_{(s_1,a_1) \sim (\hat{P}^{(t)},\pi_t)} \left[f_1^{(t)}(s_1,a_1)\right] &\leq \sqrt{\EE_{(s_1,a_1) \sim (\hat{P}^{(t)},\pi_t)} \left[f_1^{(t)}(s_1,a_1)^2\right]}\\
        &\leq \sqrt{\omega_t\EE_{(s_1,a_1) \sim (P^{(*,t)},\pi_t^b)}\left[f_1^{(t)}(s_1,a_1)^2\right]}\\
        &= \sqrt{\omega_t\zeta_1^{(t)}}
    \end{align*}
    where the first inequality follows from Jensen's inequality and the second inequality follows from importance sampling.
    Denoting $\zeta_h^{(t)} = \EE_{(s_h,a_h) \sim (P^{(*,t)},\pi_t^b)} \left[f_h^{(t)}(s,a)^2\right]$, for $h \geq 2$, we have
    \begin{align*}
        &\EE_{s_h \sim (\hat{P}^{(t)},\pi_t) \atop a_h \sim \pi_t}\left[f_h^{(t)}(s_h,a_h)\right]\\
        &\overset{(\romannumeral1)}{\leq}\EE_{(s_{h-1},a_{h-1})\sim (\hat{P}^{(t)}, \pi_t)}\left[\left\|\hat{\phi}_{h-1}(s_{h-1},a_{h-1})\right\|_{(\Sigma_{h-1,\pi_t^b, \hat{\phi}})^{-1}} \times\right.\\
 		&\quad\scriptstyle\left.
 		\sqrt{n \omega_t \EE_{(s_{h},a_h) \sim (P^{(\star,t)},\pi_t^b)}[f_h^{(t)}(s_h,a_h)^2]+\lambda d + n\EE_{(s_{h-1},a_{h-1})\sim(P^{(\star,t)},\pi_t^b)}\left[f_{h-1}^{(t)}(s_{h-1},a_{h-1})^2\right]}\right]\\
        &\overset{(\romannumeral2)}{=}\EE_{(s_{h-1},a_{h-1})\sim (\hat{P}^{(t)}, \pi_t)}\left[\sqrt{n\omega_t \zeta_h^{(t)} + \lambda d + n\zeta_{h-1}^{(t)}}\left\|\hat{\phi}_{h-1}(s_{h-1},a_{h-1})\right\|_{(\Sigma_{h-1,\pi_t^b, \hat{\phi}})^{-1}}\right]\\
        &= \EE_{(s_{h-1},a_{h-1})\sim (\hat{P}^{(t)}, \pi_t)}\left[\alpha_h^{(t)}\left\|\hat{\phi}_{h-1}(s_{h-1},a_{h-1})\right\|_{(\Sigma_{h-1,\pi_t^b, \hat{\phi}})^{-1}}\right]\\
    \end{align*}
    where $(i)$ follows from \Cref{Lemma:one-step-back} and $|f_h^{(t)}(s_h,a_h)| \leq 1$, $(ii)$ uses notations defined in \Cref{section:frequently_used_notations}.
\end{proof}

\subsection{Proof of \cref{lemma_restatement:bounded_difference_of_summation}}
\begin{proof}[Proof of \cref{lemma_restatement:bounded_difference_of_summation}]
    We condition on the events:
    \begin{equation}
         \sum_{t=1}^T \EE_{(s_h,a_h) \sim (P^{(*,t)},\pi_t^b)} \left[f_h^{(t)}(s,a)^2\right] \leq \zeta_n , \quad\mbox{ where } \zeta_n:=\frac{2\log(2|\Phi||\Psi|^TnH/\delta)}{n},
    \end{equation}
    and 
    \begin{equation}\label{Eq:concentration_penalty}
        \forall \phi \in \Phi: \left\|\phi_h(s,a)\right\|_{(\hat{\Sigma}_{h,\phi}^{(t)})^{-1}} = \Theta \left(\left\|\phi_h(s,a)\right\|_{(\Sigma_{h,\pi_t^b, \phi})^{-1}}\right).
    \end{equation}
    From \Cref{Lemma:multitask_mle} and \Cref{Lemma:concentration_penalty_term}, this event happens with probability $1-\delta$. Conditioning on this event, we have
    \begin{equation}
        \alpha_h^{(t)} = \sqrt{n\omega_t \zeta_h^{(t)} + \lambda d + n\zeta_{h-1}^{(t)}} \leq \sqrt{2n\omega\zeta_n + \lambda d} = \alpha
    \end{equation}
    We have
    \begin{align*}
        &\sum_{t=1}^T\left[V_{\hat{P}^{(t)},r^t-\hat{b}^{(t)}}^{\pi_t}-V_{P^{(*,t)},r^t}^{\pi_t} \right] \\
        &\overset{(\romannumeral1)}{=} \sum_{t=1}^T\sum_{h=1}^H \EE_{(s_h,a_h)\sim (\hat{P}^{(t)},\pi_t)}\left[-\hat{b}_h^{(t)}(s_h,a_h) + (P_h^{(*,t)}-\hat{P}_h^{(t)})V_{h+1,P^{(*,t)},r^t}^{\pi_t}(s_h,a_h)\right]\\
        &\overset{(\romannumeral2)}{\leq} H\sum_{t=1}^T\sum_{h=1}^H \EE_{(s_h,a_h)\sim (\hat{P}^{(t)},\pi_t)} \left[-\hat{b}_h^{(t)}(s_h,a_h) +f_h^{(t)}(s_h,a_h)\right]\\
        &\overset{(\romannumeral3)}{\leq} H\sum_{t=1}^T\sum_{h=2}^H \EE_{(s_{h-1},a_{h-1})\sim (\hat{P}^{(t)}, \pi_t)}\left[\min\left\{\alpha_h^{(t)}\left\|\hat{\phi}_{h-1}(s_{h-1},a_{h-1})\right\|_{(\Sigma_{h-1,\pi_t^b, \hat{\phi}})^{-1}},1\right\}\right] \\
        & \qquad + H\sum_{t=1}^T\sqrt{\omega\zeta_1^{(t)}} + H\sum_{t=1}^T\sum_{h=1}^H \EE_{(s_h,a_h)\sim (\hat{P}^{(t)},\pi_t)} \left[-\hat{b}_h^{(t)}(s_h,a_h)\right]\\
        &\overset{(\romannumeral4)}{\leq} H\sum_{t=1}^T\sum_{h=2}^H \EE_{(s_{h-1},a_{h-1})\sim (\hat{P}^{(t)}, \pi_t)}\left[\min\left\{\alpha_h^{(t)}\left\|\hat{\phi}_{h-1}(s_{h-1},a_{h-1})\right\|_{(\Sigma_{h-1,\pi_t^b, \hat{\phi}})^{-1}},1\right\}\right] \\
        & \qquad + H\sqrt{\omega T\zeta_n}+ H\sum_{t=1}^T\sum_{h=1}^H \EE_{(s_h,a_h)\sim (\hat{P}^{(t)},\pi_t)} \left[-\hat{b}_h^{(t)}(s_h,a_h)\right]\\
        &\overset{(\romannumeral5)}{\lesssim} H\sum_{t=1}^T\sum_{h=2}^H \EE_{(s_{h-1},a_{h-1})\sim (\hat{P}^{(t)}, \pi_t)}\left[ \min\left\{\alpha\left\|\hat{\phi}_{h-1}(s_{h-1},a_{h-1})\right\|_{(\Sigma_{h-1,\pi_t^b, \hat{\phi}})^{-1}},1\right\}\right] \\
        & \qquad + H\sqrt{\omega T\zeta_n}+ H\sum_{t=1}^T\sum_{h=1}^H \EE_{(s_h,a_h)\sim (\hat{P}^{(t)},\pi_t)} \left[-\min\left\{\alpha\left\|\hat{\phi}_{h}(s_{h},a_{h})\right\|_{(\Sigma_{h,\pi_t^b, \hat{\phi}})^{-1}},1\right\}\right]\\
        &\overset{(\romannumeral6)}{\leq} H\sqrt{\omega T\zeta_n}+ H\sum_{t=1}^T \EE_{(s_1,a_1)\sim (\hat{P}^{(t)},\pi_t)} \left[-\min\left\{\alpha\left\|\hat{\phi}_{1}(s_{1},a_{1})\right\|_{(\Sigma_{1,\pi_t^b, \hat{\phi}})^{-1}},1\right\}\right]\\
        &\leq H\sqrt{\omega T\zeta_n}.
    \end{align*}
     where $(i)$ follows from \Cref{lemma:simulation_lemma}, $(ii)$ follows from the observation $V_{P^{(*,t)},r^t}^{\pi_t}\leq H$, $(iii)$ follows from \Cref{lemma:model-estimation-error}, $(iv)$ follows from Cauchy-Schwarz inequality and the fact that $\sum_{t=1}^T \zeta_h^{(t)} \leq \zeta_n$, $(v)$ follows from ~\eqref{Eq:concentration_penalty}.
\end{proof}

\section{Multitask Offline MLE}
Consider a sequential conditional probability estimation setting with an instance space $\cX$ and target space $\cY$ and with a conditional probability density $p(y \given x) = f^*(x,y)$. We consider a function class $\cF: (\cX \times \cY) \rightarrow \RR$ for modeling the condition distribution $f^*$, and we further assume that the realizability condition holds i.e. $f^* \in \cF$. We are given a dataset $D :=\{(x_i,y_i)\}_{i=1}^n$, where $x_i \sim \cD_i = \cD_i(x_{1:i-1}, y_{1:i-1})$ and $y_i \sim p(\cdot \given x_i)$. Let $D'$ denote a tangent sequence $\{(x_i', y_i')\}_{i=1}^n$ where $x_i'\sim \cD_i(x_{1:i-1}, y_{1:i-1})$ and $y'_i \sim p(\cdot \given x'_i)$. Note that here $x'_i$ depends on the original sequence, and so the tangent sequence is independent conditional on $D$.

\begin{lemma}[Lemma 24 of \citet{agarwal2020flambe}]\label{lemma:lemma24_agrawal}
    Let $D$ be a dataset of $n$ examples, and let $D'$ be a tangent sequence. Let $L(f,D) = \sum_{i=1}^n \ell(f, (x_i,y_i))$ be any function that decomposes additively across examples where $\ell$ is any function, and let $\hat{f}(D)$ be any estimator taking as input random variable $D$ and with range $\cF$. Then
    \begin{equation*}
        \EE_D \left[\exp\left(L(\hat{f}(D),D)-\log \EE_{D'}\left[\exp(L(\hat{f}(D),D'))\right] - \log |\cF|\right)\right] \leq 1.
    \end{equation*}
\end{lemma}

\begin{lemma}[Lemma 25 of \citet{agarwal2020flambe}]\label{lemma:lemma25_agrawal}
    For any two conditional probability densities $f_1, f_2$ and any distribution $\cD \in \Delta(\cX)$ we have
    \begin{equation*}
        \EE_{x\sim D}\|f_1(x,\cdot) - f_2(x, \cdot)\|_{TV}^2 \leq -2 \log \EE_{x\sim D, y \sim f_2(\cdot \given x)}\exp\left(-\frac{1}{2}\log(f_2(x,y)/f_1(x,y))\right).
    \end{equation*}
\end{lemma}

\begin{lemma}[Multitask offline MLE guarantee]\label{Lemma:multitask_mle}
Given $\delta \in (0,1)$, consider the transition kernels learned in \Cref{Algorithm:multi-task-offline}. For any $n, h$ with probability at least $1-\delta/2$, we have
\begin{equation}
         \sum_{t=1}^T \EE_{(s_h,a_h) \sim (P^{(*,t)},\pi_t^b)} \left[f_h^{(t)}(s,a)^2\right] \leq \zeta_n , \quad\mbox{ where } \zeta_n:=\frac{2\log(2|\Phi||\Psi|^TnH/\delta)}{n}.
    \end{equation}
\end{lemma}

\begin{proof}[Proof of \cref{Lemma:multitask_mle}]
    Let $\hat{f}(D)$ denote empirical maximum likelihood estimator:
    \begin{equation*}
        \hat{f}(D) := \argmax_{f\in \cF} \sum_{(x_i,y_i) \in D} \log f(x_i,y_i)
    \end{equation*}
    We combine \Cref{lemma:lemma24_agrawal} with the Chernoff method to obtain the following exponential tail bound: with probability $1-\delta$, we have
    \begin{equation}\label{Eq:chernoff}
        -\log \EE_{D'}\left[\exp(L(\hat{f}(D),D'))\right] \leq -L(\hat{f}(D),D) + \log |\cF| + \log(1/\delta).
    \end{equation}
    Now, we set $L(f,D) = \sum_{i=1}^n -\frac{1}{2}\log(f^*(x_i,y_i)/f(x_i,y_i))$ where $D = \{x_i,y_i\}_{i=1}^n$ is a dataset and $D' = \{x'_i,y'_i\}_{i=1}^n$ is a tangent sequence. In the multitask offline RL setting, let $x = \{(s_h^t, a_h^t)\}_{t=1}^T$, $y = \{s_{h+1}^t\}_{t=1}^T$ and $f(x,y) = \prod_{t=1}^T P_h^t(s_{h+1}^t \given s_h^t, a_h^t)$. Then, the dataset $D_h$ can be decomposed into $D_h = \cup_{t=1}^T D_h^{(t)}$ where $D_h^{(t)} = \{(s_h^{(i,t)}, a_h^{(i,t)}, s_{h+1}^{(i,t)}\}_{i\in[n]}$. Similarly, $D'_h = \cup_{t=1}^T (D'_h)^{(t)}$, and $(\cD_h^t)_i := (\cD_h^t)_i((s_h^t,a_h^t)_{1:i-1},(s_{h+1}^t)_{1:i-1})$. Thus, in the multitask offline RL setting, we have the cardinality $|\cF|  = |\Phi||\Psi|^T$. With this choice, the right hand side of ~\eqref{Eq:chernoff} is
    \begin{equation}\label{Eq:RHS_chernoff}
        \sum_{i=1}^n \frac{1}{2}\log(f^*(x_i,y_i)/\hat{f}(x_i,y_i)) + \log|\cF| + \log(1/\delta) \leq \log|\cF| + \log(1/\delta) = \log(|\Phi||\Psi|^T/\delta),
    \end{equation}
    where the first inequality follows because $\hat{f}$ is the empirical maximum likelihood estimator and the realizability assumption. The equality follows because $|\cF| = |\Phi||\Psi|^T$. On the other hand, the left hand side of ~\eqref{Eq:chernoff} is
    \begin{align}\label{Eq:LHS_chernoff}
        & -\log \EE_{D'_h} \left[\exp\left(\sum_{i=1}^n -\frac{1}{2}\log\left(\frac{f^*(x'_i,y'_i)}{\hat{f}(x'_i,y'_i)}\right)\right) \Bigggiven D_h \right]\nonumber\\
        &\overset{(\romannumeral1)}{=} -\log \EE_{D'_h} \left[\exp\left(\sum_{i=1}^n -\frac{1}{2}\log\left(\prod_{t=1}^T\frac{P_h^{(*,t)}(s_{h+1}^{(i,t)} \given s_h^{(i,t)}, a_h^{(i,t)})}{\hat{P}_h^{(t)}(s_{h+1}^{(i,t)} \given s_h^{(i,t)}, a_h^{(i,t)})}\right)\right) \Bigggiven D_h \right]\nonumber\\
        &\overset{(\romannumeral2)}{=} -\sum_{t=1}^T \log \EE_{(D'_h)^{(t)}} \left[\exp\left(\sum_{i=1}^n -\frac{1}{2}\log\left(\frac{P_h^{(*,t)}(s_{h+1}^{(i,t)} \given s_h^{(i,t)}, a_h^{(i,t)})}{\hat{P}_h^{(t)}(s_{h+1}^{(i,t)} \given s_h^{(i,t)}, a_h^{(i,t)})}\right)\right) \Bigggiven D_h \right]\nonumber\\
        &\overset{(\romannumeral3)}{=} -\sum_{t=1}^T \sum_{i=1}^n \log \EE_{(D_h^{(t)})_i} \left[\exp\left(-\frac{1}{2}\log\left(\frac{P_h^{(*,t)}(s_{h+1}^{(i,t)} \given s_h^{(i,t)}, a_h^{(i,t)})}{\hat{P}_h^{(t)}(s_{h+1}^{(i,t)} \given s_h^{(i,t)}, a_h^{(i,t)})}\right)\right)\right]\nonumber\\
        &\overset{(\romannumeral4)}{\geq} \sum_{t=1}^T \frac{1}{2} \sum_{i=1}^n \EE_{(s_h,a_h)\sim (\cD_h^t)_i} \left\|\hat{P}^{(t)}(\cdot \given s_h, a_h) - P^{(*,t)}(\cdot \given s_h,a_h)\right\|_{TV}^2\nonumber\\
        &\overset{(\romannumeral5)}{=} \frac{n}{2} \sum_{t=1}^T \EE_{(s_h,a_h) \sim (P^{(*,t)},\pi_t^b)} \left[f_h^{(t)}(s,a)^2\right],
    \end{align}
    where $(i)$ follows from the above definition of $f(x,y)$, $(ii)$ follows because the data of $T$ tasks are independent conditional on $D_h$, $(iii)$ follows because $\hat{P}^{(t)}$ is independent of the dataset $(D'_h)^{(t)}$ and from the definition of $D'_h$, $(iv)$ follows from \Cref{lemma:lemma25_agrawal}, and $(v)$ follows because in task $T$, the data is collected using behavior policy $\pi_t^b$.

    Combining ~\eqref{Eq:chernoff}, ~\eqref{Eq:RHS_chernoff},~\eqref{Eq:LHS_chernoff}, we get
    \begin{equation}
        \frac{n}{2} \sum_{t=1}^T \EE_{(s_h,a_h) \sim (P^{(*,t)},\pi_t^b)} \left[f_h^{(t)}(s,a)^2\right] \leq \log(|\Phi||\Psi|^T/\delta)
    \end{equation}
    Using union bound, we obtain that for any $h\in[H]$ and $n$ with probability at least $1-\delta/2$, it holds that
    \begin{equation*}
         \sum_{t=1}^T \EE_{(s_h,a_h) \sim (P^{(*,t)},\pi_t^b)} \left[f_h^{(t)}(s,a)^2\right] \leq \frac{2\log(2|\Phi||\Psi|^TnH/\delta)}{n}.
    \end{equation*}
    This completes the proof.
\end{proof}

\section{One-Step Back Lemma and Concentration of Penalty Term}

\subsection{One-step back lemma}
The following one-step back lemma is a key technical lemma for our proof. One-step back lemma for offline setting was first introduced in \citet{uehara2021representation} for infinite-horizon stationary MDP. Our lemma extends their result to finite-horizon non-stationary MDP for offline setting. For any function $g \in \cS\times\cA\rightarrow \RR$, policy $\pi$ and transition kernel $P$, the lemma shows that we can relate the expected value $\EE_{(s_h,a_h)\sim (P,\pi)}[g(s_h,a_h)]$ to the potential function $\EE_{(s_{h-1},a_{h-1})\sim (P, \pi)}\left\|\phi_{h-1}(s_{h-1},a_{h-1})\right\|_{(\Sigma_{h-1,\pi_t^b, \phi})^{-1}}$.
\begin{lemma}[One-step back inequality for non-stationary finite-horizon MDP in offline setting]\label{Lemma:one-step-back}
For  each task $t \in [T]$, let $P \in \{\hat{P}^{(t)}, P^{(*,t)}\}$ with embedding $\phi \in \{\hat{\phi}, \phi^*\}$ and $\mu$ be an MDP model, and $\Sigma_{h,\pi_t^b, \phi} = n\EE_{(s_h,a_h)\sim (P^{(*,t)},\pi_t^b)}[\phi\phi^\top] + \lambda I$ be the covariance matrix following the behavior policy $\pi_t^b$ under the true environment $P^{(*,t)}$. Denote the total variation distance between $P^{(*,t)}$ and $P$ at time step $h$ by $f^t(s_h,a_h)$. Take any $g \in \cS \times \cA \rightarrow \RR$ such that $\|g\|_\infty \leq B$. Then, letting $\omega = \max_{s,a}(1/\pi_t^b(a\given s))$ for all $h \geq 2$, and for any policy $\pi$, we have
\begin{align*}
 		\EE_{(s_h,a_h) \sim (P, \pi)}[g(s_h,a_h)]  &\leq \EE_{(s_{h-1},a_{h-1})\sim (P, \pi)}\left[\left\|\phi_{h-1}(s_{h-1},a_{h-1})\right\|_{(\Sigma_{h-1,\pi_t^b, \phi})^{-1}} \times\right.\nonumber\\
 		&\quad\scriptstyle\left.
 		\sqrt{n \omega \EE_{(s_{h},a_h) \sim (P^{(\star,t)},\pi_t^b)}[g^2(s_h,a_h)]+\lambda dB^2 + nB^2\EE_{(s_{h-1},a_{h-1})\sim(P^{(\star,t)},\pi_t^b)}\left[f^t(s_{h-1},a_{h-1})^2\right]}\right].
 \end{align*}
\end{lemma}
\begin{proof}[Proof of \cref{Lemma:one-step-back}]
First, we have
\begin{align*}
    & \EE_{(s_h,a_h) \sim (P, \pi)}[g(s_h,a_h)] \\
    & = \EE_{(s_{h-1},a_{h-1}) \sim (P, \pi)} \left[\int_{s_h} \sum_{a_h} g(s_h,a_h) \pi(a_h \given s_h) \la \phi_{h-1}(s_{h-1}, a_{h-1}), \mu_{h-1}(s_h)\ra ds_h\right]\\
    &\leq \EE_{(s_{h-1},a_{h-1}) \sim (P, \pi)} \left[\|\phi_{h-1}(s_{h-1}, a_{h-1})\|_{(\Sigma_{h-1,\pi_t^b, \phi})^{-1}} \left\|\int \sum_{a_h} g(s_h,a_h)\pi(a_h\given s_h)\mu_{h-1}(s_h)ds_h\right\|_{\Sigma_{h-1,\pi_t^b, \phi}}\right],
\end{align*}
where the inequality follows from Cauchy-Schwarz inequality. 

Then,
\begin{align*}
    &\left\|\int \sum_{a_h} g(s_h,a_h)\pi(a_h\given s_h)\mu_{h-1}(s_h)ds_h\right\|_{\Sigma_{h-1,\pi_t^b, \phi}}^2\\
    &= \left\{\int \sum_{a_h} g(s_h,a_h)\pi(a_h\given s_h)\mu_{h-1}(s_h)ds_h\right\}^\top \mkern-16mu \left\{n\EE_{s_{h-1}\sim P^{(*,t)} \atop a_{h-1} \sim \pi_b}[\phi\phi^\top] + \lambda I\right\}\mkern-6mu\left\{\int \sum_{a_h} g(s_h,a_h)\pi(a_h\given s_h)\mu_{h-1}(s_h)ds_h\right\}\\
    &\overset{(\romannumeral1)}{\leq} n\EE_{(s_{h-1},a_{h-1})\sim (P^{(*,t)},\pi_b)}\left[\left(\int \sum_{a_h} g(s_h,a_h) \pi(a_h\given s_h)\mu_{h-1}(s_h)^\top \phi(s_{h-1},a_{h-1})ds_h\right)^2\right] + B^2\lambda d\\
    &= n\EE_{(s_{h-1},a_{h-1})\sim (P^{(*,t)},\pi_b)}\left[\EE_{s_h \sim P(\cdot \given s_{h-1},a_{h-1}) \atop a_h \sim \pi}[g(s_h,a_h)^2]\right] + B^2\lambda d\\
    &\overset{(\romannumeral2)}{\leq} n\EE_{(s_{h-1},a_{h-1})\sim (P^{(*,t)},\pi_b)}\left[\EE_{s_h \sim P^{(*,t)} \atop a_h \sim \pi}[g(s_h,a_h)^2] \right] + B^2\lambda d + nB^2\EE_{(s_{h-1},a_{h-1})\sim (P^{(*,t)},\pi_b)}[f^t(s_{h-1},a_{h-1})^2]\\
    &\overset{(\romannumeral3)}{\leq} n\omega \EE_{(s_h,a_h)\sim (P^{(*,t)},\pi_b)}\left[g(s_h,a_h)^2\right]+ B^2\lambda d + nB^2\EE_{(s_{h-1},a_{h-1})\sim (P^{(*,t)},\pi_b)}[f^t(s_{h-1},a_{h-1})^2]\\ 
\end{align*}
where $(i)$ follows from the assumption $\|g\|_\infty \leq B$ and for any function $h:\cS\rightarrow [0,1]$, $\|\int \mu_h(s)h(s)ds\|_2 \leq \sqrt{d}$, $(ii)$ follows from the definition of $f^t(s_h,a_h)$ which is the total variation distance between $P^*$ and $P$ at time step $h$, and finally $(iii)$ follows from importance sampling. This completes the proof.
\end{proof}

\subsection{Concentration of penalty term}
Recall that $\Sigma_{h,\pi_t^b, \phi} = n\EE_{(s_h,a_h)\sim (P^{(*,t)},\pi_t^b)}[\phi\phi^\top] + \lambda I$. Thus, $\hat{\Sigma}_h^{(t)}$ is equal to $\Sigma_{h,\pi_t^b, \hat{\phi}}$ in expectation. We now provide an important lemma to ensure the concentration of the penalty term. The version for fixed $\phi$ is proved in \citet{zanette2021cautiously}. Here, we take a union bound over the whole feature $\phi \in \Phi$, number of total tasks $T$, horizon $H$ and cardinality $n$ of each offline dataset from individual tasks. 

\begin{lemma}[Concentration of the penalty term]\label{Lemma:concentration_penalty_term}
    Fix $\delta \in (0,1)$ and set $\lambda = O(d\log(2nTH|\Phi|/\delta))$ for any $n$. With probability at least $1-\delta/2$, we have that for any $n \in \NN$, $h \in [H]$, $t\in[T]$ and $\phi\in \Phi$, 
    \begin{equation*}
        \beta_1 \left\|\phi_h(s,a)\right\|_{(\Sigma_{h,\pi_t^b, \phi})^{-1}} \leq \left\|\phi_h(s,a)\right\|_{(\hat{\Sigma}_h^{(t)})^{-1}} \leq \beta_2 \left\|\phi_h(s,a)\right\|_{(\Sigma_{h,\pi_t^b, \phi})^{-1}},
    \end{equation*}
    where $\beta_1$ and $\beta_2$ are some absolute constants.
\end{lemma}

\section{Proof of \Cref{Lemma:approximate_feature_new_task}: Approximate Feature for New Task}

We first restate \Cref{Lemma:approximate_feature_new_task}.

\begin{lemma}\label{Lemma_restate:approximate_feature_new_task}
    Under \Cref{Assumption:upstream-to-downstream}, the output $\hat{\phi}$ of \cref{Algorithm:multi-task-offline} is a $\xi_{\text{down}}$-approximate feature for MDP $\cM^{T+1}$ where $\xi_{\text{down}} = \xi + \frac{C_L C_R\nu}{\kappa} \sqrt{\frac{2T\log(2|\Phi||\Psi|^TnH/\delta)}{n}}$, i.e. there exist a time-dependent unknown (signed) measure $\hat{\mu}^*$ over $\cS$ such that for any $(s,a) \in \cS\times \cA$, we have 
    \begin{equation*}
        \|P_h^{(*,T+1)}(\cdot |s,a) - \la \hat{\phi}_h(s,a), \hat{\mu}_h^*(\cdot)\ra\|_{\text{TV}} \leq \xi_{\text{down}}.
    \end{equation*}
    Furthermore, for any $g:\cS\rightarrow [0,1]$, we have $\|\int \hat{\mu}_h^*(s)g(s)ds\|_2 \leq C_L \sqrt{d}$.
\end{lemma}

The following proof is motivated from the proof of Lemma 1 in \citet{cheng2022provable}.
\begin{proof}[Proof of \cref{Lemma_restate:approximate_feature_new_task}]
    For all $(s,a)\in \cS\times\cA$, $h\in[H]$ and for any $t\in[T]$ we have
    \begin{align}\label{Equation:bound_tv_distance_sum}
        &\sum_{t=1}^T \|\hat{P}_h^{(t)}(\cdot|s,a) - P_h^{(*,t)}(\cdot|s,a)\|_{\text{TV}}\notag\\
        &\leq \sum_{t=1}^T \max_{s\in \cS, a\in \cA}\|\hat{P}_h^{(t)}(\cdot|s,a) - P_h^{(*,t)}(\cdot|s,a)\|_{\text{TV}}\notag\\
        &\overset{(\romannumeral1)}{\leq} \sum_{t=1}^T C_R \EE_{(s_h,a_h)\sim \cU(\cS,\cA)}\|\hat{P}_h^{(t)}(\cdot|s_h,a_h) - P_h^{(*,t)}(\cdot|s_h,a_h)\|_{\text{TV}}\notag\\
        &\overset{(\romannumeral2)}{\leq} \frac{C_R\nu}{\kappa}\sum_{t=1}^T  \EE_{(s_h,a_h)\sim (P^{(*,t)},\pi_t^b)}\|\hat{P}_h^{(t)}(\cdot|s_h,a_h) - P_h^{(*,t)}(\cdot|s_h,a_h)\|_{\text{TV}}\notag\\
        &\overset{(\romannumeral3)}{\leq} \frac{C_R\nu}{\kappa} \sqrt{\frac{2T\log(2|\Phi||\Psi|^TnH/\delta)}{n}}
    \end{align}
    where $(i), (ii)$ follows from \Cref{Assumption:upstream-to-downstream} and  $(iii)$ follows from \cref{theorem:sample-complexity-upstream}.

    Defining $\hat{\mu}^*(\cdot) = \sum_{t=1}^T c_t\hat{\mu}^{(t)}(\cdot)$, we have
    \begin{align*}
        &\|P_h^{(*,T+1)}(\cdot|s,a)-\la \hat{\phi}_h(s,a),\hat{\mu}_h^*(\cdot)\ra\|_{\text{TV}}\\
        &= \|P_h^{(*,T+1)}(\cdot|s,a)-\la \hat{\phi}_h(s,a),\sum_{t=1}^T c_t\hat{\mu}^{(t)}(\cdot)\ra\|_{\text{TV}}\\
        &= \|P_h^{(*,T+1)}(\cdot|s,a)-\sum_{t-1}^Tc_t\hat{P}_h^{(t)}(\cdot|s,a)\|_{\text{TV}}\\
        &\leq \|P_h^{(*,T+1)}(\cdot|s,a)-\sum_{t-1}^Tc_t P_h^{(*,t)}(\cdot|s,a)\|_{\text{TV}} + \sum_{t=1}^T c_t \|P_h^{(*,t)}(\cdot|s,a) - \hat{P}_h^{(t)}(\cdot|s,a)\|_{\text{TV}}\\
        &\overset{(\romannumeral1)}{\leq} \xi + \frac{C_L C_R\nu}{\kappa} \sqrt{\frac{2T\log(2|\Phi||\Psi|^TnH/\delta)}{n}},
    \end{align*}
    where $(i)$ follows from \cref{Assumption:upstream-to-downstream}, ~\eqref{Equation:bound_tv_distance_sum} and the fact that $c_t \in [0,C_L]$ for all $t\in [T]$. Moreover, by normalization, for any $g:\cS\rightarrow[0,1]$, we get
    \begin{align*}
        \left\|\int \hat{\mu}_h^*(s)g(s)ds\right\|_2 &\leq \sum_{t=1}^T\left\|\int \hat{\mu}_h^{(t)}(s)g(s)ds\right\|_2\\
        &\leq C_L  \sqrt{d},
    \end{align*}
    where the last inequality follows from \cref{Ass:realizability}.
\end{proof}

\section{Proof for Downstream Reward-Free RL}\label{Appendix:Proof for Downstream Reward-Free RL}

For any $h \in [H]$, we define 
\begin{align*}
    &P_h^{(*,T+1)}(\cdot|s,a) = \la \phi_h^*(s,a), \mu_h^{(*,T+1)}(\cdot)\ra,\\
    &\overline{P}_h(\cdot|s,a) = \la \hat{\phi}_h(s,a), \hat{\mu}_h^*(\cdot)\ra.
\end{align*}

Given a reward function $r$ (as is provided in the planning phase of reward-free RL setting), for any function $f:\cS\rightarrow \RR$ and $h \in [H]$, we define the transition operators as follows

\begin{align*}
    &(P_h^{(*,T+1)}f)(s,a,r) = \int_{s'} \la \phi_h^*(s,a), \mu_h^{(*,T+1)}(s')\ra f(s')ds',\\
    & (\overline{P}_h f)(s,a,r) = \int_{s'} \la \hat{\phi}_h(s,a), \hat{\mu}_h^*(s')\ra f(s')ds'.\\
\end{align*}
 
When no reward function is provided as is the case in the exploration phase of reward-free RL setting, we simply omit $r$ from the above operator notation.

In this section for notational simplicity we denote $V^\pi_{h,P^{(*,T+1)},r}(s)$ and $Q^\pi_{h,P^{(*,T+1)},r}(s,a)$ by $V_h^\pi(s,r)$ and $Q_h^\pi(s,a,r)$ respectively where $r$ is reward function provided in the planning phase of downstream reward-free RL task. We similarly denote the optimal value function and action-value function under reward function $r$ as $V_h^*(s,r)$ and $Q_h^*(s,a,r)$ respectively. 

We also introduce the truncated optimal value function $\tilde{V}_h^*(s,r)$ in the planning phase, which is recursively defined from step $H+1$ to step 1. Compared to the definition of standard optimal value function $V_h^*(s,r)$, the main difference is that we take minimization over the value function and $H$ in each step in this definition. We provide the formal definition as follows.
\begin{definition}[Truncated Optimal Value Function]
We introduce the truncated optimal value function $\tilde{V}_h^*(s,r)$ which is recursively defined from step $H+1$ to step 1: 
\begin{align*}
    \tilde{V}_{H+1}^*(s,r) &= 0, \,\ \forall s \in \cS \\
    \tilde{Q}_h^*(s,a,r) &= r_h(s,a) + P_h^{(*,T+1)}\tilde{V}_{h+1}^*(s,a,r), \,\ \forall (s,a) \in \cS \times \cA \\
    \tilde{V}_h^*(s,r) &= \min \Big\{ \max_{a \in \cA} \Big\{ r_h(s,a) + P_h^{(*,T+1)}\tilde{V}_{h+1}^*(s,a,r)\Big\}, H \Big\}, \,\ \forall s \in \cS, h \in [H].
\end{align*}
We can similarly define  $\tilde{V}_h^\pi(s,r)$ and $\tilde{Q}_h^\pi(s,a,r)$.
\end{definition}

\subsection{Supporting Lemmas}
Now we state the supporting lemmas that are used in the proof of \Cref{Theorem:reward-free-exploration-main}. The proof of these lemmas are provided in \Cref{Section:proof of supporting lemmas for downstream RFE}.

The following lemma shows that the linear weight $\hat{w}_h^k$ in \Cref{Algorithm:downstream-reward-free-exploration-phase} is bounded. 

\begin{lemma}[Bounds on Weights in \Cref{Algorithm:downstream-reward-free-exploration-phase}]\label{Lemma:weight_bound_downstream_reward-free-exploration-phase}
    For any $h \in [H]$, the weight $\hat{w}_h^k$ in \Cref{Algorithm:downstream-reward-free-exploration-phase} satisfies 
    \begin{align*}
        \big\|\hat{w}_h^k\big\|_2 \leq H\sqrt{d K}.
    \end{align*}
\end{lemma}

\begin{lemma}\label{Lemma:self-normalized-bound-reward-free-exploration}
    Let $\mathcal{E}$ be the event that for all $(k, h) \in [K_{\text{RFE}}] \times [H]$, 
    \begin{equation*}
        \left \| \sum_{\tau = 1}^{k-1} \hat{\phi}_h^\tau \left(\hat{V}_{h+1}^k(s_{h+1}^\tau) - P_h^{(*,T+1)}\hat{V}_{h+1}^k(s_h^\tau,a_h^\tau)\right)\right\|_{(\Lambda_h^k)^{-1}} \lesssim dH\sqrt{\log\left(\frac{dK_{\text{RFE}}H\max(\xi_{\text{down}},1)}{\delta}\right)}.
    \end{equation*}
     Then $\text{Pr}[\mathcal{E}] \geq 1- \delta/8$.
\end{lemma}

\begin{lemma}\label{lemma:reward-free-pre-lemma}
    With probability $1- \delta/8$, we have for all $(s,a) \in \cS \times \cA$,
    \begin{equation*}
        \Big|\hat{\phi}_h(s,a)^\top \hat{w}_h^k - P_h^{(*,T+1)} \hat{V}_{h+1}^k(s,a)\Big| \lesssim \beta \|\hat{\phi}_h(s,a)\|_{(\Lambda_h^k)^{-1}} + H \xi_{\text{down}}
    \end{equation*}
\end{lemma}

\begin{lemma}\label{Lemma:V_optimism_RFE}
    With probability $1-\delta/4$, for all $(h,k) \in [H]\times[K_{\text{RFE}}]$, and any $s \in \cS$, we have
    \begin{equation*}
        \tilde{V}_h^*(s,r^k) \lesssim \hat{V}_h^k(s) + H(H-h+1)\xi_{\text{down}}
    \end{equation*}
    and 
    \begin{equation*}
        \sum_{k=1}^{K_{\text{RFE}}} \hat{V}_1^k(s_1^k) \leq c\sqrt{d^3H^4K_{\text{RFE}}\log(dK_{\text{RFE}}H/\delta)}+ H^2K_{\text{RFE}}\xi_{\text{down}},
    \end{equation*}
    where $c>0$ is a constant.
\end{lemma}

\begin{lemma}\label{Lemma:expectation-bound-planning-bonus-RFE}
    With probability $1- \delta/2$, for the function $u_h(\cdot,\cdot)$ defined in Line 5 of \Cref{Algorithm:downstream-reward-free-planning-phase}, we have
    \begin{equation*}
        \EE_{s\sim \mu} \Big[\tilde{V}_1^*(s,u)\Big] \leq c'\sqrt{d^3H^4\log(dK_{\text{RFE}}H/\delta)/K_{\text{RFE}}} +  2 H^2\xi_{\text{down}}.
    \end{equation*}
\end{lemma}

\begin{lemma}\label{Lemma:optimism-planning-RFE}
    With probability $1-\delta/2$, for any reward function which is linear with respect to the unknown feature $\phi^*: \cS\times \cA \rightarrow \RR^d$, for all $(s,a) \in \cS \times \cA$ and $h \in [H]$, we have 
    \begin{equation*}
        Q_h^*(s,a,r) - H(H-H+1)\xi_{\text{down}} \leq \hat{Q}_h(s,a) \leq r_h(s,a) + P_h^{(*,T+1)}\hat{V}_{h+1}(s,a) + 2u_h(s,a) + H\xi_{\text{down}}
    \end{equation*}
\end{lemma}

\subsection{Proof of \Cref{Theorem:reward-free-exploration-main}}

We first restate \Cref{Theorem:reward-free-exploration-main}

\begin{theorem}\label{Theorem:reward-free-exploration-appendix}
    Under \Cref{Assumption:upstream-to-downstream}, after collecting $K_{\text{RFE}}$ trajectories during the exploration phase in \Cref{Algorithm:downstream-reward-free-exploration-phase}, with probability at least $1-\delta$, the output of \Cref{Algorithm:downstream-reward-free-planning-phase}, policy $\pi$ satisfies
    \begin{equation}\label{Equation:RFE-theorem-expectation-bound}
        \EE_{s_1\sim \mu}[V_1^*(s_1,r) - V_1^\pi(s_1,r)] \leq c'\sqrt{d^3H^4\log(dK_{\text{RFE}}H/\delta)/K_{\text{RFE}}} +  6 H^2\xi_{\text{down}}.
    \end{equation}

    If the linear combination misspecification error $\xi$ in \Cref{Assumption:upstream-to-downstream} satisfies $\tilde{O}(\sqrt{d^3/K_{\text{RFE}}})$ and the number of trajectories in the offline dataset for each upstream task is  at least $\tilde{O}(TK_{\text{RFE}}/d^3)$, then, provided $K_{\text{RFE}}$ is at least $O(d^3H^4\log(dH\delta^{-1}\epsilon^{-1})/\epsilon^2)$,  with probability $1-\delta$, the policy $\pi$ will be an $\epsilon$-optimal policy for any given reward during the planning phase. 
\end{theorem}

\begin{proof}[Proof of \Cref{Theorem:reward-free-exploration-appendix}]
    We condition on the events defined in \Cref{Lemma:expectation-bound-planning-bonus-RFE} and \Cref{Lemma:optimism-planning-RFE} which, by union bound, hold with probability at least $1- \delta$. By \Cref{Lemma:optimism-planning-RFE}, for any $s \in \cS$, we have
    \begin{align*}
        \hat{V}_1(s) = \max_{a\in \cA} \hat{Q}_1(s,a) &\geq \max_{a\in \cA} Q^*_1(s,a,r) - H^2\xi_{\text{down}}\\
        &= V_1^*(s,r) - H^2\xi_{\text{down}}.\\
    \end{align*}

    This implies
    \begin{equation*}
        \EE_{s_1\sim \mu}[V_1^*(s_1,r) - V_1^\pi(s_1,r)] \leq \EE_{s_1\sim \mu}[\hat{V}_1(s_1) - V_1^\pi(s_1,r)] + H^2\xi_{\text{down}},
    \end{equation*}
    where $\pi$ is the policy returned by \Cref{Algorithm:downstream-reward-free-planning-phase}.

    Observe that, using \Cref{Lemma:optimism-planning-RFE}, we have
    \begin{align*}
        &\EE_{s_1 \sim \mu}[\hat{V}_1(s_1) - V_1^\pi(s_1,r)]\\
        &=\EE_{s_1 \sim \mu}[\hat{Q}_1(s_1, \pi_1(s_1)) - Q_1^\pi(s_1,\pi_1(s_1),r)]\\
        &\leq \EE_{s_1 \sim \mu}[r_1(s_1, \pi_1(s_1)) + P_1^{(*,T+1)}\hat{V}_2(s_1,\pi_1(s_1),r) + 2u_1(s_1, \pi_1(s_1)) + H\xi_{\text{down}} -r_1(s_1, \pi_1(s_1)) \\
        & \qquad \qquad - P_1^{(*,T+1)}V_2^\pi(s_1,\pi_1(s_1),r)]\\
        &=\EE_{s_1 \sim \mu, s_2 \sim P_1^{(*,T+1)}(\cdot|s_1,\pi_1(s_1))}[\hat{V}_2(s_2) - V_2^\pi(s_2,r) + 2u_1(s_1,\pi_1(s_1))] + H\xi_{\text{down}}\\
        &\leq \ldots \\
        &\leq 2 \EE_{s\sim \mu}[V_1^\pi(s,u)] + H^2\xi_{\text{down}}.
    \end{align*}

    Moreover, note that $0 \leq \hat{V}_h(s) \leq H$ and $0 \leq V_h^\pi(s,r) \leq H$ as $0 \leq r(s,a) \leq 1$. Thus, we would always have $\hat{V}_h(s) - V_h^\pi(s,r) \leq H$. Along with the previous derivation, this implies,

    \begin{equation*}
        \EE_{s_1 \sim \mu}[\hat{V}_1(s_1) - V_1^\pi(s_1,r)] \leq 2\EE_{s\sim\mu}[\tilde{V}_1^\pi(s,u)] + H^2\xi_{\text{down}}.
    \end{equation*}

    By definition of $\tilde{V}_1^*(s,u)$, we further have $\EE_{s\sim \mu}[\tilde{V}_1^\pi(s,u)] \leq \EE_{s\sim \mu}[\tilde{V}_1^*(s,u)]$. By \Cref{Lemma:expectation-bound-planning-bonus-RFE}, we have
    \begin{equation*}
        \EE_{s\sim \mu} \Big[\tilde{V}_1^*(s,u)\Big] \leq c'\sqrt{d^3H^4\log(dK_{\text{RFE}}H/\delta)/K_{\text{RFE}}} +  2 H^2\xi_{\text{down}}.
    \end{equation*}

    So, we have, 
    \begin{align}\label{Equation:bound-rfe}
        \EE_{s_1\sim \mu}[V_1^*(s_1,r) - V_1^\pi(s_1,r)] &\leq \EE_{s_1\sim \mu}[\hat{V}_1(s_1) - V_1^\pi(s_1,r)] + H^2\xi_{\text{down}}\notag\\
        &\leq 2\EE_{s\sim\mu}[\tilde{V}_1^\pi(s,u)] + 2H^2\xi_{\text{down}}\notag\\
        &\leq 2\EE_{s\sim\mu}[\tilde{V}_1^*(s,u)] + 2H^2\xi_{\text{down}}\notag\\
        &\leq 2c'\sqrt{d^3H^4\log(dK_{\text{RFE}}H/\delta)/K_{\text{RFE}}} +  6 H^2\xi_{\text{down}}.
    \end{align}

    Recall the definition of $\xi_{\text{down}}$ from \Cref{Lemma:approximate_feature_new_task}, that is, $\xi_{\text{down}} = \xi + \frac{C_L C_R\nu}{\kappa} \sqrt{\frac{2T\log(2|\Phi||\Psi|^TnH/\delta)}{n}}$. If the linear combination misspecification error $\xi$ in \Cref{Assumption:upstream-to-downstream} satisfies $\tilde{O}(\sqrt{d^3/K_{\text{RFE}}})$ and the number of trajectories in the offline dataset for each task in the upstream stage is  at least $\tilde{O}(TK_{\text{RFE}}/d^3)$, then the first term in \Cref{Equation:bound-rfe} dominates the second term $6H^2\xi_{\text{down}}$. Then, by taking $K_{\text{RFE}} = c_Kd^3H^4\log(dH\delta^{-1}\epsilon^{-1})/\epsilon^2$ for a sufficiently large constant $c_K >0$, we have 
    \begin{equation*}
        \EE_{s_1\sim \mu}[V_1^*(s_1,r) - V_1^\pi(s_1,r)] \leq 2c'\sqrt{d^3H^4\log(dK_{\text{RFE}}H/\delta)/K_{\text{RFE}}} +  6 H^2\xi_{\text{down}} \leq \epsilon.
    \end{equation*}

    This completes the proof.
\end{proof}
\section{Proof of Supporting Lemmas in \Cref{Appendix:Proof for Downstream Reward-Free RL}}\label{Section:proof of supporting lemmas for downstream RFE}

\subsection{Proof of \Cref{Lemma:weight_bound_downstream_reward-free-exploration-phase}}
\begin{proof}[Proof of \Cref{Lemma:weight_bound_downstream_reward-free-exploration-phase}]
    We have 
    \begin{align*}
        \big\|\hat{w}_h^k\big\| &= \bigg\|(\Lambda_h^k)^{-1}\sum_{\tau=1}^{k-1} \hat{\phi}_h(s_h^\tau,a_h^\tau)\hat{V}_{h+1}^k(s_{h+1}^\tau)\bigg\|\\
        &\leq \sqrt{k}\bigg(\sum_{\tau=1}^{k-1} \big\|\hat{V}_{h+1}^k(s_{h+1}^\tau) \hat{\phi}_h(s_h^\tau,a_h^\tau)\big\|^2_{(\Lambda_h^k)^{-1}}\bigg)^{1/2}\\
        &\leq \sqrt{K_{\text{RFE}}}\cdot H\cdot\bigg(\sum_{\tau=1}^{k-1} \big\| \hat{\phi}_h(s_h^\tau,a_h^\tau)\big\|^2_{(\Lambda_h^k)^{-1}}\bigg)^{1/2}\\
        &\leq H\sqrt{d K_{\text{RFE}}}
    \end{align*}
    where the first inequality follows from ~\Cref{Lemma:D.5_Ishfaq} and the fact that the largest eigenvalue of $(\Lambda_h^k)^{-1}$ is at most $1$, second inequality follows from the fact that $|\hat{V}_{h+1}^k(s)|\leq H$ for all $s\in\cS$ and the last inequality follows from \Cref{Lemma:D.1-chijin}.
\end{proof}

\subsection{Proof of \Cref{Lemma:self-normalized-bound-reward-free-exploration}}
\begin{proof}[Proof of \Cref{Lemma:self-normalized-bound-reward-free-exploration}]

The proof is similar to that of Lemma B.3 in \citep{jin2019provably} with the major difference being the usage of approximate feature map $\hat{\phi}(\cdot,\cdot)$ and different reward function at different episodes. We provide the full outline of the proof for completeness. 

For all $(k,h) \in [K_{\text{RFE}}]\times[H]$, by \Cref{Lemma:weight_bound_downstream_reward-free-exploration-phase}, we have $\|\hat{w}_h^k\|_2 \leq H\sqrt{dK}$. Moreover, we have $r_h^k(\cdot, \cdot) = u_h^k(\cdot, \cdot)$ and hence we have
\begin{equation*}
    r_h^k(\cdot, \cdot) + u_h^k(\cdot,\cdot) = 2 \beta \sqrt{\hat{\phi}(\cdot,\cdot)^\top (\Lambda_h^k)^{-1}\hat{\phi}(\cdot,\cdot)}.
\end{equation*}

Thus, our value function $\hat{V}_{h+1}^k$ is of the form
\begin{equation*}
    V(\cdot) := \min \Big\{ \max_{a \in \cA} \hat{\phi}(\cdot, a)^\top w + \beta \sqrt{\hat{\phi}(\cdot,a)^\top (\Lambda)^{-1}\hat{\phi}(\cdot,a)}, H \Big\},
\end{equation*}
for some $\Lambda \in \RR^{d\times d}$, and $w \in \RR^d$ which matches the value function class defined in \Cref{Lemma:covering_number_parametric_form-chi-jin-original}. Moreover, by construction, the minimum eigenvalue of $\Lambda_h^k$ is lower bounded by 1. Combining \Cref{Lemma:self-normalized-process} and \Cref{Lemma:covering_number_parametric_form-chi-jin-original}, we have for any fixed $\varepsilon > 0$ that
\begin{align*}
    &\left \| \sum_{\tau = 1}^{k-1} \hat{\phi}_h^\tau \left(\hat{V}_{h+1}^k(s_{h+1}^\tau) - P_h^{(*,T+1)}\hat{V}_{h+1}^k(s_h^\tau,a_h^\tau)\right)\right\|^2_{(\Lambda_h^k)^{-1}}\\
    &\leq 4H^2 \left[\frac{d}{2}\log(k+1) + d \log \left(1 + \frac{8H\sqrt{dk}}{\varepsilon}\right) + d^2 \log \left(1 + \frac{32\sqrt{d}\beta^2}{\varepsilon^2}\right) + \log \left(\frac{8}{\delta}\right)\right] + 8k^2\varepsilon^2
\end{align*}
We set the hyperparameter $\beta = C_L H\sqrt{d} + dH\sqrt{\log(dK_{\text{RFE}}H\max(\xi_{\text{down}},1)/\delta)} + H\xi_{\text{down}}\sqrt{dK_{\text{RFE}}}$. Finally, picking $\varepsilon = dH/k$, we have 
\begin{equation*}
        \left \| \sum_{\tau = 1}^{k-1} \hat{\phi}_h^\tau \left(\hat{V}_{h+1}^k(s_{h+1}^\tau) - P_h^{(*,T+1)}\hat{V}_{h+1}^k(s_h^\tau,a_h^\tau)\right)\right\|_{(\Lambda_h^k)^{-1}} \lesssim dH\sqrt{\log\left(\frac{dK_{\text{RFE}}H\max(\xi_{\text{down}},1)}{\delta}\right)},
    \end{equation*} 
    which concludes the proof.
\end{proof}

\subsection{Proof of \Cref{lemma:reward-free-pre-lemma}}
\begin{proof}[Proof of \Cref{lemma:reward-free-pre-lemma}]

We condition on the event $\mathcal{E}$ defined in \Cref{Lemma:self-normalized-bound-reward-free-exploration}, which holds with probability at least $1- \delta/8$. We define $\tilde{w}_h^k = \int_{s'} \hat{V}_{h+1}^k(s') \hat{\mu}^*(s')ds'$ where $\hat{\mu}^*$ is defined in the proof of \Cref{Lemma:approximate_feature_new_task}. Note that by \Cref{Lemma:approximate_feature_new_task}, we have $\|\tilde{w}_h^k\| \leq C_L H\sqrt{d}$.

Now,
\begin{align}\label{Equation:w_hat_minus-V_hat_RFE}
    &\Big|\hat{\phi}_h(s,a)^\top \hat{w}_h^k - P_h^{(*,T+1)} \hat{V}_{h+1}^k(s,a)\Big|\notag\\
    &= \Big|\hat{\phi}_h(s,a)^\top \hat{w}_h^k - \hat{\phi}_h(s,a)^\top \tilde{w}_h^k + \hat{\phi}_h(s,a)^\top \tilde{w}_h^k - P_h^{(*,T+1)} \hat{V}_{h+1}^k(s,a)\Big|\notag\\
    &\leq \Big|\hat{\phi}_h(s,a)^\top \hat{w}_h^k - \hat{\phi}_h(s,a)^\top \tilde{w}_h^k \Big| +  \Big|\overline{P}_h\hat{V}_{h+1}^k(s,a) - P_h^{(*,T+1)} \hat{V}_{h+1}^k(s,a)\Big|\notag\\
    &\leq \Big|\hat{\phi}_h(s,a)^\top \hat{w}_h^k - \hat{\phi}_h(s,a)^\top \tilde{w}_h^k \Big| + H\xi_{\text{down}},
\end{align}
where the last inequality follows from \Cref{Lemma:approximate_feature_new_task} and $|\hat{V}_{h+1}^k| \leq H$.

The first term in ~\eqref{Equation:w_hat_minus-V_hat_RFE} can be written as,
\begin{align}\label{Equation:phi_times_estimated_w_minus_tilde_w_rfe}
        &\hat{\phi}_h(s,a)^\top \hat{w}_h^k - \hat{\phi}_h(s,a)^\top \tilde{w}_h^k \notag\\
        &= \hat{\phi}_h(s,a)^\top (\Lambda_h^k)^{-1}\sum_{\tau=1}^{k-1} \hat{\phi}_h(s_h^\tau,a_h^\tau)\hat{V}_{h+1}^k(s_{h+1}^\tau) -\hat{\phi}_h(s,a)^\top (\Lambda_h^k)^{-1}(\Lambda_h^k) \tilde{w}_h^k\notag\\
        &= \hat{\phi}_h(s,a)^\top (\Lambda_h^k)^{-1}\bigg\{\sum_{\tau=1}^{k-1} \hat{\phi}_h(s_h^\tau,a_h^\tau)\hat{V}_{h+1}^k(s_{h+1}^\tau) -  \tilde{w}_h^k - \sum_{\tau=1}^{k-1}\hat{\phi}_h(s_h^\tau,a_h^\tau) \overline{P}_h \hat{V}_{h+1}^k \bigg\}\notag\\
        &= \underbrace{- \hat{\phi}_h(s,a)^\top (\Lambda_h^k)^{-1} \tilde{w}_h^k}_{\text{(a)}} + \underbrace{\hat{\phi}_h(s,a)^\top (\Lambda_h^k)^{-1}\bigg\{\sum_{\tau=1}^{k-1} \hat{\phi}_h(s_h^\tau,a_h^\tau)\Big[\hat{V}_{h+1}^k(s_{h+1}^\tau) - P_h^{(*,T+1)}\hat{V}_{h+1}^k(s_{h+1}^\tau)\bigg\}}_{\text{(b)}}\notag\\
        &\qquad + \underbrace{\hat{\phi}_h(s,a)^\top (\Lambda_h^k)^{-1}\bigg\{\sum_{\tau=1}^{k-1} \hat{\phi}_h(s_h^\tau,a_h^\tau) \Big(P_h^{(*,T+1)}-\overline{P}_h\Big)\hat{V}_{h+1}^k(s_{h+1}^\tau)\bigg\}}_{\text{(c)}}
    \end{align}

    We now bound $(a), (b), (c)$ in ~\eqref{Equation:phi_times_estimated_w_minus_tilde_w_rfe} individually.

    \textbf{Term (a).} We have,
    \begin{align}\label{Equation:term_a_rfe}
        \big|-\hat{\phi}_h(s,a)^\top (\Lambda_h^k)^{-1} \tilde{w}_h^k\big| &\leq \big\|\hat{\phi}_h(s,a)\big\|_{(\Lambda_h^k)^{-1}} \| \tilde{w}_h^k\|_{(\Lambda_h^k)^{-1}}\notag\\
        &\leq \|\tilde{w}_h^k\|_2\big\|\hat{\phi}_h(s,a)\big\|_{(\Lambda_h^k)^{-1}}\notag\\
        &\leq C_L H \sqrt{ d}\big\|\hat{\phi}_h(s,a)\big\|_{(\Lambda_h^k)^{-1}}.
    \end{align}

\textbf{Term (b).} Using Cauchy-Schwarz inequality and  the definition of the event $\mathcal{E}$ from \Cref{Lemma:self-normalized-bound-reward-free-exploration}, we have
    \begin{align}\label{Equation:term_b_rfe}
        &\hat{\phi}_h(s,a)^\top (\Lambda_h^k)^{-1}\bigg\{\sum_{\tau=1}^{k-1} \hat{\phi}_h(s_h^\tau,a_h^\tau)\Big[\hat{V}_{h+1}^k(s_{h+1}^\tau) - P_h^{(*,T+1)}\hat{V}_{h+1}^k(s_{h+1}^\tau)\bigg\}\notag\\
        &\leq \big\|\hat{\phi}_h(s,a)\big\|_{(\Lambda_h^k)^{-1}}\bigg\|\sum_{\tau=1}^{k-1} \hat{\phi}_h(s_h^\tau,a_h^\tau)\Big[\hat{V}_{h+1}^k(s_{h+1}^\tau) - P_h^{(*,T+1)}\hat{V}_{h+1}^k(s_{h+1}^\tau)\bigg\|_{(\Lambda_h^k)^{-1}}\notag\\
        &\lesssim dH\sqrt{\log\left(\frac{dK_{\text{RFE}}H \max (\xi_{\text{down}},1)}{\delta}\right)}\big\|\hat{\phi}_h(s,a)\big\|_{(\Lambda_h^k)^{-1}}.
    \end{align}

    \textbf{Term (c).} We have
    \begin{align}\label{Equation:term_c_rfe}
        &\hat{\phi}_h(s,a)^\top (\Lambda_h^k)^{-1}\bigg\{\sum_{\tau=1}^{k-1} \hat{\phi}_h(s_h^\tau,a_h^\tau) \Big(P_h^{(*,T+1)}-\overline{P}_h\Big)\hat{V}_{h+1}^k(s_{h+1}^\tau)\bigg\} \notag\\
        &\overset{(\romannumeral1)}{\leq} \bigg|\hat{\phi}_h(s,a)^\top (\Lambda_h^k)^{-1}\bigg\{\sum_{\tau=1}^{k-1} \hat{\phi}_h(s_h^\tau,a_h^\tau)\bigg\}\bigg|H\xi_{\text{down}}\notag\\
        &=\sum_{\tau=1}^{k-1}\big|\hat{\phi}_h(s,a)^\top (\Lambda_h^k)^{-1}\hat{\phi}_h(s_h^\tau,a_h^\tau)\big|H\xi_{\text{down}}\notag\\
        &\overset{(\romannumeral2)}{\leq} \sqrt{\bigg(\sum_{\tau=1}^{k-1} \big\|\hat{\phi}_h(s,a)\big\|^2_{(\Lambda_h^k)^{-1}}\bigg)\bigg(\sum_{\tau=1}^{k-1} \big\|\hat{\phi}_h(s_h^\tau,a_h^\tau)\big\|^2_{(\Lambda_h^k)^{-1}}\bigg)}H\xi_{\text{down}}\notag\\
        &\overset{(\romannumeral3)}{\leq} H\xi_{\text{down}}\sqrt{dk}\big\|\hat{\phi}_h(s,a)\big\|_{(\Lambda_h^k)^{-1}},
    \end{align}
    where $(i)$ follows from \Cref{Lemma:approximate_feature_new_task}, $(ii)$ follows from Cauchy-Schwarz inequality and $(iii)$ follows from \Cref{Lemma:D.1-chijin}.

    Substituting \eqref{Equation:term_a_rfe}, \eqref{Equation:term_b_rfe}, \eqref{Equation:term_c_rfe}, into \eqref{Equation:phi_times_estimated_w_minus_tilde_w_rfe}, and denoting $\beta =C_L H \sqrt{ d} + dH\sqrt{\log\left(\frac{dK_{\text{RFE}}H\max (\xi_{\text{down}},1)}{\delta}\right)} + H\xi_{\text{down}}\sqrt{dk} $ we get
    \begin{align*}
        \Big| \hat{\phi}_h(s,a)^\top \hat{w}_h^k - \hat{\phi}_h(s,a)^\top \tilde{w}_h^k \Big| \lesssim \beta \big\|\hat{\phi}_h(s,a)\big\|_{(\Lambda_h^k)^{-1}}
    \end{align*}

    Putting everything together in \eqref{Equation:w_hat_minus-V_hat_RFE}, we get 
    \begin{equation*}
        \Big|\hat{\phi}_h(s,a)^\top \hat{w}_h^k - P_h^{(*,T+1)} \hat{V}_{h+1}^k(s,a)\Big| \lesssim \beta \big\|\hat{\phi}_h(s,a)\big\|_{(\Lambda_h^k)^{-1}} + H \xi_{\text{down}},
    \end{equation*}
    which concludes the proof.
\end{proof}

\subsection{Proof of \Cref{Lemma:V_optimism_RFE}}

\begin{proof}[Proof of \Cref{Lemma:V_optimism_RFE}]
   We prove the first part using backward induction on $h$. For $h=H$, we have
   \begin{align*}
       \tilde{V}_H^*(s,r^k) &= \min \Big\{ \max_{a \in \cA} \Big\{r_H^k(s,a) + P_H^{(*,T+1)}\tilde{V}_{H+1}^*(s,a,r^k)\Big\}, H \Big\}\\
       & = \min \Big\{ \max_{a \in \cA} r_H^k(s,a), H \Big\}\\
       & \leq \min \Big\{ \max_{a \in \cA} \Big\{r_H^k(s,a) + \hat{\phi}_{H}(s,a)^\top \hat{w}_H^k + \beta \|\hat{\phi}_H(s,a)\|_{(\Lambda_H^k)^{-1}}\Big\}, H\Big\}\\
       &= \hat{V}_H^k(s)\\
       &\leq \hat{V}_H^k(s) + H(H-H+1)\xi_{\text{down}}.
   \end{align*}
   Suppose for some $h+1 \in [H]$, it holds that for all $s \in \cS$, 
   \begin{equation*}
       \tilde{V}_{h+1}^*(s,r^k) \leq \hat{V}_{h+1}^k(s) + H(H-h)\xi_{\text{down}}.
   \end{equation*}

   Then,
   \begin{align*}
       \tilde{V}_{h}^*(s,r^k) &= \min \Big\{ \max_{a \in \cA} \Big\{ r_h^k(s,a) + P_h^{(*,T+1)}\tilde{V}_{h+1}^*(s,a,r^k)\Big\}, H \Big\}\\
       &\leq \max_{a \in \cA} \Big\{ r_h^k(s,a) + P_h^{(*,T+1)}\tilde{V}_{h+1}^*(s,a,r^k)\Big\}\\
       &\leq \max_{a \in \cA} \Big\{ r_h^k(s,a) + P_h^{(*,T+1)}\hat{V}_{h+1}^k(s,a)\Big\} + H(H-h)\xi_{\text{down}}\\
       &\lesssim \max_{a \in \cA} \Big\{ r_h^k(s,a) + \hat{\phi}_h(s,a)^\top \hat{w}_h^k + \beta \|\hat{\phi}_h(s,a)\|_{(\Lambda_h^k)^{-1}} \Big\} + H\xi_{\text{down}} + H(H-h)\xi_{\text{down}}\\
       &= \max_{a \in \cA} \Big\{ r_h^k(s,a) + \hat{\phi}_h(s,a)^\top \hat{w}_h^k + \beta \|\hat{\phi}_h(s,a)\|_{(\Lambda_h^k)^{-1}} \Big\} + H(H-h + 1)\xi_{\text{down}},\\
   \end{align*}
   where the last inequality follows from \Cref{lemma:reward-free-pre-lemma}.

   Thus, we have
   \begin{align*}
       \tilde{V}_{h}^*(s,r^k) &\lesssim \min \Big\{\max_{a \in \cA} \Big\{ r_h^k(s,a) + \hat{\phi}_h(s,a)^\top \hat{w}_h^k + \beta \|\hat{\phi}_h(s,a)\|_{(\Lambda_h^k)^{-1}} \Big\}, H\Big\}+ H(H-h + 1)\xi_{\text{down}}\\
       &= \hat{V}_h^k(s) + H(H-h + 1)\xi_{\text{down}},\\
   \end{align*}
   as desired. This completes the first part of the proof.

   Now we prove the second part of the proof. For all $(k, h) \in [K_{\text{RFE}}]\times[H-1]$, we denote,
   \begin{equation*}
       \xi_h^k = P_h^{(*,T+1)}\hat{V}_{h+1}^k(s_h^k, a_h^k) - \hat{V}_{h+1}^k(s_{h+1}^k).
   \end{equation*}
   
   Conditioned on $\mathcal{E}$ from \Cref{Lemma:self-normalized-bound-reward-free-exploration} where $\text{Pr}[\mathcal{E}] \geq 1- \delta/8$,

   \begin{align*}
       \sum_{k=1}^{K_{\text{RFE}}} \hat{V}_1^k(s_1^k) &\leq \sum_{k=1}^{K_{\text{RFE}}} \Big( r_1^k(s_1^k,a_1^k) + \hat{\phi}_1(s_1^k, a_1^k)^\top \hat{w}_1^k + \beta \|\hat{\phi}_1(s_1^k,a_1^k)\|_{(\Lambda_1^k)^{-1}}\Big)\\
       &= \sum_{k=1}^{K_{\text{RFE}}} \Big( \hat{\phi}_1(s_1^k, a_1^k)^\top \hat{w}_1^k + 2\beta \|\hat{\phi}_1(s_1^k,a_1^k)\|_{(\Lambda_1^k)^{-1}}\Big)\\
       &\lesssim \sum_{k=1}^{K_{\text{RFE}}} \Big(P_1^{(*,T+1)} \hat{V}_2^k(s_1^k,a_1^k) + 3\beta \|\hat{\phi}_1(s_1^k,a_1^k)\|_{(\Lambda_1^k)^{-1}} +H\xi_{\text{down}}\Big)\\
       &= \sum_{k=1}^{K_{\text{RFE}}} \Big(\xi_1^k + \hat{V}_2^k(s_2^k) + 3\beta \|\hat{\phi}_1(s_1^k,a_1^k)\|_{(\Lambda_1^k)^{-1}} \Big)   + HK_{\text{RFE}}\xi_{\text{down}}\\
       &\leq \cdots\\
       &\leq \sum_{k=1}^{K_{\text{RFE}}} \sum_{h=1}^{H-1} \xi_h^k + \sum_{k=1}^{K_{\text{RFE}}}\sum_{h=1}^{H} 3\beta \|\hat{\phi}_h(s_h^k,a_h^k)\|_{(\Lambda_h^k)^{-1}} + H^2K_{\text{RFE}}\xi_{\text{down}},\\
   \end{align*}
   where the second inequality follows from \Cref{lemma:reward-free-pre-lemma}.

   Note that for each $h \in [H-1]$, $\{\xi_h^k\}_{k=1}^{K_{\text{RFE}}}$ is a martingale difference sequence with $|\xi_h^k| \leq H$. We define the event $\mathcal{E}'$ to be the event that 
   \begin{equation*}
       \Bigg| \sum_{k=1}^{K_{\text{RFE}}} \sum_{h=1}^{H-1} \xi_h^k \Bigg| \leq c'H^2\sqrt{K_{\text{RFE}}\log(K_{\text{RFE}}H/\delta)}.
   \end{equation*}
   By Azuma-Hoeffding inequality, we have $\text{Pr}[\mathcal{E}'] \geq 1 - \delta/8$.

   By Cauchy-Schwarz inequality, we have 
   \begin{equation*}
       \sum_{k=1}^{K_{\text{RFE}}} \sum_{h=1}^H \|\hat{\phi}_h(s_h^k,a_h^k)\|_{(\Lambda_h^k)^{-1}} \leq \sqrt{K_{\text{RFE}}H \sum_{k=1}^{K_{\text{RFE}}} \sum_{h=1}^H \hat{\phi}_h(s_h^k,a_h^k)^\top (\Lambda_h^k)^{-1} \hat{\phi}_h(s_h^k,a_h^k)}.
   \end{equation*}

   Using Lemma D.2 of \cite{jin2019provably}, we have 
   \begin{equation*}
       \sum_{k=1}^{K_{\text{RFE}}} \sum_{h=1}^H \hat{\phi}_h(s_h^k,a_h^k)^\top (\Lambda_h^k)^{-1} \hat{\phi}_h(s_h^k,a_h^k) \leq 2dH\log(K_{\text{RFE}}).
   \end{equation*}

   Conditioned on $\mathcal{E}$ and $\mathcal{E}$ where $\text{Pr}(\mathcal{E}\cap \mathcal{E}') \geq 1- \delta/4$, we have 
   \begin{align*}
       \sum_{k=1}^{K_{\text{RFE}}} \hat{V}_1^k(s_1^k) &\leq c'H^2\sqrt{K_{\text{RFE}}\log(K_{\text{RFE}}H/\delta)} + 3\beta \sqrt{K_{\text{RFE}}H\cdot 2dH\log(K_{\text{RFE}})} + H^2K_{\text{RFE}}\xi_{\text{down}}\\
       &\leq c\sqrt{d^3H^4K_{\text{RFE}}\log(dK_{\text{RFE}}H/\delta)}+ H^2K_{\text{RFE}}\xi_{\text{down}},\\
   \end{align*}
   which completes the proof. 
\end{proof}

\subsection{Proof of \Cref{Lemma:expectation-bound-planning-bonus-RFE}}

\begin{proof}[Proof of \Cref{Lemma:expectation-bound-planning-bonus-RFE}]
    Denote $\Delta^k = \tilde{V}_1^*(s_1^k, r^k) - \EE_{s\sim \mu} [\tilde{V}_1^*(s,r^k)]$. Note that $r^k$ depends only on the data collected during the first $k-1$ episodes. Thus, $\{\Delta^k\}_{k=1}^{K_{\text{RFE}}}$ is a martingale difference sequence. Moreover, we have $|\Delta^k| \leq H$. Using Azuma-Hoeffding inequality, with probability $1-\delta/4$, we have
    \begin{equation*}
        \Bigg| \sum_{k=1}^{K_{\text{RFE}}} \Delta^k \Bigg| \leq c_1 H \sqrt{K_{\text{RFE}} \log(1/\delta)},
    \end{equation*}
    where $c_1 > 0$ is an absolute constant. Therefore, we have

    \begin{equation*}
        \EE_{s \sim \mu}\Bigg[ \sum_{k=1}^{K_{\text{RFE}}} \tilde{V}_1^*(s,r^k)\Bigg] \leq \sum_{k=1}^{K_{\text{RFE}}} \tilde{V}_1^*(s_1^k, r^k) + c_1 H\sqrt{K_{\text{RFE}}\log(1/\delta)}.
    \end{equation*}

    Now, notice that for all $k \in [K_{\text{RFE}}]$, $\Lambda_h \succeq \Lambda_h^k$. Thus for all $(k,h) \in [K_{\text{RFE}}]\times[H]$, we have 
    \begin{equation*}
        r_h^k(\cdot,\cdot) \geq u_h(\cdot,\cdot),
    \end{equation*}
    which implies
    \begin{equation*}
        V_1^*(\cdot, u_h) \leq V_1^*(\cdot, r_h^k).
    \end{equation*}

    Using \Cref{Lemma:V_optimism_RFE} and union bound, we have with probability $1- \delta/2$,
    \begin{align*}
        \EE_{s\sim \mu} \big[\tilde{V}_1^*(s,u)\big] &\leq \EE_{s\sim \mu} \Bigg[\sum_{k=1}^{K_{\text{RFE}}} \tilde{V}_1^*(s,r^k)/K_{\text{RFE}}\Bigg]\\
        &\leq \frac{1}{K_{\text{RFE}}}\sum_{k=1}^{K_{\text{RFE}}} \tilde{V}_1^*(s_1^k,r^k) + c_1 H\sqrt{\frac{\log(1/\delta)}{K_{\text{RFE}}}}\\
        &\leq \frac{1}{K_{\text{RFE}}}\sum_{k=1}^{K_{\text{RFE}}} \big(\hat{V}_1^k(s_1^k) + H^2\xi_{\text{down}}\big) + c_1 H\sqrt{\frac{\log(1/\delta)}{K_{\text{RFE}}}}\\
        &= \frac{1}{K_{\text{RFE}}}\sum_{k=1}^{K_{\text{RFE}}} \hat{V}_1^k(s_1^k) + H^2\xi_{\text{down}} + c_1 H\sqrt{\frac{\log(1/\delta)}{K_{\text{RFE}}}}\\
        &\leq \frac{1}{K_{\text{RFE}}} \big(c\sqrt{d^3H^4K_{\text{RFE}}\log(dK_{\text{RFE}}H/\delta)}+ H^2K_{\text{RFE}}\xi_{\text{down}}\big) + H^2\xi_{\text{down}} + c_1 H\sqrt{\frac{\log(1/\delta)}{K_{\text{RFE}}}}\\
        &= c\sqrt{\frac{d^3H^4\log(dK_{\text{RFE}}H/\delta)}{K_{\text{RFE}}}} + c_1 H\sqrt{\frac{\log(1/\delta)}{K_{\text{RFE}}}}  + 2H^2\xi_{\text{down}}\\
        &\leq c'\sqrt{\frac{d^3H^4\log(dK_{\text{RFE}}H/\delta)}{K_{\text{RFE}}}} + 2H^2\xi_{\text{down}},\\
    \end{align*}
    for some absolute constant $c' > 0$. This completes the proof.
\end{proof}

\subsection{Proof of \Cref{Lemma:optimism-planning-RFE}}

\begin{proof}[Proof of \Cref{Lemma:optimism-planning-RFE}]
    Following the same argument as in the proof of \Cref{lemma:reward-free-pre-lemma}, for all $h \in [H] $ and $(s,a) \in \cS\times\cA$, with probability $1- \delta/4$, we have
    \begin{equation*}
        \Big|\hat{\phi}_h(s,a)^\top \hat{w}_h - P_h^{(*,T+1)} \hat{V}_{h+1}(s,a)\Big| \lesssim \beta \|\hat{\phi}_h(s,a)\|_{(\Lambda_h)^{-1}} + H \xi_{\text{down}}
    \end{equation*}

    Thus, for all $h \in [H]$ and $(s,a) \in \cS\times\cA$, we have
    \begin{align*}
        \hat{Q}_h(s,a) &\leq \hat{\phi}_h(s,a)^\top\hat{w}_h + r_h(s,a) + u_h(s,a)\\
        &\lesssim r_h(s,a) + P_h^{(*,T+1)}\hat{V}_{h+1}(s,a) + 2\beta \|\hat{\phi}_h(s,a)\|_{(\Lambda_h)^{-1}} + H\xi_{\text{down}}.
    \end{align*}

    Since, $\hat{Q}_h(s,a) \leq H$ and $u_h(\cdot,\cdot) = \min\left\{\beta \sqrt{\hat{\phi}(\cdot,\cdot)^\top (\Lambda_h)^{-1}\hat{\phi}(\cdot,\cdot)},H\right\}$, we have,
    \begin{equation*}
        \hat{Q}_h(s,a) \leq r_h(s,a) + P_h^{(*,T+1)}\hat{V}_{h+1}(s,a) + 2u_h(s,a) + H\xi_{\text{down}}.
    \end{equation*}
    This completes the first part of the proof.

    Now, using induction, we prove that for all $ h \in [H]$ and $(s,a) \in \cS\times\cA$, we have $Q_h^*(s,a,r) - H(H-h+1)\xi_{\text{down}} \leq \hat{Q}_h(s,a)$.

    When $h = H+1$, the claim is trivially true. Suppose for some $h \in [H]$, we have, for all $(s,a)\in \cS\times\cA$,
    \begin{equation*}
        Q_{h+1}^*(s,a,r) - H(H-h)\xi_{\text{down}} \leq \hat{Q}_h(s,a).
    \end{equation*}

    Note that, 
    \begin{equation*}
        \hat{Q}_h(s,a) = \min \left\{\hat{\phi}_h(s,a)^\top \hat{w}_h + r_h(s,a) + u_h(s,a), H\right\}
    \end{equation*}
    Since $Q_h^*(s,a,r) \leq H$ and $u_h(\cdot,\cdot) = \min\left\{\beta \sqrt{\hat{\phi}(\cdot,\cdot)^\top (\Lambda_h)^{-1}\hat{\phi}(\cdot,\cdot)},H\right\}$, it suffices to show that 

    \begin{equation*}
        Q_h^*(s,a,r) \leq \hat{\phi}_h(s,a)^\top \hat{w}_h + r_h(s,a) + \beta \|\hat{\phi}_h(s,a)\|_{(\Lambda_h)^{-1}} + H(H-h+1)\xi_{\text{down}}.
    \end{equation*}

    Applying the $\max$ operator on both side of the inductive hypothesis, we get

    \begin{equation*}
        V_{h+1}^*(s,r) \leq \hat{V}_{h+1} + H(H-h)\xi_{\text{down}}.
    \end{equation*}

    Now,
    \begin{align*}
        Q_h^*(s,a,r) &= r_h(s,a) + P_h^{(*,T+1)}V_{h+1}^*(s,a,r)\\
        &\leq r_h(s,a) + P_h^{(*,T+1)}\hat{V}_{h+1}(s,a) + H(H-h)\xi_{\text{down}}\\
        &\lesssim \hat{\phi}_h(s,a)^\top \hat{w}_h + \beta \|\hat{\phi}_h(s,a)\|_{(\Lambda_h)^{-1}} + H\xi_{\text{down}} + H(H-h)\xi_{\text{down}} + r_h(s,a)\\
        &= \hat{\phi}_h(s,a)^\top \hat{w}_h + r_h(s,a) + \beta \|\hat{\phi}_h(s,a)\|_{(\Lambda_h)^{-1}} + H(H-h+1)\xi_{\text{down}}.
    \end{align*}

    This completes the proof.
\end{proof}

\section{Proof for Downstream Offline RL}\label{Appendix:Proof for Downstream Offline RL}

First, we state the supporting lemmas that are used in the proof of \Cref{Theorem:offline_downstream}. The proof of these lemmas are provided in \Cref{Section:proof of supporting lemmas for downstream offline}.

\subsection{Supporting Lemmas}

The following lemma shows that the linear weight $\hat{w}_h$ in \Cref{Algorithm:downstream-offline-RL} is bounded. 
\begin{lemma}[Bounds on Weights in \Cref{Algorithm:downstream-offline-RL}]\label{Lemma:weight_bound_downstream_offline}
    For any $h \in [H]$, the weight $\hat{w}_h$ in \Cref{Algorithm:downstream-offline-RL} satisfies 
    \begin{align*}
        \big\|\hat{w}_h\big\|_2 \leq H\sqrt{d N_{\text{off}}/\lambda_d}.
    \end{align*}
\end{lemma}

Next, we present our main concentration lemma for this section that upper-bounds the stochastic noise in regression.

\begin{lemma}\label{Lemma:self-normalizing-argument-offline-downstream}
Setting $\lambda_d = 1$, $\beta(\delta) = c_\beta(Hd\sqrt{\iota(\delta)}+H\sqrt{dN_{\text{off}}}\xi_{\text{down}})$, where $\iota(\delta) = \log(HdN_{\text{off}}\max(\xi_{\text{down}},1)/\delta)$, with probability at least $1-\delta$, for all $h \in [H]$, we have
\begin{align*}
    \Bigg\|\sum_{\tau=1}^{N_{\text{off}}} \hat{\phi}_h(s_h^\tau,a_h^\tau)\Big[(P_h^{(*,T+1)}\hat{V}_{h+1})(s_h^\tau,a_h^\tau) - \hat{V}_{h+1}(s_{h+1}^\tau)\Big]\Bigg\|_{\Lambda_h^{-1}} &\lesssim Hd\sqrt{\iota}.
\end{align*}

\end{lemma}

Recall that, we define the Bellman operator $\BB_h$ as $(\BB_h f)(s,a) = r_h(s,a) + (P_h^{(*,T+1)} f)(s,a)$ for any $f:\cS\times\cA \rightarrow \RR$. In the next lemma, we denote the Bellman estimate $(\hat{\BB}_h\hat{V}_{h+1})(\cdot,\cdot) = \hat{\phi}_h(\cdot,\cdot)^\top \hat{w}_h$. This lemma provides an upper bound for the Bellman update error $|(\BB_h\hat{V}_{h+1} - \hat{\BB}_h\hat{V}_{h+1})(s,a)|$ and characterizes the impact of the misspecification of the representation taken from the upstream learning.

\begin{lemma}[Bound on Bellman update error]\label{Lemma:bound_on_Bellman_update_error}
    Set $\lambda_d = 1$, $\beta = c_\beta(Hd\sqrt{\iota}+H\sqrt{dN_{\text{off}}}\xi_{\text{down}})$, where $\iota = \log(HdN_{\text{off}}\max(\xi_{\text{down}},1)/\delta)$. 
    Define the following event
    \begin{align*}
        \cE(\delta) &= \Big\{\big|(\BB_h\hat{V}_{h+1} - \hat{\BB}_h\hat{V}_{h+1})(s,a)\big| \leq \beta \big\|\hat{\phi}_h(s,a)\big\|_{\Lambda_h^{-1}} + H\xi_{\text{down}}, \forall h \in [H] \text{ and } \forall (s,a) \in \cS\times\cA \Big\}.
    \end{align*}
    Then under \Cref{Assumption:upstream-to-downstream}, we have $\PP(\cE(\delta)) \geq 1-\delta$.
\end{lemma}

\begin{definition}[Model prediction error]\label{Definition:model_prediction_error}
For all $h \in [H]$, we define the model prediction error as,
\begin{equation*}
    l_{h}(s,a) = (\BB_h\hat{V}_{h+1})(s,a) - \hat{Q}_h(s,a).
\end{equation*}
\end{definition}

The following lemma decomposes the suboptimality gap into the summation of uncertainty metric of each step. 
\begin{lemma}\label{Lemma:model_prediction_error_offline}
    Let $\{\hat{\pi}\}_{h=1}^H$ be the output of \Cref{Algorithm:downstream-offline-RL}. Conditioned on the event $\cE(\delta)$ defined in \Cref{Lemma:bound_on_Bellman_update_error}, for any $h \in [H]$ and $(s,a)\in \cS\times\cA$, we have
    \begin{equation*}
        0 \leq l_h(s,a)\leq 2\Gamma_h(s,a),
    \end{equation*}
    where $\Gamma_h(s,a) = \beta \big\|\hat{\phi}_h(s,a)\big\|_{\Lambda_h^{-1}} + H\xi_{\text{down}}$. Furthermore, we have
    \begin{align*}
        V^{\pi^*}_{P^{(*,T+1)},r}(s) - V^{\hat{\pi}}_{P^{(*,T+1)},r}(s) \leq 2 \sum_{h=1}^H \EE_{\pi^*} [\Gamma_h(s_h,a_h)|s_1 = s]
    \end{align*}
\end{lemma}

\subsection{Proof of \Cref{Theorem:offline_downstream}}

We first restate \Cref{Theorem:offline_downstream}.

\begin{theorem}\label{Theorem_restate:offline_downstream}
    Under \Cref{Assumption:upstream-to-downstream}, setting $\lambda_d = 1$, $\beta = O(Hd\sqrt{\iota}+ H\sqrt{dN_{\text{off}}}\xi_{\text{down}})$, where $\iota = \log(HdN_{\text{off}}\max(\xi_{\text{down}},1)/\delta)$, with probability at least $1-\delta$, the suboptimality gap of \Cref{Algorithm:downstream-offline-RL} is at most \begin{align}\label{Equation_restatement:suboptimality_gap_behavior_policy_theorem_offline_RL}
        V^{\pi^*}_{P^{(*,T+1)},r}(s) - V^{\hat{\pi}}_{P^{(*,T+1)},r}(s) &\leq 2H^2\xi_{\text{down}} + 2\beta \sum_{h=1}^H \EE_{\pi^*}\bigg[\big\|\hat{\phi}_h(s_h,a_h)\big\|_{\Lambda_h^{-1}} \big| s_1 = s\bigg].
    \end{align}
    Additionally if \Cref{Assumption:feature_coverage} holds, and the sample size satisfies $N_{\text{off}} \geq 40/\kappa_\rho \cdot \log(4dH/\delta)$, then with probability $1-\delta$, we have,
    \begin{align}\label{Equation-restated:refined-downstream-offline-suboptimality-bound}
        & V^{\pi^*}_{P^{(*,T+1)},r}(s) - V^{\hat{\pi}}_{P^{(*,T+1)},r}(s)\notag\\
        &\leq O\bigg(\kappa_\rho^{-1/2}H^2d^{1/2}\xi_{\text{down}} + \kappa_\rho^{-1/2}H^2d\sqrt{\frac{\log(HdN_{\text{off}}\max(\xi_{\text{down}},1)/\delta)}{N_{\text{off}}}}\bigg).
    \end{align}
\end{theorem}

\begin{remark}

Compared to Theorem 4.4 in \citet{jin2021pessimism}, the suboptimality gap in \eqref{Equation_restatement:suboptimality_gap_behavior_policy_theorem_offline_RL} has an additional term $2H^2\xi_{\text{down}}$. When the linear misspecification error $\xi =\tilde{O}(\sqrt{d/N_{\text{off}}})$ and the number of trajectories $n$ in each upstream offline task dataset satisfies $n = \tilde{O}\big(\frac{T N_{\text{off}}}{d}\big)$, the RHS of \eqref{Equation-restated:refined-downstream-offline-suboptimality-bound} is dominated by $\tilde{O}(N_{\text{off}}^{-1/2} H^2 d)$ improving the suboptimality gap bound of REP-LCB \citep{uehara2021representation} by an order of $\tilde{O}(Hd)$ under low-rank MDP with unknown representation.
\end{remark}

\begin{proof}[Proof of \Cref{Theorem_restate:offline_downstream}]
    Letting $\Gamma_h = \beta\|\hat{\phi}_h(s,a)\|_{\Lambda_{h}^{-1}} + H\xi_{\text{down}}$ in \Cref{Lemma:model_prediction_error_offline}, with probability at least $1-\delta$, we have
    \begin{align*}
        V^{\pi^*}_{P^{(*,T+1)},r}(s) - V^{\hat{\pi}}_{P^{(*,T+1)},r}(s) &\leq 2 \sum_{h=1}^H \EE_{\pi^*} [\Gamma_h(s_h,a_h)|s_1 = s]\\
        &\leq 2H^2\xi_{\text{down}} + 2\beta \sum_{h=1}^H \EE_{\pi^*}\bigg[\big\|\hat{\phi}_h(s_h,a_h)\big\|_{\Lambda_h^{-1}} \big| s_1 = s\bigg]
    \end{align*}
    This finishes the first part of the proof. Now we will provide the suboptimality bound under the feature coverage assumption in \Cref{Assumption:feature_coverage}. From Appendix B.4 of \citet{jin2021pessimism}, if $N_{\text{off}} \geq 40/\kappa_\rho \cdot \log(4dH/\delta)$, then with probability $1-\delta/2$, for any $h\in[H]$ and $(s,a)\in \cS\times \cA$, we have
    \begin{align*}
        \big\|\hat{\phi}_h(s,a)\big\|_{\Lambda_h^{-1}} \leq \sqrt{\frac{2}{\kappa_\rho}}\cdot \frac{1}{\sqrt{N_{\text{off}}}}.
    \end{align*}

    Setting $\beta(\delta/2) = c_\beta(Hd\sqrt{\iota(\delta/2)}+H\sqrt{dN_{\text{off}}}\xi_{\text{down}})$, with probability $1-\delta/2$, we have
    \begin{align*}
        V^{\pi^*}_{P^{(*,T+1)},r}(s) - V^{\hat{\pi}}_{P^{(*,T+1)},r}(s) &\leq 2H^2\xi_{\text{down}} + 2\beta \sum_{h=1}^H \EE_{\pi^*}\bigg[\big\|\hat{\phi}_h(s_h,a_h)\big\|_{\Lambda_h^{-1}} \big| s_1 = s\bigg].
    \end{align*}

    Using union bound, we have the following bound with probability at least $1-\delta$
    \begin{align}
        & V^{\pi^*}_{P^{(*,T+1)},r}(s) - V^{\hat{\pi}}_{P^{(*,T+1)},r}(s)\notag\\
        &\leq 2H\bigg(H\xi_{\text{down}} + \beta(\delta/2)\sqrt{\frac{2}{\kappa_\rho}}\cdot \frac{1}{\sqrt{N_{\text{off}}}} \bigg)\notag\\
        &= O\bigg(\kappa_\rho^{-1/2}H^2d^{1/2}\xi_{\text{down}} + \kappa_\rho^{-1/2}H^2d\sqrt{\frac{\log(HdN_{\text{off}}\max(\xi_{\text{down}},1)/\delta)}{N_{\text{off}}}}\bigg).
    \end{align}
\end{proof}

\section{Proof of Supporting Lemmas in \Cref{Appendix:Proof for Downstream Offline RL}}\label{Section:proof of supporting lemmas for downstream offline}
In this section, we provide the proofs of the lemmas that we used in the proof of \Cref{Theorem:offline_downstream}.

\subsection{Proof of \Cref{Lemma:weight_bound_downstream_offline}}
\begin{proof}[Proof of \Cref{Lemma:weight_bound_downstream_offline}]
    We have 
    \begin{align*}
        \big\|\hat{w}_h\big\| &= \bigg\|\Lambda_h^{-1}\sum_{\tau=1}^{N_{\text{off}}} \hat{\phi}_h(s_h^\tau,a_h^\tau)(r_h^\tau + \hat{V}_{h+1}(s_{h+1}^\tau))\bigg\|\\
        &\leq \sqrt{\frac{N_{\text{off}}}{\lambda_d}}\bigg(\sum_{\tau=1}^{N_{\text{off}}} \big\|(r_h^\tau +\hat{V}_h(s_{h+1}^\tau) )\hat{\phi}_h(s_h^\tau,a_h^\tau)\big\|^2_{\Lambda_h^{-1}}\bigg)^{1/2}\\
        &\leq H\sqrt{\frac{N_{\text{off}}}{\lambda_d}}\bigg(\sum_{\tau=1}^{N_{\text{off}}} \big\| \hat{\phi}_h(s_h^\tau,a_h^\tau)\big\|^2_{\Lambda_h^{-1}}\bigg)^{1/2}\\
        &\leq H\sqrt{\frac{d N_{\text{off}}}{\lambda_d}},
    \end{align*}
    where the first inequality follows from ~\Cref{Lemma:D.5_Ishfaq} and the fact that the largest eigenvalue of $\Lambda_h^{-1}$ is at most $1/\lambda_d$, second inequality follows from the fact that $r_h^\tau \in [0,1]$ and  $|\hat{V}_{h+1}(s)|\leq H -1 $ for all $s\in\cS$ and the last inequality follows from \Cref{Lemma:D.1-chijin}.
\end{proof}

\subsection{Proof of \Cref{Lemma:self-normalizing-argument-offline-downstream}}
\begin{proof}[Proof of \Cref{Lemma:self-normalizing-argument-offline-downstream}]
    The value function $\hat{V}_{h+1}$ has the parametric form of
    \begin{align*}
        V(\cdot) &= \min\Big\{ \max_{a\in \cA}  w^\top \phi(\cdot, a) - \beta \sqrt{\phi(\cdot,a)^\top\Lambda^{-1}\phi(\cdot,a)},H - h + 1\Big\},
    \end{align*}
    where $w \in \RR^d$ and positive definite matrix $\Lambda$ is such that its minimum eigenvalue satisfies $\lambda_{\min}(\Lambda) \geq \lambda_d$. From \Cref{Lemma:weight_bound_downstream_offline}, we have $\|\hat{w}_h\| \leq H\sqrt{dN_{\text{off}}/\lambda_d}$. Thus, applying \Cref{Lemma:D.1-chijin} and \Cref{Lemma:covering_number_parametric_form-chi-jin-original}, we have, for any fixed $\varepsilon>0$ and for all $h \in [H]$, with probability at least $1-\delta$, 
    \begin{align}\label{Eq:bound_self_normalized_offline}
        &\Bigg\|\sum_{\tau=1}^{N_{\text{off}}} \hat{\phi}_h(s_h^\tau,a_h^\tau)\Big[(P_h^{(*,T+1)}\hat{V}_{h+1})(s_h^\tau,a_h^\tau) - \hat{V}_{h+1}(s_{h+1}^\tau)\Big]\Bigg\|_{\Lambda_h^{-1}}^2\notag\\
        &\leq 4H^2\bigg[\frac{d}{2}\log \bigg(\frac{N_{\text{off}} + \lambda_d}{\lambda_d}\bigg) + \log \frac{\cN_\varepsilon}{\delta}\bigg] + \frac{8N_{\text{off}}^2 \varepsilon^2}{\lambda_d}\notag\\
        &\leq 4H^2\bigg[\frac{d}{2}\log \bigg(\frac{N_{\text{off}} + \lambda_d}{\lambda_d}\bigg) + d\log \bigg(1+ \frac{4H\sqrt{dN_{\text{off}}}}{\varepsilon\sqrt{\lambda_d}}\bigg) + d^2 \log \bigg(1+ \frac{8\sqrt{d}\beta^2}{\lambda_d \varepsilon^2}\bigg) + \log \frac{1}{\delta}\bigg] + \frac{8N_{\text{off}}^2 \varepsilon^2}{\lambda_d} \notag\\
        &\leq 4H^2\bigg[\frac{d}{2}\log \bigg(\frac{N_{\text{off}} + \lambda_d}{\lambda_d}\bigg) + d\log \bigg(1+ \frac{4H\sqrt{dN_{\text{off}}}}{\varepsilon\sqrt{\lambda_d}}\bigg) + d^2 \log \bigg(1+ \frac{8\sqrt{d}\beta^2}{\lambda_d \varepsilon^2}\bigg) +  \log \frac{1}{\delta}\bigg] + \frac{8N_{\text{off}}^2 \varepsilon^2}{\lambda_d}.
    \end{align}
     Setting $\varepsilon = dH/N_{\text{off}}$, $\lambda_d = 1$, $\beta = c_\beta(Hd\sqrt{\iota}+H\sqrt{dN_{\text{off}}}\xi_{\text{down}})$, where $\iota = \log(HdN_{\text{off}}\max(\xi_{\text{down}},1)/\delta)$, we can further upper bound ~\eqref{Eq:bound_self_normalized_offline} by
    \begin{align*}
        &4H^2\bigg[\frac{d}{2}\log(1+N_{\text{off}}) + d\log(1+4d^{-1/2}N_{\text{off}}^{3/2}) + d^2\log(1 + 8d^{-3/2}H^{-2}N_{\text{off}}^2\beta^2) +  \log(1/\delta)\bigg] + 8d^2H^2\\
        &\lesssim H^2d\log N_{\text{off}} + H^2d\log(d^{-1/2}N_{\text{off}}^{3/2})  + H^2\log(1/\delta) + H^2d^2\log(Hd^{1/2}\iota N_{\text{off}}^3\xi_{\text{down}}) + d^2H^2\\
        &\lesssim H^2d^2\iota
    \end{align*}
    Therefore, we have 
    \begin{align*}
        \Bigg\|\sum_{\tau=1}^{N_{\text{off}}} \hat{\phi}_h(s_h^\tau,a_h^\tau)\Big[(P_h^{(*,T+1)}\hat{V}_{h+1})(s_h^\tau,a_h^\tau) - \hat{V}_{h+1}(s_{h+1}^\tau)\Big]\Bigg\|_{\Lambda_h^{-1}} &\lesssim Hd\sqrt{\iota}.
    \end{align*}
\end{proof}

\subsection{Proof of \Cref{Lemma:bound_on_Bellman_update_error}}
\begin{proof}[Proof of \Cref{Lemma:bound_on_Bellman_update_error}]
    For $h \in [H]$, we define $\hat{w}_h^* = \int_{s'}\hat{\mu}^*(s')\hat{V}_{h+1}(s')ds'$. Then we have
    \begin{align}\label{Eq:decomposition_bellman_update_error}
        & \big|(\BB_h \hat{V}_{h+1} - \hat{B}_h \hat{V}_{h+1})(s,a)\big|\notag\\
        & =\big|(\BB_h \hat{V}_{h+1} - \overline{\BB}_h \hat{V}_{h+1} + \overline{\BB}_h \hat{V}_{h+1} -  \hat{B}_h \hat{V}_{h+1})(s,a)\big|\notag\\
        &\leq \big|(\BB_h \hat{V}_{h+1} - \overline{\BB}_h \hat{V}_{h+1})(s,a)\big| + \big|(\overline{\BB}_h \hat{V}_{h+1} -  \hat{B}_h \hat{V}_{h+1})(s,a)\big|\notag\\
        &= \big| (P_h^{(*,T+1)}\hat{V}_{h+1} - \overline{P}_h \hat{V}_{h+1})(s,a)\big| + \big|(\overline{P}_h \hat{V}_{h+1})(s,a) - \hat{\phi}_h(s,a)^\top \hat{w}_h\big|\notag\\
        &\overset{(\romannumeral1)}{\leq} H\big\|P_h^{(*,T+1)}(\cdot|s,a) - \la \hat{\phi}_h(s,a),\hat{\mu}_h^*)(\cdot)\ra \big\|_{\text{TV}} + \big|\int_{s'}\la \hat{\phi}_h(s,a),\hat{\mu}_h^*(s')\ra \hat{V}_{h+1}(s')ds' - \hat{\phi}_h(s,a)^\top\hat{w}_h\big|\notag\\
        &\overset{(\romannumeral2)}{\leq} H\xi_{\text{down}} + \big| \hat{\phi}_h(s,a)^\top(\hat{w}_h^*-\hat{w})\big|,
    \end{align}
    where $(i)$ follows from the fact that $|\hat{V}_{h+1}(s)|\leq H$ for all $s\in\cS$ and  $(ii)$ follows from \Cref{Lemma:approximate_feature_new_task}.

    We now decompose the second term in ~\eqref{Eq:decomposition_bellman_update_error}. Recall that $\Lambda_h = \sum_{\tau=1}^{N_{\text{off}}} \hat{\phi}_h(s_h^\tau,a_h^\tau)\hat{\phi}_h(s_h^\tau,a_h^\tau)^\top + \lambda_d I_d$ and $\hat{w}_h = \Lambda_h^{-1}\sum_{\tau=1}^{N_{\text{off}}} \hat{\phi}_h(s_h^\tau,a_h^\tau)\hat{V}_{h+1}(s_{h+1}^\tau)$. Then, we have
    \begin{align}\label{Equation:decomposition}
       &\hat{\phi}_h(s,a)^\top(\hat{w}_h^*-\hat{w})\notag\\
       &=\hat{\phi}_h(s,a)^\top \Lambda_h^{-1}\bigg\{\bigg(\sum_{\tau=1}^{N_{\text{off}}} \hat{\phi}_h(s_h^\tau,a_h^\tau)\hat{\phi}_h(s_h^\tau,a_h^\tau)^\top + \lambda_d I_d \bigg) \hat{w}_h^* - \bigg(\sum_{\tau=1}^{N_{\text{off}}} \hat{\phi}_h(s_h^\tau,a_h^\tau)\hat{V}_{h+1}(s_{h+1}^\tau)\bigg)\bigg\}\notag\\
       &= \underbrace{\lambda_d \hat{\phi}_h(s,a)^\top \Lambda_h^{-1}\hat{w}_h^*}_{\text{(a)}} + \underbrace{\hat{\phi}_h(s,a)^\top \Lambda_h^{-1}\bigg\{\sum_{\tau=1}^{N_{\text{off}}} \hat{\phi}_h(s_h^\tau,a_h^\tau)\Big[(P_h^{(*,T+1)}\hat{V}_{h+1})(s_h^\tau,a_h^\tau) - \hat{V}_{h+1}(s_{h+1}^\tau)\Big]\bigg\}}_{\text{(b)}}\notag\\
       &\qquad + \underbrace{\hat{\phi}_h(s,a)^\top \Lambda_h^{-1}\bigg\{\sum_{\tau=1}^{N_{\text{off}}} \hat{\phi}_h(s_h^\tau,a_h^\tau)\Big[(\overline{P}_h \hat{V}_{h+1} -P_h^{(*,T+1)}\hat{V}_{h+1})(s_h^\tau,a_h^\tau)\Big]\bigg\}}_{\text{(c)}}
    \end{align}
    We now provide an upper bound for each of the terms in ~\eqref{Equation:decomposition}.

    \textbf{Term (a).} We have

    \begin{align}
        \lambda_d \hat{\phi}_h(s,a)^\top \Lambda_h^{-1}\hat{w}_h^* &\overset{(\romannumeral1)}{\leq} \lambda_d \big\|\hat{\phi}_h(s,a)\big\|_{\Lambda_h^{-1}}\|\hat{w}_h^*\|_{\Lambda_h^{-1}}\notag\\
        &\overset{(\romannumeral2)}{\leq} \sqrt{\lambda_d}\big\|\hat{\phi}_h(s,a)\big\|_{\Lambda_h^{-1}}\|\hat{w}_h^*\|_2\notag\\
        &\overset{(\romannumeral3)}{\leq}\sqrt{\lambda_d Hd}\big\|\hat{\phi}_h(s,a)\big\|_{\Lambda_h^{-1}},
    \end{align}
    where $(i)$ follows from Cauchy-Schwarz inequality, $(ii)$ follows from the fact that the largest eigenvalue of $\Lambda_h^{-1}$ is at most $1/\lambda_d$ and $(iii)$ follows from \Cref{Ass:realizability} and $|\hat{V}_{h+1}(s)|\leq H$ for all $s\in\cS$.

    \textbf{Term (b).} We have 
    \begin{align}
        &\hat{\phi}_h(s,a)^\top \Lambda_h^{-1}\bigg\{\sum_{\tau=1}^{N_{\text{off}}} \hat{\phi}_h(s_h^\tau,a_h^\tau)\Big[(P_h^{(*,T+1)}\hat{V}_{h+1})(s_h^\tau,a_h^\tau) - \hat{V}_{h+1}(s_{h+1}^\tau)\Big]\bigg\}\notag\\
        &\leq \big\|\hat{\phi}_h(s,a)\big\|_{\Lambda_h^{-1}} \Bigg\|\sum_{\tau=1}^{N_{\text{off}}} \hat{\phi}_h(s_h^\tau,a_h^\tau)\Big[(P_h^{(*,T+1)}\hat{V}_{h+1})(s_h^\tau,a_h^\tau) - \hat{V}_{h+1}(s_{h+1}^\tau)\Big]\Bigg\|_{\Lambda_h^{-1}}\notag\\
        &\lesssim Hd\sqrt{\iota}\big\|\hat{\phi}_h(s,a)\big\|_{\Lambda_h^{-1}}
    \end{align}
     where the first inequality comes from Cauchy Schwarz inequality and the last inequality follows from \Cref{Lemma:self-normalizing-argument-offline-downstream}.

    \textbf{Term (c).} We have
    \begin{align}\label{Eq:bound_term_c}
        &\hat{\phi}_h(s,a)^\top \Lambda_h^{-1}\bigg\{\sum_{\tau=1}^{N_{\text{off}}} \hat{\phi}_h(s_h^\tau,a_h^\tau)\Big[(\overline{P}_h \hat{V}_{h+1} -P_h^{(*,T+1)}\hat{V}_{h+1})(s_h^\tau,a_h^\tau)\Big]\bigg\}\notag\\
        & \leq \bigg|\hat{\phi}_h(s,a)^\top \Lambda_h^{-1}\bigg(\sum_{\tau=1}^{N_{\text{off}}} \hat{\phi}_h(s_h^\tau,a_h^\tau)\bigg)\bigg|\cdot H \xi_{\text{down}}\notag\\
        & = \bigg|\sum_{\tau=1}^{N_{\text{off}}} \hat{\phi}_h(s,a)^\top \Lambda_h^{-1}\hat{\phi}_h(s_h^\tau,a_h^\tau)\bigg|\cdot H\xi_{\text{down}}\notag\\
        &\leq \sqrt{\bigg(\sum_{\tau=1}^{N_{\text{off}}} \big\|\hat{\phi}_h(s,a)\big\|^2_{\Lambda_h^{-1}}\bigg)\bigg(\sum_{\tau=1}^{N_{\text{off}}}\big\|\hat{\phi}_h(s_h^\tau,a_h^\tau)\big\|^2_{\Lambda_h^{-1}}\bigg)}\cdot H\xi_{\text{down}}\notag\\
        &\leq H\xi_{\text{down}} \sqrt{d N_{\text{off}}}\big\|\hat{\phi}_h(s,a)\big\|_{\Lambda_h^{-1}},
    \end{align}
    where the first inequality follows from $|\hat{V}_{h+1}(s)|\leq H$ for all $s\in \cS$ and from \Cref{Lemma:approximate_feature_new_task}, the second inequality follows from Cauchy-Schwarz inequality and the last inequality follows from \Cref{Lemma:D.1-chijin}.

      Setting $\lambda_d = 1$, $\beta = c_\beta(Hd\sqrt{\iota}+ H\sqrt{dN_{\text{off}}}\xi_{\text{down}})$, where $\iota = \log(HdN_{\text{off}}\max(\xi_{\text{down}},1)/\delta)$ and combining \eqref{Eq:decomposition_bellman_update_error} to \eqref{Eq:bound_term_c}, we get for any $h\in[H]$ and for any $(s,a)\in \cS\times\cA$, with probability at least $1-\delta$, 
    \begin{align*}
        \big|(\BB_h\hat{V}_{h+1} - \hat{\BB}_h\hat{V}_{h+1})(s,a)\big| &\leq \beta \big\|\hat{\phi}_h(s,a)\big\|_{\Lambda_h^{-1}} + H \xi_{\text{down}}.
    \end{align*}
    This completes the proof.
\end{proof}

\subsection{Proof of \Cref{Lemma:model_prediction_error_offline}}
\begin{proof}[Proof of \Cref{Lemma:model_prediction_error_offline}]

    Recall that
    \begin{align*}
        \hat{Q}_h(\cdot,\cdot) &= \min\{ \hat{\phi}_h(\cdot,\cdot)^\top \hat{w}_h - \Gamma_h(\cdot,\cdot),H -h + 1\}^+
    \end{align*}

    If $ \hat{\phi}_h(s,a)^\top \hat{w}_h - \Gamma_h(s,a) \leq 0$, $\hat{Q}_h(s,a) = 0$. Then, we have $l_{h}(s,a) = (\BB_h\hat{V}_{h+1})(s,a) - \hat{Q}_h(s,a) = (\BB_h\hat{V}_{h+1})(s,a) > 0$. 

    If $ \hat{\phi}_h(s,a)^\top \hat{w}_h - \Gamma_h(s,a) > 0$, we have
    \begin{align*}
        l_{h}(s,a) &= (\BB_h\hat{V}_{h+1})(s,a) - \hat{Q}_h(s,a)\\
        &\geq (\BB_h\hat{V}_{h+1})(s,a)  -\hat{\phi}_h(s,a)^\top \hat{w}_h + \Gamma_h(s,a)\\
        &= (\BB_h\hat{V}_{h+1})(s,a) - (\hat{\BB}_h\hat{V}_{h+1})(s,a) + \Gamma_h(s,a)\\
        &\geq 0,
    \end{align*}
    where the last inequality follows from conditioning on the event $\cE(\delta)$. Thus, we have $l_h(s,a) \geq 0$ for any $h\in[H]$ and $(s,a) \in \cS\times\cA$.

    We now show that $l_h(s,a) \leq 2\Gamma_h(s,a)$. Observe that
    \begin{align*}
        \hat{\phi}_h(s,a)^\top \hat{w}_h - \Gamma_h(s,a) &= (\hat{\BB}_h\hat{V}_{h+1})(s,a) - \Gamma_h(s,a)\\
        &\leq (\BB_h\hat{V}_{h+1})(s,a)\\
        &\leq H-h+1,
    \end{align*}
    where the first inequality follows from the conditioning on the event $\cE(\delta)$ and the last inequality follows from the fact that $r_h \in [0,1]$ and $\hat{V}_{h+1} \in [0,H-h]$.
    
    Now,
    \begin{align*}
        \hat{Q}_h(s,a) &= \min\{ \hat{\phi}_h(s,a)^\top \hat{w}_h - \Gamma_h(s,a), H-h+1\}^+\\
        &=\max\{  \hat{\phi}_h(s,a)^\top \hat{w}_h - \Gamma_h(s,a),0\}\\
        &\geq  \hat{\phi}_h(s,a)^\top \hat{w}_h - \Gamma_h(s,a)\\
        &= (\hat{\BB}_h\hat{V}_{h+1})(s,a) - \Gamma_h(s,a).
    \end{align*}
    Now, from the definition of $l_h$, we have
    \begin{align*}
        l_{h}(s,a) &= (\BB_h\hat{V}_{h+1})(s,a) - \hat{Q}_h(s,a)\\
        &\leq (\BB_h\hat{V}_{h+1})(s,a)- (\hat{\BB}_h\hat{V}_{h+1})(s,a) + \Gamma_h(s,a)\\
        &\leq 2\Gamma_h(s,a),
    \end{align*}
    where the last inequality follows from the conditioning on the event $\cE(\delta)$.

    Finally, we obtain
    \begin{align*}
        &V^{\pi^*}_{P^{(*,T+1)},r}(s) - V^{\hat{\pi}}_{P^{(*,T+1)},r}(s)\\
        &\leq -\sum_{h=1}^H \EE_{\hat{\pi}}[l_h(s_h,a_h)|s_1 = s] + \sum_{h=1}^H \EE_{\pi^*} [l_h(s_h,a_h)|s_1 = s] \\
        &\leq 2 \sum_{h=1}^H \EE_{\pi^*} [\Gamma_h(s_h,a_h)|s_1 = s]
    \end{align*}
    where the first inequality follows from \Cref{Lemma:decomposition_suboptimality} and definition of $\hat{\pi}$, and the last inequality follows because with probability at least $1-\delta$, we have $0\leq l_h(s,a) \leq 2\Gamma_h(s,a)$ for any $h\in[H]$ and $(s,a)\in \cS\times\cA$.
\end{proof}

\section{Proof for Downstream Online RL}\label{Appendix:Proof for Downstream Online RL}

In this section for notational simplicity we denote $V^\pi_{h,P^{(*,T+1)},r^{T+1}}(s)$ and $Q^\pi_{h,P^{(*,T+1)},r^{T+1}}(s,a)$ by $V_h^\pi(s)$ and $Q_h^\pi(s,a)$ respectively.

First, we state the supporting lemmas that are used in the proof of \Cref{Theorem:downstream_online_task}. The proof of these lemmas are provided in \Cref{Section:proof of supporting lemmas for downstream online}.

\subsection{Supporting Lemmas}

The following concentration lemma for online RL upper-bounds the stochastic noise in regression. The proof is omitted since it is quite similar to the one of \Cref{Lemma:self-normalizing-argument-offline-downstream}.
\begin{lemma}\label{Lemma:self-normalizing-argument-online-downstream}
Setting $\lambda_d = 1$, $\beta_n= c_\beta(Hd\sqrt{\iota_n(\delta)}+H\sqrt{dn}\xi_{\text{down}}+ C_L \sqrt{Hd})$, where $\iota_n = \log(Hdn\max(\xi_{\text{down}},1)/\delta)$, with probability at least $1-\delta/2$, for all $h \in [H]$ and any $n\in[N]$, we have
\begin{align*}
    \Bigg\|\sum_{\tau=1}^{n-1} \hat{\phi}_h(s_h^\tau,a_h^\tau)\Big[(P_h^{(*,T+1)}\hat{V}^n_{h+1})(s_h^\tau,a_h^\tau) - \hat{V}^n_{h+1}(s_{h+1}^\tau)\Big]\Bigg\|_{\Lambda_h^{-1}} &\lesssim Hd\sqrt{\iota_n}.
\end{align*}

\end{lemma}

The next lemma recursively bounds the difference between the value function  maintained in \Cref{Algorithm:downstream-online-RL} (with-out bonus) and the true value function of any policy $\pi$. We provide a bound for this difference using their expected difference at next step, plus an error term. This error term is upper-bounded by the bonus term with high probability.

\begin{lemma}\label{Lemma:bound_action_value_difference}
    There exists an absolute constant $c_\beta$ such that for $\beta_n= c_\beta(Hd\sqrt{\iota_n(\delta)}+H\sqrt{dn}\xi_{\text{down}}+ C_L \sqrt{Hd})$, where $\iota_n = \log(Hdn\max(\xi_{\text{down}},1)/\delta)$, and for any policy $\pi$, with probability at least $1-\delta/2$, for any $(s,a)\in \cS\times \cA$, $n\in [N_{\text{on}}]$ and $h \in [H]$, we have
    \begin{align*}
         \hat{\phi}_h(s,a)^\top \hat{w}_h^n - Q_h^\pi(s,a) = P_h^{(*,T+1)}(\hat{V}_{h+1}^n - V_{h+1}^\pi)(s,a) + \Delta_h^n(s,a),
    \end{align*}
    for some $\Delta_h^n(s,a)$ that satisfies $\|\Delta_h^n(s,a)\| \leq \beta_n \|\hat{\phi}_h(s,a)\|_{(\Lambda_h^n)^{-1}} + 2H\xi_{\text{down}}$. 
\end{lemma}

We now prove optimism of the estimated value function in the following lemma.
\begin{lemma}[Optimism of value function]\label{Lemma:optimism_of_value_function}
    With probability at least $1-\delta/2$, for any $(s,a)\in \cS\times\cA$, $h\in[H]$ and $n \in [N_{\text{on}}]$, we have
    \begin{align*}
        \hat{Q}_h^n(s,a) \geq Q_h^*(s,a) - 2H(H-h+1) \xi_{\text{down}}.
    \end{align*}
\end{lemma}

\Cref{Lemma:bound_action_value_difference} also easily transforms a recursive formula for the value function difference $\delta_h^n = \hat{V}_h^n(s_h^n) - V_h^{\pi^n}(s_h^n)$. The following lemma will be useful in proving the regret bound for the online downstream task.
\begin{lemma}[Recursive formula]\label{Lemma:recursive_formula}
    Let $\delta_h^n = \hat{V}_h^n(s_h^n) - V_h^{\pi^n}(s_h^n)$ and $\xi_{h+1}^n = \EE[\delta_{h+1}^n |s_h^n,a_h^n] - \delta_{h+1}^n$. Then for any $(n,h) \in [N_{\text{on}},H]$ with probability at least $1-\delta/2$, we have
    \begin{align*}
        \delta_h^n \leq \delta_{h+1}^n + \xi_{h+1}^n + \beta_n \big\|\hat{\phi}_h(s_h^n,a_h^n)\big\|_{(\Lambda_h^n)^{-1}} + 2H\xi_{\text{down}}.
    \end{align*}
\end{lemma}

\subsection{Proof of \Cref{Theorem:downstream_online_task}}

We are now ready to prove the main theorem in this section. We first restate \Cref{Theorem:downstream_online_task}.

\begin{theorem}\label{Theorem_restate:downstream_online_task}
    Let $\tilde{\pi}$ be the uniform mixture of $\pi^1,\ldots,\pi^{N_{on}}$ in \Cref{Algorithm:downstream-online-RL}. Under \Cref{Assumption:upstream-to-downstream}, setting $\lambda_d = 1$, $\beta_n= O(Hd\sqrt{\iota_n(\delta)}+H\sqrt{dn}\xi_{\text{down}}+ C_L \sqrt{Hd})$, where $\iota_n = \log(Hdn\max(\xi_{\text{down}},1)/\delta)$, with probability  $1-\delta$, the suboptimality gap of \Cref{Algorithm:downstream-online-RL} satisfies
    \begin{align}\label{Equation-restated:refined-downstream-online-suboptimality-bound}
        V_{P^{(*,T+1)},r}^* -V_{P^{(*,T+1)},r}^{\tilde{\pi}} &\leq \tilde{O} ( H^2 d\xi_{\text{down}} + H^2d^{3/2}N_{\text{on}}^{-1/2}).
    \end{align}
\end{theorem}

\begin{remark}

As we use estimated representation $\{\hat{\phi}_h\}_{h=1}^H$ from the upstream tasks, we get an extra term $H^2d\xi_{\text{down}}$ in the suboptimality gap above. When the linear misspecification error $\xi =\tilde{O}(\sqrt{d/N_{\text{off}}})$ and the of trajectories $n$ in each upstream offline task dataset satisfies $n = \tilde{O}\big(\frac{T N_{\text{on}}}{d}\big)$, the RHS of \eqref{Equation-restated:refined-downstream-online-suboptimality-bound} is dominated by $\tilde{O}(H^2 d ^{3/2}N^{-1/2}_{\text{on}})$ improving the suboptimality gap bound of REP-UCB \citep{uehara2021representation} \footnote{We convert the $1/(1-\gamma)$ horizon dependence in REP-UCB \citet{uehara2021representation} to H. We further rescale their suboptimality gap by a factor of $H^2$ as we do not assume the sum of rewards to be within $[0,1]$. } by an order of $\tilde{O}(H^{3/2}K\sqrt{d})$ under low-rank MDP with unknown representation which attests to the benefit of upstream representation learning.
\end{remark}

\begin{proof}[Proof of \Cref{Theorem_restate:downstream_online_task}]

We first bound the cumulative regret by
\begin{align}\label{Equation:online_cumulative_regret}
    &\sum_{n=1}^{N_{\text{on}}} \Big(V_{P^{(*,T+1)},r}^* - V_{P^{(*,T+1)},r}^{\pi^n}\Big)\notag\\
    &\overset{(\romannumeral1)}{\leq} \sum_{n=1}^{N_{\text{on}}} \big(V_1^n - V_{P^{(*,T+1)},r}^{\pi^n}\big) + 2H^2N_{\text{on}}\xi_{\text{down}}\notag\\
    &\overset{(\romannumeral2)}{\leq} \sum_{n=1}^{N_{\text{on}}}\sum_{h=1}^H \Big[\xi_{h}^n + \beta_n \big\|\hat{\phi}_h(s_h^n,a_h^n)\big\|_{(\Lambda_h^n)^{-1}} + 2H\xi_{\text{down}}\Big]+ 2H^2N_{\text{on}}\xi_{\text{down}}\notag\\
    &\leq \underbrace{\sum_{n=1}^{N_{\text{on}}}\sum_{h=1}^H \xi_{h}^n}_{\text{(a)}} + \underbrace{\sum_{n=1}^{N_{\text{on}}}\sum_{h=1}^H \beta_n \big\|\hat{\phi}_h(s_h^n,a_h^n)\big\|_{(\Lambda_h^n)^{-1}}}_{\text{(b)}} + 4H^2N_{\text{on}}\xi_{\text{down}}
\end{align}
where $(i)$ follows from \Cref{Lemma:optimism_of_value_function} and $(ii)$ follows from \Cref{Lemma:recursive_formula}.

Note that in term $(a)$, $\{\xi_h^n\}_{n=1,h=1}^{N_{\text{on}},H}$ is a martingale difference sequence with $|\xi_h^n|\leq 2$. By Azuma-Hoeffding inequality, we have, with probability at least $1- \delta/4$,
\begin{align}\label{Equation:online_cumulative_regret_part_a}
    \Bigg|\sum_{n=1}^{N_{\text{on}}}\sum_{h=1}^H \xi_{h}^n\Bigg| \leq \sqrt{8N_{\text{on}} H \log (8/\delta)}.
\end{align}

For term $(b)$, we have 
\begin{align}\label{Equation:online_cumulative_regret_part_b}
    &\sum_{n=1}^{N_{\text{on}}}\sum_{h=1}^H \beta_n \big\|\hat{\phi}_h(s_h^n,a_h^n)\big\|_{(\Lambda_h^n)^{-1}}\notag\\
    &\overset{(\romannumeral1)}{\leq} \sum_{h=1}^H \sqrt{\sum_{n=1}^{N_{\text{on}}} \beta_n^2}\sqrt{\sum_{n=1}^{N_{\text{on}}}\big\|\hat{\phi}_h(s_h^n,a_h^n)\big\|^2_{(\Lambda_h^n)^{-1}}}\notag\\
    &\overset{(\romannumeral2)}{\lesssim}\sum_{h=1}^H\sqrt{2c_\beta^2(H^2d^2\iota_nN_{\text{on}}  + H^2 N_{\text{on}}^2d\xi_{\text{down}}^2 )}\sqrt{2d\log(1+N_{\text{on}}/(d\lambda))}\notag\\
    &\leq H\sqrt{2c_\beta^2(H^2d^2\iota_nN_{\text{on}}  + H^2N_{\text{on}}^2d\xi_{\text{down}}^2 )}\sqrt{4d\log N_{\text{on}}}\notag\\
    &\overset{(\romannumeral3)}{\leq} 2\sqrt{2}c_\beta \big(H^2\sqrt{d^3\iota_nN_{\text{on}}\log N_{\text{on}}} + H^2dN_{\text{on}}\xi_{\text{down}}\sqrt{\log N_{\text{on}}}),
\end{align}
where $(i)$ follows from Cauchy-Schwarz inequality, $(ii)$ follows from \Cref{Lemma:elliptical potential lemma}, and $(iii)$ follows from the inequality that $\sqrt{x+y} \leq \sqrt{x} + \sqrt{y}$ for all $x, y \geq 0$.

Combining \eqref{Equation:online_cumulative_regret}, \eqref{Equation:online_cumulative_regret_part_a} and \eqref{Equation:online_cumulative_regret_part_b} we get
\begin{align*}
    \sum_{n=1}^{N_{\text{on}}} \Big(V_{P^{(*,T+1)},r}^* - V_{P^{(*,T+1)},r}^{\pi^n} \Big) &\lesssim \tilde{O}(H^2dN_{\text{on}}\xi_{\text{down}} + H^2 \sqrt{d^3N_{\text{on}}} ).
\end{align*}
Dividing both sides by $N_{\text{on}}$, we get
\begin{align*}
    V_{P^{(*,T+1)},r}^* - V_{P^{(*,T+1)},r}^{\tilde{\pi}} &\leq \tilde{O} ( H^2 d\xi_{\text{down}} + H^2d^{3/2}N_{\text{on}}^{-1/2}).
\end{align*}
This completes the proof.
\end{proof}

\section{Proof of Supporting Lemmas in \Cref{Appendix:Proof for Downstream Online RL}}\label{Section:proof of supporting lemmas for downstream online}

In this section, we provide the proofs of the lemmas that we used in the proof of \Cref{Theorem:downstream_online_task}.

\subsection{Proof of \Cref{Lemma:bound_action_value_difference}}
\begin{proof}[Proof of \Cref{Lemma:bound_action_value_difference}]
    For any policy $\pi$, we define $w_h^\pi = \int V_{h+1}^\pi(s')\hat{\mu}^*(s')ds'$. Note that $\hat{\phi}_h(s,a)^\top w_h^\pi = \overline{P}_h V_{h+1}^\pi(s,a)$  and by \Cref{Lemma:approximate_feature_new_task}, we have $\|w_h^\pi\| \leq C_L  \sqrt{d}$.  

    Now, we derive the following
    \begin{align}\label{Equation:decompose_value_action_difference}
        & \hat{\phi}_h(s,a)^\top \hat{w}_h^n - Q_h^\pi(s,a)\notag\\
        &= \hat{\phi}_h(s,a)^\top \hat{w}_h^n -\hat{\phi}_h(s,a)^\top w_h^\pi + \hat{\phi}_h(s,a)^\top w_h^\pi - Q_h^\pi(s,a)\notag\\
        &\leq \big(\hat{\phi}_h(s,a)^\top \hat{w}_h^n -\hat{\phi}_h(s,a)^\top w_h^\pi\big) + \big| \hat{\phi}_h(s,a)^\top w_h^\pi - Q_h^\pi(s,a)\big|
    \end{align}

    Using \Cref{Lemma:approximate_feature_new_task}, we bound the second term as
    \begin{align}\label{Equation:Q_pi_minus_w_pi}
        \big|Q_h^\pi(s,a)  - \hat{\phi}_h(s,a)^\top w_h^\pi \big| &= \big|P_h^{(*,T+1)} V_{h+1}^\pi(s,a) - \overline{P}_h V_{h+1}^\pi(s,a)\big|\notag\\
        &\leq H\xi_{\text{down}},
    \end{align}
    where we used the observation $|V_{h+1}^\pi| \leq H$. 

    Now, the first term in \eqref{Equation:decompose_value_action_difference} can be bounded by
    \begin{align}\label{Equation:phi_times_estimated_w_minus_any_policy_w}
        &\hat{\phi}_h(s,a)^\top \hat{w}_h^n -\hat{\phi}_h(s,a)^\top w_h^\pi\notag\\
        &= \hat{\phi}_h(s,a)^\top (\Lambda_h^n)^{-1}\sum_{\tau=1}^{n-1} \hat{\phi}_h(s_h^\tau,a_h^\tau)\hat{V}_{h+1}^n(s_{h+1}^\tau) -\hat{\phi}_h(s,a)^\top w_h^\pi\notag\\
        &= \hat{\phi}_h(s,a)^\top (\Lambda_h^n)^{-1}\bigg\{\sum_{\tau=1}^{n-1} \hat{\phi}_h(s_h^\tau,a_h^\tau)\hat{V}_{h+1}^n(s_{h+1}^\tau) - \lambda_d w_h^\pi - \sum_{\tau=1}^{n-1}\hat{\phi}_h(s_h^\tau,a_h^\tau) \overline{P}_h V_{h+1}^\pi \bigg\}\notag\\
        &= \underbrace{-\lambda_d \hat{\phi}_h(s,a)^\top (\Lambda_h^n)^{-1} w_h^\pi}_{\text{(a)}} + \underbrace{\hat{\phi}_h(s,a)^\top (\Lambda_h^n)^{-1}\bigg\{\sum_{\tau=1}^{n-1} \hat{\phi}_h(s_h^\tau,a_h^\tau)\Big[\hat{V}_{h+1}^n(s_{h+1}^\tau) - P_h^{(*,T+1)}\hat{V}_{h+1}^n(s_{h+1}^\tau)\bigg\}}_{\text{(b)}}\notag\\
        &\qquad + \underbrace{\hat{\phi}_h(s,a)^\top (\Lambda_h^n)^{-1}\bigg\{\sum_{\tau=1}^{n-1} \hat{\phi}_h(s_h^\tau,a_h^\tau)\overline{P}_h(\hat{V}_{h+1}^n - V_{h+1}^\pi)(s_h^\tau,a_h^\tau)\bigg\}}_{\text{(c)}}\notag\\
        &\qquad + \underbrace{\hat{\phi}_h(s,a)^\top (\Lambda_h^n)^{-1}\bigg\{\sum_{\tau=1}^{n-1} \hat{\phi}_h(s_h^\tau,a_h^\tau) \Big(P_h^{(*,T+1)}-\overline{P}_h\Big)\hat{V}_{h+1}^n(s_{h+1}^\tau)\bigg\}}_{\text{(d)}}
    \end{align}
    We now bound $(a), (b), (c), (d)$ in ~\eqref{Equation:phi_times_estimated_w_minus_any_policy_w} individually.

    \textbf{Term (a).} We have,
    \begin{align}\label{Equation:term_a_online}
        \big|-\lambda_d \hat{\phi}_h(s,a)^\top (\Lambda_h^n)^{-1} w_h^\pi\big| &\leq \big\|\hat{\phi}_h(s,a)\big\|_{(\Lambda_h^n)^{-1}} \|\lambda_d w_h^\pi\|_{(\Lambda_h^n)^{-1}}\notag\\
        &\leq \sqrt{\lambda_d}\|w_h^\pi\|_2\big\|\hat{\phi}_h(s,a)\big\|_{(\Lambda_h^n)^{-1}}\notag\\
        &\leq C_L \sqrt{\lambda_d Hd}\big\|\hat{\phi}_h(s,a)\big\|_{(\Lambda_h^n)^{-1}},
    \end{align}
    where the first inequality follows from Cauchy-Schwarz inequality, the second inequality follows from the fact that the largest eigenvalue of $(\Lambda_h^n)^{-1}$ is at most $1/\lambda_d$ and the last inequality follows from \Cref{Ass:realizability} and $|V^\pi_{h+1}(s)|\leq H$ for all $s\in\cS$.

    \textbf{Term (b).} Using Cauchy-Schwarz inequality and  \Cref{Lemma:self-normalizing-argument-online-downstream}, we have
    \begin{align}\label{Equation:term_b_online}
        &\hat{\phi}_h(s,a)^\top (\Lambda_h^n)^{-1}\bigg\{\sum_{\tau=1}^{n-1} \hat{\phi}_h(s_h^\tau,a_h^\tau)\Big[\hat{V}_{h+1}^n(s_{h+1}^\tau) - P_h^{(*,T+1)}\hat{V}_{h+1}^n(s_{h+1}^\tau)\bigg\}\notag\\
        &\leq \big\|\hat{\phi}_h(s,a)\big\|_{(\Lambda_h^n)^{-1}}\bigg\|\sum_{\tau=1}^{n-1} \hat{\phi}_h(s_h^\tau,a_h^\tau)\Big[\hat{V}_{h+1}^n(s_{h+1}^\tau) - P_h^{(*,T+1)}\hat{V}_{h+1}^n(s_{h+1}^\tau)\bigg\|_{(\Lambda_h^n)^{-1}}\notag\\
        &\lesssim Hd\sqrt{\iota_n} \big\|\hat{\phi}_h(s,a)\big\|_{(\Lambda_h^n)^{-1}}.
    \end{align}

    \textbf{Term (c).} We have
    \begin{align}\label{Equation:term_c_online}
        &\hat{\phi}_h(s,a)^\top (\Lambda_h^n)^{-1}\bigg\{\sum_{\tau=1}^{n-1} \hat{\phi}_h(s_h^\tau,a_h^\tau)\overline{P}_h(\hat{V}_{h+1}^n - V_{h+1}^\pi)(s_h^\tau,a_h^\tau)\bigg\}\notag\\
        &\leq \bigg|\hat{\phi}_h(s,a)^\top (\Lambda_h^n)^{-1}\bigg\{\sum_{\tau=1}^{n-1} \hat{\phi}_h(s_h^\tau,a_h^\tau)\hat{\phi}_h(s_h^\tau,a_h^\tau)^\top \int (\hat{V}_{h+1}^n - V_{h+1}^\pi)(s')\hat{\mu}_h^*(s')ds'\bigg\}\bigg|\notag\\
        &= \bigg|\hat{\phi}_h(s,a)^\top (\Lambda_h^n)^{-1}(\Lambda_h^n - \lambda_d I)\int (\hat{V}_{h+1}^n - V_{h+1}^\pi)(s')\hat{\mu}_h^*(s')ds'\bigg|\notag\\
        &\leq \bigg|\hat{\phi}_h(s,a)^\top \int (\hat{V}_{h+1}^n - V_{h+1}^\pi)(s')\hat{\mu}_h^*(s')ds'\bigg| + \bigg|\lambda_d \hat{\phi}_h(s,a)^\top (\Lambda_h^n)^{-1}\int (\hat{V}_{h+1}^n - V_{h+1}^\pi)(s')\hat{\mu}_h^*(s')ds'\bigg|\notag\\
        &\overset{(\romannumeral1)}{\leq} \big|\overline{P}_h(\hat{V}_{h+1}^n - V_{h+1}^\pi)(s,a)\big| + C_L  \sqrt{\lambda_d Hd}\big\|\hat{\phi}_h(s,a)\big\|_{(\Lambda_h^n)^{-1}}\notag\\
        &\overset{(\romannumeral2)}{\leq} P_h^{(*,T+1)}(\hat{V}_{h+1}^n - V_{h+1}^\pi)(s,a) + H\xi_{\text{down}} + C_L  \sqrt{\lambda_d Hd}\big\|\hat{\phi}_h(s,a)\big\|_{(\Lambda_h^n)^{-1}},
    \end{align}
    where $(i)$ follows from similar steps as in \eqref{Equation:term_a_online} and $(ii)$ follows from \Cref{Lemma:approximate_feature_new_task}.

    \textbf{Term (d).} We have
    \begin{align}\label{Equation:term_d_online}
        &\hat{\phi}_h(s,a)^\top (\Lambda_h^n)^{-1}\bigg\{\sum_{\tau=1}^{n-1} \hat{\phi}_h(s_h^\tau,a_h^\tau) \Big(P_h^{(*,T+1)}-\overline{P}_h\Big)\hat{V}_{h+1}^n(s_{h+1}^\tau)\bigg\} \notag\\
        &\overset{(\romannumeral1)}{\leq} \bigg|\hat{\phi}_h(s,a)^\top (\Lambda_h^n)^{-1}\bigg\{\sum_{\tau=1}^{n-1} \hat{\phi}_h(s_h^\tau,a_h^\tau)\bigg\}\bigg|\cdot H\xi_{\text{down}}\notag\\
        &=\sum_{\tau=1}^{n-1}\big|\hat{\phi}_h(s,a)^\top (\Lambda_h^n)^{-1}\hat{\phi}_h(s_h^\tau,a_h^\tau)\big|\cdot H\xi_{\text{down}}\notag\\
        &\overset{(\romannumeral2)}{\leq} \sqrt{\bigg(\sum_{\tau=1}^{n-1} \big\|\hat{\phi}_h(s,a)\big\|^2_{(\Lambda_h^n)^{-1}}\bigg)\bigg(\sum_{\tau=1}^{n-1} \big\|\hat{\phi}_h(s_h^\tau,a_h^\tau)\big\|^2_{(\Lambda_h^n)^{-1}}\bigg)}\cdot H\xi_{\text{down}}\notag\\
        &\overset{(\romannumeral3)}{\leq} H\xi_{\text{down}}\sqrt{dn}\big\|\hat{\phi}_h(s,a)\big\|_{(\Lambda_h^n)^{-1}},
    \end{align}
    where $(i)$ follows from \Cref{Lemma:approximate_feature_new_task} and $|\hat{V}^n_{h+1}(s)|\leq H$ for all $s\in\cS$, $(ii)$ follows from Cauchy-Schwarz inequality and $(iii)$ follows from \Cref{Lemma:D.1-chijin}.

    Substituting \eqref{Equation:term_a_online}, \eqref{Equation:term_b_online}, \eqref{Equation:term_c_online}, \eqref{Equation:term_d_online} into \eqref{Equation:phi_times_estimated_w_minus_any_policy_w}, and setting $\lambda_d =1$ we get
    \begin{align}\label{Equation:combining_all_abcd}
        &\hat{\phi}_h(s,a)^\top \hat{w}_h^n -\hat{\phi}_h(s,a)^\top w_h^\pi\notag\\
        &\lesssim c_\beta \big(Hd\sqrt{\iota_n(\delta)}+ H\sqrt{dn}\xi_{\text{down}}+ C_L \sqrt{Hd} \big)\big\|\hat{\phi}_h(s,a)\big\|_{(\Lambda_h^n)^{-1}} \notag\\
        &\qquad + P_h^{(*,T+1)}(\hat{V}_{h+1}^n - V_{h+1}^\pi)(s,a) + H\xi_{\text{down}}\notag\\
        &= \beta_n \big\|\hat{\phi}_h(s,a)\big\|_{(\Lambda_h^n)^{-1}} + P_h^{(*,T+1)}(\hat{V}_{h+1}^n - V_{h+1}^\pi)(s,a) + H\xi_{\text{down}}.
    \end{align}
    Combining \eqref{Equation:combining_all_abcd}, \eqref{Equation:Q_pi_minus_w_pi} in \eqref{Equation:decompose_value_action_difference} completes the proof. 
\end{proof}

\subsection{Proof of \Cref{Lemma:optimism_of_value_function}}
\begin{proof}[Proof of \Cref{Lemma:optimism_of_value_function}]
We use backward induction for our proof. First, we prove the base case, at the last step $H$. By \Cref{Lemma:bound_action_value_difference}, we have
\begin{align*}
    \big| \hat{\phi}_H(s,a) \hat{w}_H^n - Q_H^\pi(s,a)\big| &= \big|P_H^{(*,T+1)}(\hat{V}_{H+1}^n - V_{H+1}^\pi)(s,a) + \Delta_H^n(s,a)\big|\\
    &\leq \beta_n \big\|\hat{\phi}_H(s,a)\big\|_{(\Lambda_H^n)^{-1}} + 2H\xi_{\text{down}}.
\end{align*}
Thus, we have
\begin{align*}
    \hat{Q}_H^n(s,a) &= \min\Big\{\hat{\phi}_H(s,a)^\top \hat{w}_H^n +\beta_n \big\|\hat{\phi}_H(s,a)\big\|_{(\Lambda_H^n)^{-1}},H-h+1 \Big\}\\
    &\geq Q_H^*(s,a) - 2H\xi_{\text{down}}.
\end{align*}

Now, suppose the statement holds true at step $h+1$ and consider the step $h$. Using \Cref{Lemma:bound_action_value_difference}, we have 
\begin{align*}
     \hat{\phi}_h(s,a)^\top \hat{w}_h^n - Q_h^*(s,a) &= \Delta_h^n(s,a) + P_h^{(*,T+1)}(\hat{V}_{h+1}^n - V_{h+1}^*)(s,a)\\
    &\geq -\beta_n\big\|\hat{\phi}_h(s,a)\big\|_{(\Lambda_h^n)^{-1}} - 2H\xi_{\text{down}} - 2H(H-h)\xi_{\text{down}}\\
    &= -\beta_n\big\|\hat{\phi}_h(s,a)\big\|_{(\Lambda_h^n)^{-1}} - 2H(H-h+1)\xi_{\text{down}}.
\end{align*}
Therefore,
\begin{align*}
    \hat{Q}_h^n(s,a) &= \min\{\hat{\phi}_h(s,a)^\top \hat{w}^n_h +\beta_n\|\hat{\phi}(s,a)\|_{(\Lambda_h^n)^{-1}},H-h+1\}^+\\
    &\geq Q_h^*(s,a) - 2H(H-h+1)\xi_{\text{down}},
\end{align*}
which completes the proof.
\end{proof}

\subsection{Proof of \Cref{Lemma:recursive_formula}}
\begin{proof}[Proof of \Cref{Lemma:recursive_formula}]
    By \Cref{Lemma:bound_action_value_difference}, for any $(s,a) \in \cS\times\cA$, $h\in[H]$ and $n \in [N_{\text{on}}]$, with probability at least $1-\delta/2$, we have 
    \begin{align*}
        \hat{Q}_h^n(s,a) - Q_h^{\pi^n}(s,a) &\leq \Delta_h^n(s,a) + P_h^{(*,T+1)}(\hat{V}_{h+1}^n - V_{h+1}^{\pi^n})(s,a)\\
        &\leq \beta_n \big\|\hat{\phi}_h(s,a)\big\|_{(\Lambda_h^n)^{-1}}+ 2H\xi_{\text{down}}+  P_h^{(*,T+1)}(\hat{V}_{h+1}^n - V_{h+1}^{\pi^n})(s,a).
    \end{align*}

    And finally, by definition of $\pi^n$ in \Cref{Algorithm:downstream-online-RL}, we have $\pi^n(s_h^n) = a_h^n = \argmax_{a\in\cA} \hat{Q}_h^n(s_h,a)$. This implies $\hat{Q}_h^n(s_h^n,a_h^n) -Q_h^{\pi^n}(s_h^n,a_h^n) = \hat{V}_h^n(s_h^n) -V_h^{\pi^n}(s_h^n)= \delta_h^n$. Thus, we have
    \begin{align*}
        \delta_h^n \leq \delta_{h+1}^n + \xi_{h+1}^n + \beta_n \big\|\hat{\phi}_h(s,a)\big\|_{(\Lambda_h^n)^{-1}}+ 2H\xi_{\text{down}}.
    \end{align*}
\end{proof}

\section{Auxiliary Lemmas}

\subsection{Miscellaneous Lemmas}
The following lemma measures the difference between two value functions under two MDPs and reward functions. Here, we use a shorthand notation $P_hV_{h+1}(s_h,a_h) = \EE_{s\sim P_h(\cdot|s_h,a_h)}[V(s)]$.
\begin{lemma}[Simulation lemma \citep{dann2017unifying}]\label{lemma:simulation_lemma}
    Consider two MDPs with transition kernels $P_1$ and $P_2$, and reward function $r_1$ and $r_2$ respectively. Given a policy $\pi$, we have,
    \begin{align*}
        V_{h,P_1,r_1}^\pi(s_h)-V_{h,P_2,r_2}^\pi(s_h) &= \sum_{h'=h}^H \EE_{s_{h'} \sim (P_2,\pi) \atop a_{h'} \sim \pi}\big[r_1(s_{h'},a_{h'}) - r_2(s_{h'},a_{h'}) \\
        &\qquad + (P_{1,h'}-P_{2,h'})V_{h'+1,P_1,r_1}^\pi(s_{h'},a_{h'}) \given s_h\big]\\
        &= \sum_{h'=h}^H \EE_{s_{h'} \sim (P_1,\pi) \atop a_{h'} \sim \pi}\big[r_1(s_{h'},a_{h'}) - r_2(s_{h'},a_{h'}) \\
        &\qquad + (P_{1,h'}-P_{2,h'})V_{h'+1,P_2,r_2}^\pi(s_{h'},a_{h'}) \given s_h\big]\\
    \end{align*}
\end{lemma}

We use the following lemma to deal with distribution shift in offline RL setting. 
\begin{lemma}[Distribution shift lemma \citep{chang2021mitigating}]\label{Lemma:distribution_shift}
    Consider two distributions $\rho_1 \in \Delta(\cS\times\cA)$ and $\rho_2 \in \Delta(\cS\times\cA)$, and a feature mapping $\phi:\cS\times \cA \rightarrow \RR^d$. Denote $C:= \text{sup}_{x\in \RR^d}\frac{x^\top \EE_{s,a\sim \rho_1}\phi(s,a)\phi(s,a)^\top x}{x^\top \EE_{s,a\sim \rho_2}\phi(s,a)\phi(s,a)^\top x}$. Then for any positive definite matrix $\Lambda$, we have
    \begin{align*}
        \EE_{s,a\sim \rho_1}\phi(s,a)^\top \Lambda \phi(s,a) \leq C\EE_{s,a\sim \rho_2}\phi(s,a)^\top \Lambda \phi(s,a).
    \end{align*}
\end{lemma}

We use the following lemma to bound the suboptimality in downstream offline RL.
\begin{lemma}[Decomposition of Suboptimality, Lemma 3.1 in \citep{jin2021pessimism}]\label{Lemma:decomposition_suboptimality}
    Let $\hat{\pi} = \{\hat{\pi}_h\}_{h=1}^H$ be the policy such that $\hat{V}_h(s) = \la \hat{Q}_h(s,\cdot),\hat{\pi}_h(\cdot|s)\ra_\cA$ and for each step $h\in[H]$, define the model evaluation error as $l_h(s,a) = (\BB_h \hat{V}_{h+1})(s,a) - \hat{Q}_h(s,a)$. For any $\hat{\pi}$ and $s\in \cS$, we have
    \begin{align*}
        V_1^{\pi^*}(s) - V_1^{\hat{\pi}}(s) &= -\sum_{h=1}^H \EE_{\hat{\pi}}[l_h(s_h,a_h)|s_1 = s] + \sum_{h=1}^H \EE_{\pi^*} [l_h(s_h,a_h)|s_1 = s] \\
        &\qquad + \sum_{h=1}^H \EE_{\pi^*}[\la \hat{Q}_h(s_h,\cdot), \pi_h^*(\cdot|s_h) - \hat{\pi}_h(\cdot|s_h)\ra_\cA |s_1 = s].
    \end{align*}
\end{lemma}

\subsection{Inequalities for summations}
\begin{lemma}[Lemma D.1 in \cite{jin2019provably}] \label{Lemma:D.1-chijin}

Let $\Lambda_h = \lambda I + \sum_{i=1}^t \phi_i\phi_i^\top$, where $\phi_i \in \mathbb{R}^d$ and $\lambda > 0$. Then it holds that
\begin{equation*}
    \sum_{i=1}^t \phi_i^\top(\Lambda_h)^{-1}\phi_i \leq d.
\end{equation*}
\end{lemma}

\begin{lemma}[Lemma D.5 in \citep{ishfaq2021randomized}]\label{Lemma:D.5_Ishfaq}
    Let $A \in \mathbb{R}^{d\times d}$ be a positive definite matrix where its largest eigenvalue $\lambda_{\max}(A) \leq \lambda$. Let $x_1, \ldots, x_k$ be $k$ vectors in $\mathbb{R}^d$. Then it holds that
    \begin{equation*}
        \Bigl\|A \sum_{i=1}^k x_i\Bigr\| \leq \sqrt{\lambda k} \left(\sum_{i=1}^k \|x_i\|_A^2\right)^{1/2}.
    \end{equation*}
\end{lemma}

\begin{lemma}[Lemma G.2 in \citep{agarwal2020flambe}]\label{Lemma:elliptical potential lemma}
    Consider a sequence of semidefinite matrices $X_1, \ldots, X_N \in \RR^{d\times d}$ with $\Tr(X_n) \leq 1$ for all $n \in [N]$. Define $M_0 = \lambda I$ and $M_n = M_{n-1} + X_n$. Then 
    \begin{equation*}
        \sum_{n=1}^N \Tr(X_n M_{n-1}^{-1}) \leq 2\log \frac{\det (M_N)}{\det (M_0)} \leq 2d \log (1 + N/(\lambda d)).
    \end{equation*}
\end{lemma}
\subsection{Covering numbers and self-normalized processes}

\begin{lemma}[Lemma D.4 in \cite{jin2019provably}]\label{Lemma:self-normalized-process}
    Let $\{s_i\}_{i=1}^\infty$ be a stochastic process on state space $\mathcal{S}$ with corresponding filtration $\{\mathcal{F}_i\}_{i=1}^\infty$. Let $\{\phi_i\}_{i=1}^\infty$ be an $\mathbb{R}^d$-valued stochastic process where $\phi_i \in \mathcal{F}_{i-1}$, and $\|\phi_i\| \leq 1$. Let $\Lambda_k = \lambda I + \sum_{i=1}^k \phi_i \phi_i^\top$. Then for any $\delta > 0$, with probability at least $1-\delta$, for all $k \geq 0$, and any $V \in \mathcal{V}$ with $\sup_{s \in \mathcal{S}} |V(s)| \leq H$, we have

\begin{equation*}
    \Bigl\|\sum_{i=1}^k \phi_i \bigl \{ V(s_i) - \mathbb{E}[V(s_i) \,|\, \mathcal{F}_{i-1} ]\bigr \}  \Bigr\|^2_{\Lambda_k^{-1}} \leq 4H^2 \Bigl [\frac{d}{2}\log\Bigl(\frac{k+\lambda}{\lambda}\Bigr) + \log{\frac{\mathcal{N}_\varepsilon}{\delta}}\Bigr] + \frac{8k^2\epsilon^2}{\lambda},
\end{equation*}
where $\mathcal{N}_\varepsilon$ is the $\varepsilon$-covering number of $\mathcal{V}$ with respect to the distance $\text{dist}(V,V') = \sup_{s \in \mathcal{S}} |V(s) - V'(s)|$.
\end{lemma}

\begin{lemma}[Covering number of Euclidean ball, \cite{vershynin2018high} ]\label{lemma:covering-number-euclidean-ball}
For any $\varepsilon > 0$, the $\varepsilon$-covering number, $\mathcal{N}_\varepsilon$, of the Euclidean ball of radius $B > 0$ in $\mathbb{R}^d$ satisfies 
\begin{equation*}
    \mathcal{N}_\varepsilon \leq \Bigl(1+ \frac{2B}{\varepsilon}\Bigr)^d \leq \Bigl( \frac{3B}{\varepsilon}\Bigr)^d .
\end{equation*}
\end{lemma}

\begin{lemma}[$\varepsilon$-covering number \citep{jin2019provably}]\label{Lemma:covering_number_parametric_form-chi-jin-original}
    Let $\cV$ denote a class of functions mapping from $\cS$ to $\RR$ with the following parametric form 
    \begin{align*}
        V(\cdot) &= \min\Big\{ \max_{a\in \cA}  w^\top \phi(\cdot, a) + \beta \sqrt{\phi(\cdot,a)^\top\Lambda^{-1}\phi(\cdot,a)}, H\Big\},
    \end{align*}
    where the parameters $(w,\beta,\Lambda)$ satisfy $\|w\|\leq L$, $\beta \in [0,B]$ and the minimum eigenvalue satisfies $\lambda_{\min}(\Lambda) \geq \lambda$. Assume $\|\phi(s,a)\|\leq 1$ for all $(s,a)$ pairs, and let $\cN_\varepsilon$ be the $\varepsilon$-covering number of $\cV$ with respect to the distance $\text{dist}(V,V') = \sup_{s}|V(s)-V'(s)|$. Then,
    \begin{align*}
        \log \cN_\varepsilon &\leq d\log(1+4L/\varepsilon) + d^2\log(1 + 8d^{1/2}B^2/(\lambda \varepsilon^2)).
    \end{align*}
\end{lemma}

\newpage
\section*{NeurIPS Paper Checklist}

\begin{enumerate}

\item {\bf Claims}
    \item[] Question: Do the main claims made in the abstract and introduction accurately reflect the paper's contributions and scope?
    \item[] Answer: \answerYes{} %
    \item[] Justification: The claims made in the abstract and the introduction are reflected in the paper's main contribution.

\item {\bf Limitations}
    \item[] Question: Does the paper discuss the limitations of the work performed by the authors?
    \item[] Answer: \answerYes{} %
    \item[] Justification: We discuss the assumptions made in order for our theorems to hold true and discuss the limitations of those assumptions.

\item {\bf Theory Assumptions and Proofs}
    \item[] Question: For each theoretical result, does the paper provide the full set of assumptions and a complete (and correct) proof?
    \item[] Answer: \answerYes{} %
    \item[] Justification: We provide detailed list of assumptions along with their interpretation. The complete proofs are provided in the appendix.

    \item {\bf Experimental Result Reproducibility}
    \item[] Question: Does the paper fully disclose all the information needed to reproduce the main experimental results of the paper to the extent that it affects the main claims and/or conclusions of the paper (regardless of whether the code and data are provided or not)?
    \item[] Answer: \answerNA{} %
    \item[] Justification: We do not have any experiment in this paper.

\item {\bf Open access to data and code}
    \item[] Question: Does the paper provide open access to the data and code, with sufficient instructions to faithfully reproduce the main experimental results, as described in supplemental material?
    \item[] Answer: \answerNA{} %
    \item[] Justification: Not applicable.

\item {\bf Experimental Setting/Details}
    \item[] Question: Does the paper specify all the training and test details (e.g., data splits, hyperparameters, how they were chosen, type of optimizer, etc.) necessary to understand the results?
    \item[] Answer: \answerNA{} %
    \item[] Justification: Not applicable.

\item {\bf Experiment Statistical Significance}
    \item[] Question: Does the paper report error bars suitably and correctly defined or other appropriate information about the statistical significance of the experiments?
    \item[] Answer: \answerNA{} %
    \item[] Justification: There is no experiment in the paper as it is theoretical in nature.

\item {\bf Experiments Compute Resources}
    \item[] Question: For each experiment, does the paper provide sufficient information on the computer resources (type of compute workers, memory, time of execution) needed to reproduce the experiments?
    \item[] Answer: \answerNA{} %
    \item[] Justification: The paper does not include any experiment.
    
\item {\bf Code Of Ethics}
    \item[] Question: Does the research conducted in the paper conform, in every respect, with the NeurIPS Code of Ethics \url{https://neurips.cc/public/EthicsGuidelines}?
    \item[] Answer: \answerYes{} %
    \item[] Justification: 

\item {\bf Broader Impacts}
    \item[] Question: Does the paper discuss both potential positive societal impacts and negative societal impacts of the work performed?
    \item[] Answer: \answerNA{} %
    \item[] Justification: It is a theoretical paper and thus we do not forsee any immediate societal impact.
    
\item {\bf Safeguards}
    \item[] Question: Does the paper describe safeguards that have been put in place for responsible release of data or models that have a high risk for misuse (e.g., pretrained language models, image generators, or scraped datasets)?
    \item[] Answer: \answerNA{} %
    \item[] Justification: It is a theory paper.

\item {\bf Licenses for existing assets}
    \item[] Question: Are the creators or original owners of assets (e.g., code, data, models), used in the paper, properly credited and are the license and terms of use explicitly mentioned and properly respected?
    \item[] Answer: \answerNA{} %
    \item[] Justification: Not applicable.

\item {\bf New Assets}
    \item[] Question: Are new assets introduced in the paper well documented and is the documentation provided alongside the assets?
    \item[] Answer: \answerNA{} %
    \item[] Justification: Not applicable

\item {\bf Crowdsourcing and Research with Human Subjects}
    \item[] Question: For crowdsourcing experiments and research with human subjects, does the paper include the full text of instructions given to participants and screenshots, if applicable, as well as details about compensation (if any)? 
    \item[] Answer: \answerNA{} %
    \item[] Justification: There was no crowdsourcing nor research with human subjects.

\item {\bf Institutional Review Board (IRB) Approvals or Equivalent for Research with Human Subjects}
    \item[] Question: Does the paper describe potential risks incurred by study participants, whether such risks were disclosed to the subjects, and whether Institutional Review Board (IRB) approvals (or an equivalent approval/review based on the requirements of your country or institution) were obtained?
    \item[] Answer: \answerNA{} %
    \item[] Justification: There was no study subject.

\end{enumerate}

\end{document}